\documentclass[lettersize,journal]{IEEEtran}
\usepackage{amsmath,amsfonts}
\usepackage{algorithmic}
\usepackage{algorithm}
\usepackage{array}
\usepackage{textcomp}
\usepackage{stfloats}
\usepackage{url}
\usepackage{verbatim}
\usepackage{graphicx}
\usepackage{subfigure}
\usepackage{cite}
\usepackage[breaklinks,colorlinks,
linkcolor=red,
anchorcolor=blue,
citecolor=blue
]{hyperref}
\usepackage[utf8]{inputenc} % allow utf-8 input
\usepackage[T1]{fontenc}    % use 8-bit T1 fonts
\usepackage{url}            % simple URL typesetting
\usepackage{booktabs}       % professional-quality tables
\usepackage{nicefrac}       % compact symbols for 1/2, etc.
\usepackage{microtype}      % microtypography
\usepackage{longfbox}
\usepackage{amssymb}
\usepackage{amsmath}   
\usepackage{mathtools}
\usepackage{amsthm}
\usepackage{bbm}
\usepackage{wrapfig}
\usepackage{multirow}
\usepackage{makecell}
\usepackage{parnotes}
\usepackage{float}
\usepackage[makeroom]{cancel}
\usepackage{chngcntr}
\usepackage{apptools}
\usepackage{enumitem}
\usepackage{framed,enumitem} 
\usepackage[most]{tcolorbox}
\usepackage[algo2e,algoruled,boxed,lined,norelsize]{algorithm2e} 

\graphicspath{{figures/}}

\DeclareMathOperator{\E}{\mathbb{E}} 
\DeclareMathOperator{\vol}{{Vol}} 
\newcommand{\norm}[1]{\left\lVert#1\right\rVert}

    \newcommand*{\tran}{^{\mkern-1.5mu\mathsf{T}}}
\newcommand{\circled}[1]{\raisebox{.5pt}{\textcircled{\raisebox{-.9pt} {#1}}}}

\newtheorem{theorem}{Theorem}

\newtheorem{lemma}{Lemma}

\newtheorem{definition}{Definition}

\definecolor{myblue}{rgb}{0 ,0, 1}
\begin{document}
\title{Safe MPC Alignment with Human Directional Feedback}

\author{Zhixian Xie,\quad Wenlong Zhang,\quad Yi Ren,\quad Zhaoran Wang,\quad George J. Pappas,\quad Wanxin Jin

        % <-this % stops a space
\thanks{Zhixian Xie and Wanxin Jin are with the Intelligent Robotics and Interactive Systems (IRIS) Lab in the School for Engineering of Matter, Transport and Energy at Arizona State University. Wanxin Jin is the corresponding author. Emails: zxxie@asu.edu and wanxinjin@gmail.com. 

Wenlong Zhang is with  the Polytechnic School at Arizona State University. Email: wenlong.zhang@asu.edu.

Yi Ren is with the School for Engineering of Matter, Transport and Energy at Arizona State University. Email: yiren@asu.edu.

Zhaoran Wang is with the Departments of Industrial Engineering \& Management Sciences, Northwestern University. Email: zhaoranwang@gmail.com.

George J. Pappas is with the Department of Electrical and Systems Engineering at the University of Pennsylvania. Email: pappsg@seas.upenn.edu.

This work involved human subjects or animals in its research. Approval of all ethical and experimental procedures and protocols was granted by ASU’s Institutional Review Board under Application No. STUDY00020060, and performed in the line with Title 45 of the Code of Federal Regulations, Part 46 (45 CFR 46).

We acknowledge the funding support from ABRC grant RFGA2024-022-013.
}% <-this % stops a space
}

% % \IEEEpubid{0000--0000/00\$00.00~\copyright~2021 IEEE}
% % column for its text to clear the IEEEpubid mark.

\maketitle
\markboth{This is a preprint. The published version can be accessed at IEEE Transactions on Robotics.}{My Header}

\begin{abstract}
In safety-critical robot planning or control, manually specifying safety constraints or learning them from demonstrations can be challenging. In this article, we propose a certifiable alignment method for a robot to learn a safety constraint in its model predictive control (MPC) policy from human online directional feedback. To our knowledge, it is the first method to learn safety constraints from human feedback.  The proposed method is based on an empirical observation: human directional feedback, when available, tends to guide the robot toward safer regions.
 The method only requires the direction of human feedback to update the learning hypothesis space. It is certifiable, providing an upper bound on the total number of human feedback in the case of successful learning, or declaring the hypothesis misspecification, i.e., the true  safety constraint cannot be found within the specified hypothesis space.
We evaluated the proposed method in numerical examples and user studies with two simulation games. Additionally, we  tested the proposed method on a real-world Franka robot arm performing mobile water-pouring tasks. The results demonstrate the efficacy and efficiency of our method, showing that it enables a robot to successfully learn safety constraints with a small handful  (tens) of human directional corrections.
\end{abstract}

\begin{IEEEkeywords}
Learning from Human Feedback, Constraint Inference, Model Predictive Control, Human-Robot Interaction.
\end{IEEEkeywords}

\section{Introduction}
In robot policy design, safety is always the top priority,  especially in safety-critical applications, such as robots operating near humans, where safety violations can lead to irreversible losses.
While extensive work in safe control  \cite{wabersich2023data} and reinforcement learning \cite{gu2022review}  focuses on deducing robot    policy given safety constraints, less attention has been paid to how to specify the safety constraints in the first place. In many applications, specifying proper safety constraints can be challenging. First, \textit{explicit  boundary for safe robot behavior can be hard to model}. For example, when delivering food in a crowded restaurant, a service robot must balance a tray to avoid spills; this requires constraints on robot speed, food weight distributions, and interaction with its surroundings, all of which   can be complex to model.
Second,  \textit{user's perception of safety may vary}. In the same example, some diner may feel secure with the robot staying within a foot when delivering,  while others may feel uncomfortable with such proximity.
% lane-switching of an autonomous vehicle: \textcolor{myblue}{The more aggressive the lane-switching policy is, the more likely that the vehicle will crash. Different people may have different tolerance threshold for danger on this likelihood of crashing.}

\begin{figure}
  \centering
  \includegraphics[width=0.7\linewidth]{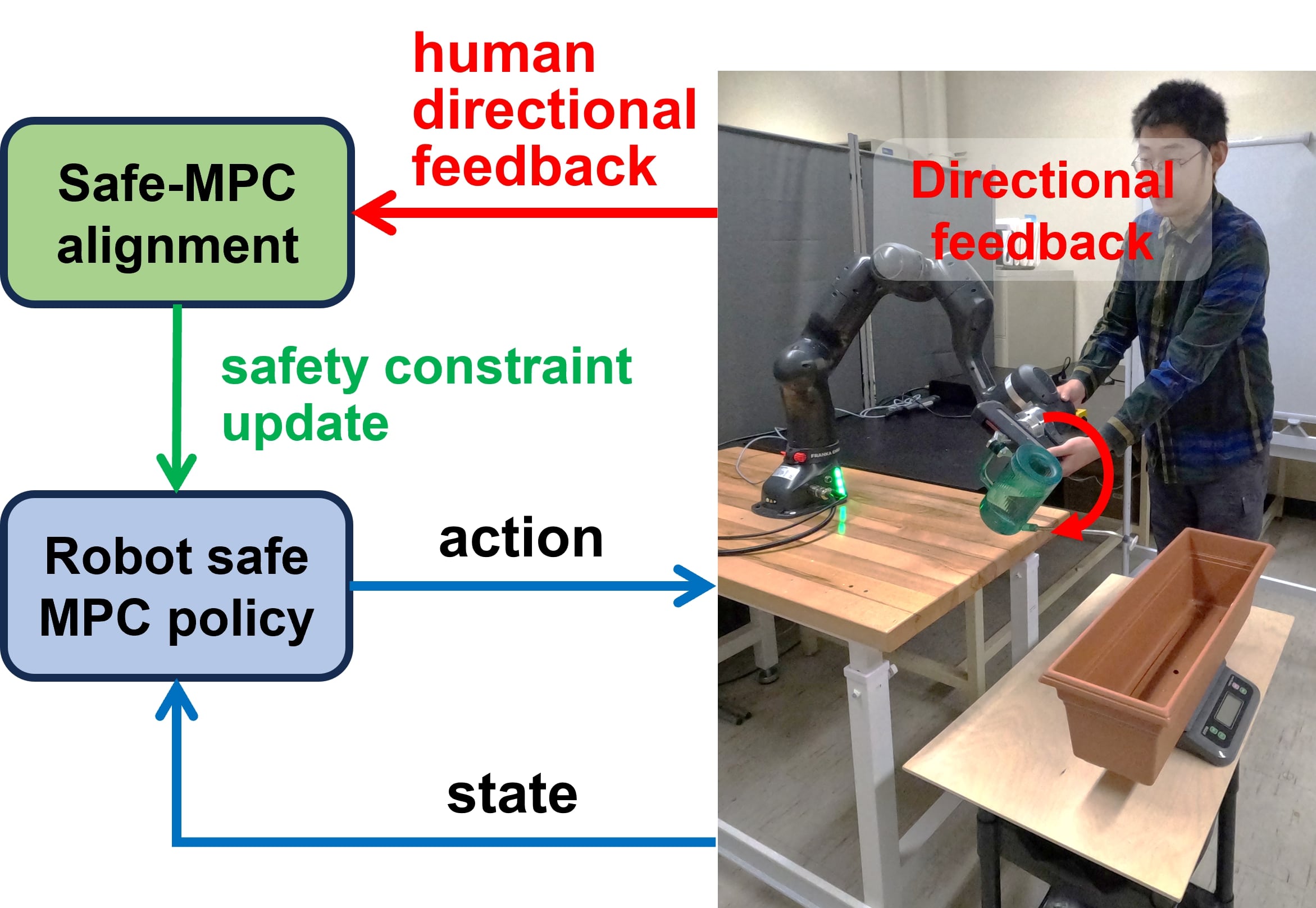}
  \caption{Illustration of the proposed safe MPC alignment. Our method online updates the safety constraint using human directional feedback, such that the robot eventually aligns the safe MPC policy with human intent for the safety-critical task.} \label{fig.holistic}
  \vspace{-15pt}
\end{figure}

The challenges in safety specification lead to two questions. Q1: can a robot  learn  safety constraints to align with human intent in an \emph{interactive} manner? Q2: can such a learning process be \emph{user-data efficient}?
Recent work \cite{chou2020learning_1, lindner2023learning, scobee2019maximum}  explored learning safety constraints from human demonstrations. However, acquiring safety-informative demonstrations could be difficult. Learning from human preference \cite{christiano2017deep, Sadigh2017ActivePL, ibarz2018reward, pmlr-v139-lee21i, hejna2023few, 10160439, pmlr-v164-myers22a} offers a new paradigm, where a user's preference of one trajectory over the other is used for robot learning. The preference, however,  can be difficult for a human to make, especially in early learning stages.  \cite{bajcsy2017learning, losey2018including, bajcsy2018learning, zhang2019learning, li2021learning, losey2022physical, mehta2023unified} proposed learning from human physical corrections. While promising,  \cite{jin2022learning} suggests those methods are potentially sensitive to correction magnitude:  corrections can lead to overshooting and inefficiency of learning. All those feedback-driven learning methods primarily focus on learning rewards instead of  safety constraints.

This paper aims to answer the previous two questions by proposing a new paradigm for a robot to efficiently learn safety constraints in its model predictive control (MPC) policy from human feedback. We name the method \textit{safe MPC alignment}. As  in Fig. \ref{fig.holistic}, a   user, observing the robot's potentially unsafe behavior, provides online corrections to the robot. Upon receiving human corrections, the robot immediately updates the safety constraints in its  MPC policy and continues with the updated policy. Our method only requires/utilizes \emph{direction} of a human correction \cite{jin2022learning}. 
For example, when a human user corrects an autonomous car, he/she only needs to give the command of `turn left' or `turn right', without dictating how much the car needs to turn. Directional corrections permit a consistently large and flexible correction space. Importantly, the proposed method guarantees the exponential convergence with finite rounds of human corrections, or outputs failure with certified hypothesis misspecification.  Specifically, we highlight our main contributions below.

(I) We propose the method of \textit{safe MPC alignment} from human directional correction. It enables a robot to learn a safe  MPC policy  from online human directional corrections. To our knowledge, it is the first work for interactive learning of safety constraints from human feedback. 

(II) We show that \textit{safe MPC alignment} is  certifiable: it guarantees exponential convergence with a finite number of human corrections and enables verification when the true safety constraint falls outside the specified hypothesis space.

(III) Extensive user studies have been conducted to validate the proposed method, including a real-world experiment on the Franka robot arm. The results show the effectiveness and efficiency of our method.

The remaining paper is organized as follows. Section II  reviews the related work, and Section III formulates the safe MPC alignment problem. The main method is presented in Section IV. Section V gives the implementation details and  the main results. Sections VI-VIII provide the experiments and user studies. Discussions and Conclusions are  in Section IX. 

\section{Related Work}
\subsection{Learning from  Demonstrations}
Under the general theme of learning from human inputs, the proposed method falls into the category of learning from demonstrations, where costs/rewards are learned from expert data using inverse reinforcement learning (IRL)   \cite{ng2000algorithms}. IRL methods typically run offline:  human demonstrations are first collected, and then reward/cost is learned \cite{ng2000algorithms,abbeel2004apprenticeship,ratliff2006maximum,ziebart2008maximum,mombaur2010human,englert2017inverse}. As pointed by \cite{ravichandar2020recent},  acquiring  demonstration data can be difficult in practice, due to either  the high degrees of freedom of a robot, or the lack of knowledge about the desired robot behavior  \cite{jain2015learning}.

In this work, we focus on interactive learning of safety constraints. Different from learning from demonstrations, the proposed method only requires a user to incrementally improve the current robot motion by providing directional feedback.

\vspace{-5pt}
\subsection{Learning from Human Feedback}
This line typically focuses on learning rewards/costs from human feedback  in forms of preferences and corrections. Human preference is typically given as user's ranking of two or multiple robot trajectories \cite{christiano2017deep, Sadigh2017ActivePL, ibarz2018reward, pmlr-v139-lee21i, hejna2023few, 10160439,pmlr-v164-myers22a}, which is modeled using the Bradley-Terry formulation \cite{bradley1952rank}.
%The recent variant seeks to directly learn a policy from human preference \cite{hejna2023contrastive}.
Human corrections require a user to intervene and correct the robot at runtime, from which a intended  reward/cost can be inferred.  Human corrections can  include language command \cite{karamcheti2022shared, cui2023no, srivastava2023generating, zha23droc, shi2024yell} and physical corrections \cite{bajcsy2017learning,bajcsy2018learning,losey2022physical}. Since human corrections can be sparse in time, many methods \cite{bajcsy2017learning,losey2018including,bajcsy2018learning,zhang2019learning,li2021learning,losey2022physical,mehta2023unified} use trajectory deformation methods to interpolate sparse corrections as a continuous trajectory. Some strategies are further proposed to deal with  learning robustness and uncertainty, such as updating one feature at a time \cite{bajcsy2018learning} or formulating the problem in the  Kalman Filter  framework \cite{losey2018including}.

Our work belongs to the category of learning from human corrections. But instead of learning rewards/costs, we focus on learning safety constraints from human corrections, which has not been addressed before. Methodologically,  we propose a certifiable method, which guarantees the  learning convergence with finite human corrections, and certifies the hypothesis misspecification. Those  results are missing in existing methods.

\vspace{-5pt}
\subsection{Learning  Constraints from Demonstrations}
An increasing number of work focuses on learning safety constraints from demonstrations. One straightforward idea is to cast the problem as a binary classification. \cite{chou2020learning_1, chou2020learning_2} consider the trajectories of lower cost as the unsafe samples. \cite{gaurav2023learning} proposes adjusting the constraints toward decreasing the safety margin of feasible policies.
Safety constraints can also be learned via mixed integer programming \cite{chou2020learning_1,chou2020learning_2}, pessimistic adjustment \cite{gaurav2023learning}, and no-regret algorithm \cite{kim2024learning}. Those methods typically require collecting negative (unsafe) samples, which can be challenging in real robot systems.
Another line of work attempts to learn safety constraints only from safe demonstration. For example, \cite{chou2020learning_3,chou2021uncertainty,chou2022gaussian} infer the safety constraints by exploiting  the optimality conditions of demonstrations. Different probabilistic formulations are also used, such as maximum likelihood \cite{scobee2019maximum,stocking2022maximum,malik2021inverse},  causal entropy \cite{mcpherson2021maximum,baert2023maximum}, and Bayesian inference \cite{park2020inferring,papadimitriou2022bayesian}. Among them, safety constraints can be represented as a set of safe states \cite{chou2020learning_3}, the belief of  constraints \cite{chou2021uncertainty},  Gaussian process \cite{chou2022gaussian}, or a convex hull in policy space \cite{lindner2023learning}.

Using demonstrations, the vast majority of the above methods formulate the constraint learning  in  offline settings, and the learning process usually involves repetitive policy optimization. Those methods can be difficult to be deployed online, where a robot is expected to quickly update its safety estimate and respond to human feedback. To fill this gap, we will develop the first method, to our best knowledge, to enable online learning of safe constraints from human feedback.

\section{Preliminary and Problem Formulation}

\subsection{Safe Model Predictive Control}
Consider a robot equipped with a safe model predictive control (Safe MPC) policy defined below, where the safety constraint is learnable and parameterized by $\boldsymbol{\theta}\in\mathbb{R}^r$. 
\begin{equation}\label{equ.robot_mpc}
 \begin{aligned}
          \min_{\boldsymbol{u}_{0:T-1}}  & \quad J( \boldsymbol{u}_{0:T-1}; \boldsymbol{x}_0)=\sum\nolimits_{t=0}^{T-1} c(\boldsymbol{x}_t, \boldsymbol{u}_t)+h(\boldsymbol{x}_{T})\\
   \text{subject to} &\quad \boldsymbol{x}_{t+1}=\boldsymbol{f}(\boldsymbol{x}_t, \boldsymbol{u}_t) \quad \text{given}\quad \boldsymbol{x}_0\\
    & \quad {g}_{\boldsymbol{\theta}}(\boldsymbol{x}_{0:T}, \boldsymbol{u}_{0:T-1}) \leq 0
 \end{aligned}.
\end{equation}
Here, $\boldsymbol{x}_t\in\mathbb{R}^n$ and $\boldsymbol{u}_t\in\mathbb{R}^m$ are robot state and action, respectively;  $t=0,1,2,...,T$ is the  time step over a short prediction horizon $T$, starting from a given $\boldsymbol{x}_0$. Define $\boldsymbol{x}_{0:T}=(\boldsymbol{x}_0, \boldsymbol{x}_1,..., \boldsymbol{x}_T)$ and $\boldsymbol{u}_{0:T-1}=(\boldsymbol{u}_0, \boldsymbol{u}_1,..., \boldsymbol{u}_{T-1})$; $J(\cdot)$ is the given  cost function including the running cost  ${c}(\cdot)$ and final cost  $h(\cdot)$; $\boldsymbol{f}(\cdot, \cdot)$ is the system dynamics; 
${g}_{\boldsymbol{\theta}}(\cdot)$ is the safety constraint parameterized
by $\boldsymbol{\theta}\in\mathbb{R}^r$.
Given starting state $\boldsymbol{x}_0$ and   parameter $\boldsymbol{\theta}$, the optimal plan  (solution) to (\ref{equ.robot_mpc}), denoted as $\boldsymbol{\xi}_{\boldsymbol{\theta}}=\{\boldsymbol{u}^{\boldsymbol{\theta}}_{0:T-1}, \boldsymbol{x}^{\boldsymbol{\theta}}_{0:T}\}$, is  an implicit function of $\boldsymbol{\theta}$.

In a receding horizon implementation,  $\boldsymbol{x}_0$ is set as the real robot's current state (estimation)  $\boldsymbol{{x}}_k^{\text{real}}$, where $k=0,1,2,...,$ is  the global rollout time step of the Safe MPC policy. Only the first action $\boldsymbol{u}^{\boldsymbol{\theta}}_0$  in  each $\boldsymbol{\xi}_{\boldsymbol{\theta}}$  is applied to the robot to move the robot to the next state $\boldsymbol{x}^{\text{real}}_{k+1}$. This process repeats at each rollout time step $k=0,1,...$. Such a receding horizon implementation leads to a closed-loop  policy, denoted as $\boldsymbol{\pi}_{\boldsymbol{\theta}}$, mapping from robot current state $\boldsymbol{x}_k^{\text{real}}$ to its action $\boldsymbol{u}^{\boldsymbol{\theta}}_0$.

\subsection{Safe Asymptotic Approximation of Safe MPC}
We present a penalty-based approximation to the  Safe MPC in (\ref{equ.robot_mpc}). Define trajectory variable $\boldsymbol \xi =\{ \boldsymbol u_{0:T-1}, \boldsymbol x_{0:T} \}$ and
 the following penalty objective:
\begin{equation}
        B(\boldsymbol \xi, \boldsymbol \theta)=J(\boldsymbol{\xi}) - \gamma \ln (-g_{\boldsymbol{\theta}}(\boldsymbol \xi)).
    \label{equ.barrier}
\end{equation}
Here, $B(\boldsymbol \xi, \boldsymbol \theta)$ includes the cost $J(\boldsymbol{\xi})$  and safety constraint $g_{\boldsymbol{\theta}}(\boldsymbol \xi)$ in a logarithmic barrier transformation. $\gamma>0$ is  the barrier parameter, controlling the weight of safety penalty $-\ln (-g_{\boldsymbol{\theta}}(\boldsymbol \xi))$ relative the  cost. With $ B(\boldsymbol \xi, \boldsymbol \theta)$, one can establish the following penalty-based  MPC
\begin{equation}\label{equ.robot_mpc_approx}
          \min_{\boldsymbol{u}_{0:T-1}}   \quad B(\boldsymbol \xi, \boldsymbol \theta)\qquad
   \text{subject to} \quad \boldsymbol{x}_{t+1}=\boldsymbol{f}(\boldsymbol{x}_t, \boldsymbol{u}_t).
\end{equation}
We denote the optimal plan  to   (\ref{equ.robot_mpc_approx}) as $\boldsymbol \xi_{\boldsymbol \theta}^\gamma=\{\boldsymbol{u}_{0:T-1}^{\boldsymbol{\theta}, \gamma}, \boldsymbol{x}_{0:T-1}^{\boldsymbol{\theta}, \gamma}\}$. The following result  from  \cite{jin2021safe} states  the penalty-based MPC (\ref{equ.robot_mpc_approx})  is an asymptotic, safe approximation to  Safe MPC (\ref{equ.robot_mpc}).

\begin{lemma}(Theorem 2 in \cite{jin2021safe}) If the  Safe MPC (\ref{equ.robot_mpc}) satisfies certain mild conditions (i.e., local second-order condition and strict complementarity), then for any small $\gamma>0$, the penalty-based MPC (\ref{equ.robot_mpc_approx}) has a (local) solution $\boldsymbol \xi_{\boldsymbol \theta}^\gamma$, and $\boldsymbol \xi_{\boldsymbol \theta}^\gamma$ is safe with respect to the original constraint, i.e.,  $g_{\boldsymbol \theta}(\boldsymbol \xi_{\boldsymbol \theta}^\gamma)<0$, and $\boldsymbol \xi_{\boldsymbol \theta}^\gamma$ satisifies the first-order condition of (\ref{equ.robot_mpc_approx}). Furthermore, $\boldsymbol \xi_{\boldsymbol \theta}^\gamma$ is asymptotically approximate $\xi_{\boldsymbol \theta}$:   $\boldsymbol \xi_{\boldsymbol \theta}^\gamma\rightarrow\boldsymbol \xi_{\boldsymbol \theta}$ as $\gamma\rightarrow0$.
\label{lemma sol}
\end{lemma}

The above lemma asserts that solving the Safe MPC (\ref{equ.robot_mpc}) can be  replaced by solving penalty-based MPC (\ref{equ.robot_mpc_approx}). The solution approximation is controlled  by $\gamma$ for arbitrary accuracy:  $\boldsymbol \xi_{\boldsymbol \theta}^\gamma\rightarrow\boldsymbol \xi_{\boldsymbol \theta}$ as $\gamma\rightarrow0$. A good approximation is achieved by a smaller $\gamma$. On the contrary, if $\gamma$ is larger, penalty-based MPC  will lead to more safety-conservative solution $\boldsymbol \xi_{\boldsymbol \theta}^\gamma$. In the development below, we fix the parameter $\gamma$ in (\ref{equ.robot_mpc_approx}). We denote the penalty-based MPC policy in (\ref{equ.robot_mpc_approx}) as $\boldsymbol{\pi}^\gamma_{\theta}$.

\subsection{Problem Formulation \label{sec.prob_procedure}}

\subsubsection{Representation of Human Intended Safety}
Consider a human is supervising the robot's behavior.  In Safe  MPC (\ref{equ.robot_mpc}), we hypothesize that user-intended safety can be represented by a subset of the trajectory space:
\begin{equation}\label{equ.user_safety}
\begin{split}
\mathcal{S}(\boldsymbol{\theta}_H)=\{
(\boldsymbol{x}_{0:T}, \boldsymbol{u}_{0:T-1})
\,\,|\,\,
\ {g}_{\boldsymbol{\theta}_H}(\boldsymbol{x}_{0:T}, \boldsymbol{u}_{0:T-1}) \leq 0\}.
\end{split}
\end{equation}
Here,  $\boldsymbol{\theta}_H$  is from a set $\boldsymbol{\theta}_H\in\Theta_H$, encoding all possible user intended safety. $\Theta_H$ is implicit and unknown to the robot. For the robot policy $\boldsymbol{\pi}_{\boldsymbol{\theta}}^\gamma$, if   $\boldsymbol{\theta}\in\Theta_H$, the user considers the planned robot motion to be safe; otherwise if $\boldsymbol{\theta}\notin\Theta_H$,  the robot motion by $\boldsymbol{\pi}_{\boldsymbol{\theta}}^\gamma$  might not be safe at certain time.
We use an implicit set $\Theta_H \in\mathbb{R}^r$, instead of a single point $\boldsymbol{\theta}_H$, to capture the potential variance of user intents. We suppose that a volume measure of $\Theta_H$, $\vol(\Theta_H)$, is non-zero,  and a user intent $\boldsymbol{\theta}_H\in\Theta_H$ can vary at different steps and contexts, but always within $\Theta_H$.

\subsubsection{Human Directional Correction}
As shown in Fig. \ref{fig-update}, we consider human correcting a robot online while the robot is executing  $\boldsymbol{\pi}_{\boldsymbol{\theta}}^\gamma$. Specifically, let $i$ be the index of learning update, and suppose  robot current policy is $\boldsymbol{\pi}_{\boldsymbol{\theta}_{i}}^\gamma$. During rollout of $\boldsymbol{\pi}_{\boldsymbol{\theta}_{i}}^\gamma$,
the robot receives a human correction $\boldsymbol{a}_{k_i}\in \mathbb{R}^{m}$ at the  rollout time step $k_{i}$.
Here, we do not pose any assumption on correction time $k_{i}$ and its duration due to the sparsity or intermittency of human feedback. In other words, the human can apply one or multiple corrections  at any time  as  necessary. But we assume  human correction $\boldsymbol{a}_{k_i}$ is in the robot action space. For example, in autonomous driving where the robot action space is steering angle and acceleration, human can directly apply feedback with the steering wheel and pedals. In some other applications, the action space correction can be achieved via a specific user interface. For example, to correct a torque-controlled robot arm, human's correction on the end-effector can be mapped into joint torque via the Jacobian transpose. The reason why we only consider the action space correction is because first, corrections in input spaces may be easier to implement, as shown in our later experiments, and second, corrections in state spaces can be infeasible for the under-actuated robot system.

After receiving a human correction $\boldsymbol{a}_{k_i}$, the robot updates the safety parameter $\boldsymbol{\theta}_i\rightarrow\boldsymbol{\theta}_{i+1}$ in (\ref{equ.robot_mpc}) for its Safe MPC. Then, the robot continues with  $\boldsymbol{\pi}_{\boldsymbol{\theta}_{i+1}}^\gamma$ until the next possible user correction. This process repeats until the convergence of the policy update or the user is satisfied with the robot's motion. 

\begin{figure}
  \centering
  \includegraphics[totalheight=1.7in]{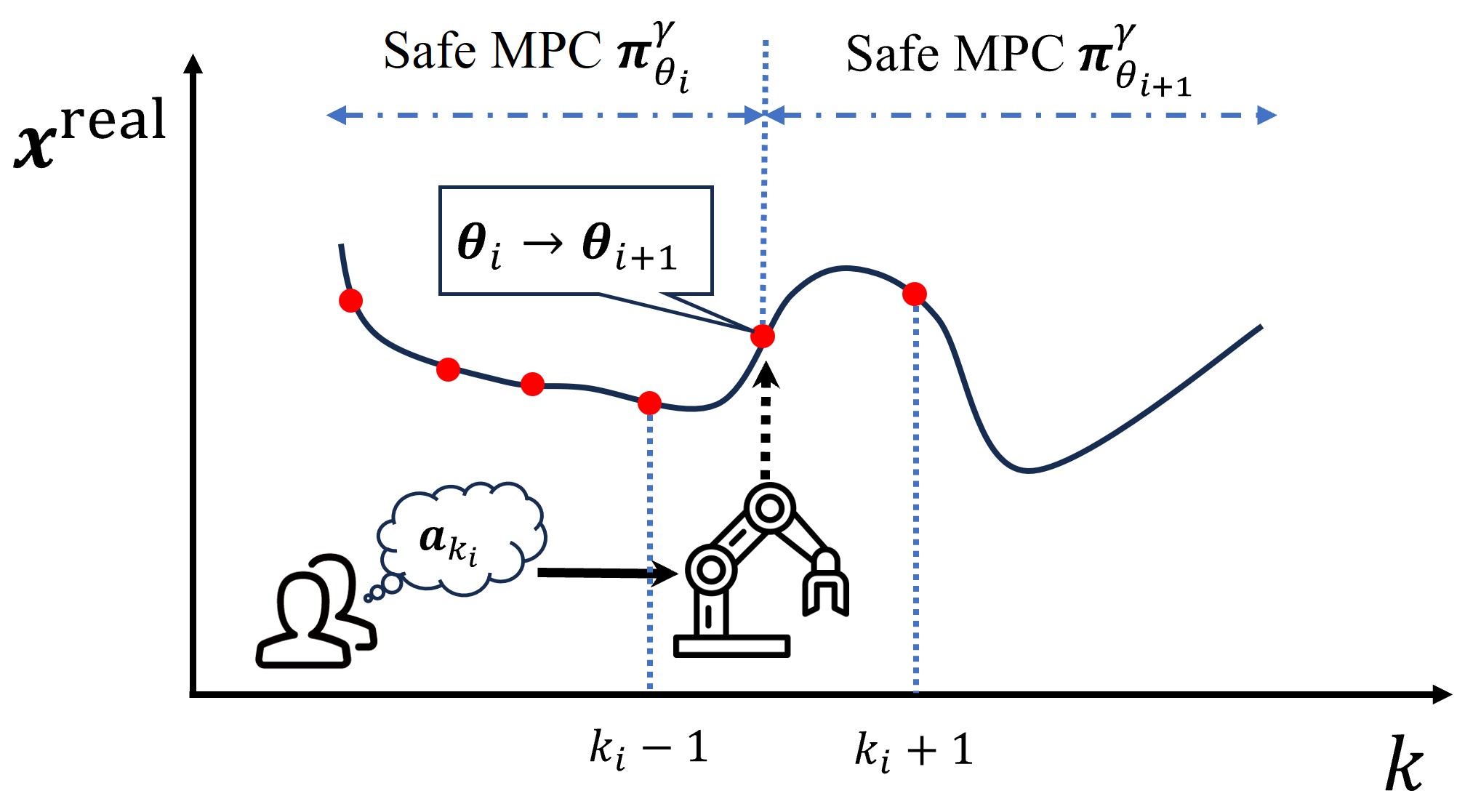}
  \caption{Update of Safe MPC policy with human directional corrections.} \label{fig-update}
\end{figure}

% \subsubsection{User Directional Corrections}
% In the above learning process, we consider that a human user only provides directional correction $\boldsymbol{a}$, i.e., the magnitude $\norm{\boldsymbol{a}}$ does not matter (theoretical justification will be given later). 

While the correction $\boldsymbol{a}$ given by human is real-valued, only its direction is used in the learning process. This means that we will use the sign (for scalar signal) or its normalized unit vector (for vector-valued correction signal) for learning update. The magnitude of the correction is NOT taken into account in the learning process. 
As we will show later, directional feedback provides a consistently large and flexible correction space, effectively overcoming the over-shooting correction, compared to magnitude feedback.

% Also, we assume that the user feedback is applied in the action space.  For some robotic systems, this assumption can be readily met, such as turning the driving wheels for autonomous vehicles. In other cases, the correction can be applied through a dedicated (hardware or algorithmic) interface to the action space; such as contact forces from end-effector to joint torque, and teleoperation device for remote control. In those cases, $\boldsymbol{a}$ is the correction signal after the interface. The reason why we only consider the correction in action space is that otherwise the corrections in robot state space can lead to dynamics unfeasibility or uncontrollability for under-actuated robotic systems  \cite{jin2022learning,tedrake2009underactuated}.

\subsubsection{Robot  Operation Safety and Emergency Stop}
\label{sec.human_supervise}
In the aforementioned learning procedure, the robot will respond to human correction by executing the updated safe MPC policy $\pi_{\boldsymbol{\theta}}^{\gamma}$. However, it is possible that the robot may still exit the user-intended safety. In such cases, the human may trigger an emergency stop to stop the robot and potentially restart it from a desired safety region. In such case, the executed motion before the emergency stop is assumed safe. In other words, until an emergency stop is triggered, human correction typically occur when the robot is still in the intended safety set; otherwise, the correction would be useless. For instance, correcting a vehicle after it is crashed is futile.

\subsubsection{Parametric Safety Constraints}\label{sec.safety_param}
In  (\ref{equ.user_safety}), the safety set is defined as a zero sublevel set of  ${g}_{\boldsymbol{\theta}}$. If ${g}_{\boldsymbol{\theta}}$ has sufficient representation power,  the safety set can represent a complex safety region. In the following, we focus on a weighted-features parameterization of ${g}_{\boldsymbol{\theta}}$, which has been commonly used by the prior work on learning cost/reward \cite{lindner2023learning, jin2022learning, zhang2023dual, biyik2022aprel, bobu2020quantifying}. Specifically, 
\begin{multline}
    \label{equ.safey_param}
    {g}_{\boldsymbol{\theta}}(\boldsymbol{\xi})=\phi_0(\boldsymbol{\xi})+\theta_1\phi_1(\boldsymbol{\xi})
    +\theta_2\phi_2(\boldsymbol{\xi})+\cdots+\theta_r\phi_r(\boldsymbol{\xi}),
\end{multline}
with $\boldsymbol{\theta}=[\theta_1, \theta_2, ...\theta_r]\tran\in \mathbb{R}^r$ and $\boldsymbol{\phi}=[\phi_1,\phi_2, \dots, \phi_r]\tran$  the learnable weight and feature vectors, respectively.  We fix the weight of $\phi_0$ as unit, because otherwise,  $\boldsymbol{\theta}$ has scaling ambiguity, i.e., any scaled  $c\cdot\boldsymbol{\theta}$ with $c>0$ lead to the same constraint. In this work, we assume that the  features in $\boldsymbol{\phi}$ are given as a priori. These features, defined in system states (e.g., position, velocity) and actions, can be either manually defined or obtained from pre-trained neural representations, leveraging prior knowledge of the environment and tasks, as demonstrated in our experiments. The feature selection itself is beyond the scope of this paper and is left for future work. A more detailed discussion of this assumption  and its limitations is  in Section~\ref{sec.feature_selection_limitation}.

\subsubsection{Problem of Safe MPC Alignment}
Following the procedure of human correction and robot policy update in Section \ref{sec.prob_procedure}, we aim to develop a policy update rule: $\boldsymbol{\theta}_i\rightarrow\boldsymbol{\theta}_{i+1}$ for each received human  correction, such that $\boldsymbol{\theta}_i\rightarrow\boldsymbol{\theta}_H \in \Theta_H$, as $i=1,2,3,\dots,$. i.e., safety constraints in the robot's safe MPC policy will be aligned with human intent.

\section{Overview of Safe MPC Alignment}
\subsection{Hypothesis on Human Directional Corrections \label{sec.idea_hypothesis}}
Informally, we have the following intuitions on human feedback. First, human corrections, on average, improve the robot's safety. Second, human correction  is typically proactive and happens before the robot leaves the human-intended safety set. In other words, human correction is made when the robot is still in  $\mathcal{S}(\boldsymbol{\theta}_H)$ in (\ref{equ.user_safety}). For example, it is useless to correct a vehicle after its crash. Third, human feedback is typically safety-margin aware: as a robot approaches the safety boundary, human users tend to prioritize safety over the task cost. For instance, when a vehicle moves closer to obstacles, humans may instinctively shift their focus to avoiding collisions rather than pursuing the original target.
In the following, we will formalize those intuitions.

Without loss of generality, suppose the motion plan  $\boldsymbol{\xi}_{\boldsymbol{\theta}}^\gamma$  is a local solution to the penalty MPC (\ref{equ.robot_mpc_approx}), and when the robot execute the first action $\boldsymbol{u}_0^{\boldsymbol{\theta},\gamma}$ (receding horizon implementation),  a human  applies a directional  correction $\boldsymbol{a}$ in robot action space. In this subsection, the subscription $k_i$ is omitted for notation simplicity. Define
\begin{equation}\label{equ.a_extension}
    \mathbf{a}=(\boldsymbol a\tran ,0,0,...) \in \mathbb{R}^{mT}.
\end{equation}
We assume there is a nonzero volume set $\bar\Theta_H$ (containing a ball) in user implicit set $\Theta_H$, i.e.,  $\bar\Theta_H\subseteq\Theta_H$, such that for any $\boldsymbol{\theta}_H \in \bar\Theta_H$, the human correction $\mathbf{a}$, on average, indicates a direction of lowering the penalty value  $B(\boldsymbol{\xi}_{\boldsymbol{\theta}}^\gamma, \boldsymbol \theta_H)$.  Concretely, it means $\mathbf{a}$ on averages points in the direction of the negative gradient  of $B(\boldsymbol \xi, \boldsymbol \theta_H)$  at $\boldsymbol{\xi}_{\boldsymbol{\theta}}^\gamma$, i.e.,
\begin{multline}
    \forall \boldsymbol{\theta}_H \in \bar\Theta_H,\ 
    \E_{\mathbf{a}}\left\langle \mathbf{a},-\nabla B(\boldsymbol\xi_{\boldsymbol \theta}^\gamma, \boldsymbol \theta_H) \right\rangle=\\
    \left\langle \mathbf{\bar{a}},-\nabla B(\boldsymbol\xi_{\boldsymbol \theta}^\gamma, \boldsymbol \theta_H) \right\rangle \geq 0.
    \label{equ.core-neq}
\end{multline}
Here, $\left\langle \cdot,\cdot \right\rangle$ denotes the  inner product operation, and ${\E}_{\mathbf{a}}(\mathbf{a})=\mathbf{\bar{a}}$ denotes the expectation value of $\mathbf{a}$. $\nabla B(\boldsymbol \xi, \boldsymbol \theta)$ denotes the full gradient of $B$ with respect to $\boldsymbol u_{0:T-1}$ evaluated at trajectory $\boldsymbol \xi=\{\boldsymbol u_{0:T-1}, \boldsymbol x_{0:T}\}$. Similar notations apply to $J$ and $g$ below. Since $\nabla B(\boldsymbol \xi, \boldsymbol \theta)$ is not dependent on $\mathbf{a}$, the expectation can be moved inside the inner product. Because state sequence $\boldsymbol x_{0:T}$ is the rollout of $\boldsymbol u_{0:T-1}$ on dynamics, the differentiation goes through the dynamics equation. 

% \textcolor{myblue}{(\ref{equ.core-neq}) shows that although the human correction $\boldsymbol{a}$ is represented as a real-valued vector, only its direction—not its magnitude—affects the valid region of $\boldsymbol{\theta}_H$ of the inequality. This is because the inequality solution remains the same under any positive scaling of $\boldsymbol{a}$, meaning that the algorithm effectively relies on the sign or direction of the feedback rather than its absolute value.} 
Equation (\ref{equ.core-neq}) shows that although the human correction $\boldsymbol{a}$ is provided as a real-valued vector, only its direction—not its magnitude—affects the inequality.
This brings one advantage of directional correction compared to the magnitude-based correction:  the allowable human correction satisfying (\ref{equ.core-neq}) always account for half of the action space, making directional corrections more accessible for human users, as shown in Fig. \ref{fig:comparison}. Recall that $B(\boldsymbol \xi, \boldsymbol \theta)=J(\boldsymbol{\xi}) - \gamma \ln (-g_{\boldsymbol{\theta}}(\boldsymbol \xi))$, by extension of the gradient, (\ref{equ.core-neq}) becomes
\begin{multline}
\label{equ.expanded_constraint}
\forall \boldsymbol{\theta}_H \in \bar\Theta_H, \left< \mathbf{\bar{a}},-\nabla J(\boldsymbol\xi_{\boldsymbol \theta}^\gamma) \right> +  \\
\gamma\frac{1}{- g_{\boldsymbol{\theta}_H}(\boldsymbol\xi_{\boldsymbol \theta}^\gamma)}\left< \mathbf{\bar{a}},-\nabla g_{\boldsymbol \theta_H}(\boldsymbol\xi_{\boldsymbol \theta}^\gamma) \right> \geq 0.
\end{multline}

\begin{figure}
\centering
\includegraphics[width=0.20\textwidth]{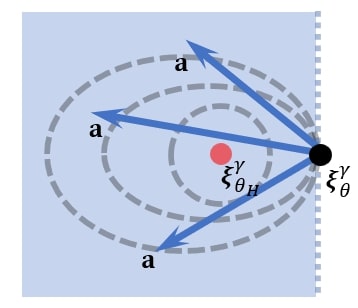}
\caption{Illustration of directional correction in the space of robot plan $\boldsymbol \xi$s. The grey circles represent the level contours of $B(\boldsymbol \xi, \boldsymbol \theta_H)$. The black dot stands for the current motion plan $\boldsymbol\xi_{\boldsymbol \theta}^\gamma$ and the red point stands for optimal plan  $\boldsymbol\xi_{\boldsymbol{\theta}_H}^\gamma$ with $\theta_H$. The blue region stands for the half-space of all possible human correction under our assumption (\ref{equ.core-neq}). Some samples of human corrections are shown by blue arrows.}

\label{fig:comparison}
\end{figure}

\noindent
For the above equation, we have the following interpretations.

(I) The first and second terms of (\ref{equ.expanded_constraint}) correspond to human correction toward lowering the cost and improving the safety of $\boldsymbol{\xi}_{\boldsymbol{\theta}}^{\gamma}$, respectively. 
$\gamma$ weighs between two improvements.

(II) If we interpret the value of $-{g_{\boldsymbol{\theta}_H}(\boldsymbol\xi_{\boldsymbol \theta}^\gamma)}$ as the safety margin of the plan  $\boldsymbol\xi_{\boldsymbol \theta}^\gamma$, (\ref{equ.expanded_constraint}) suggests that the human correction is safety-margin aware: the closer the robot motion get to the unsafe region, the more will human improve the robot safety than minimizing the cost. This property is consistent with our previous intuition.

(III)  (\ref{equ.core-neq}) implies 
$B(\boldsymbol\xi_{\boldsymbol \theta}^\gamma , \boldsymbol \theta_H)$ is well-defined in the first place:
\begin{equation}\label{equ.implication_constraint}
    \forall \boldsymbol{\theta}_H \in \bar\Theta_H,\ 
    \quad g_{\boldsymbol{\theta}_H}(\boldsymbol\xi_{\boldsymbol \theta}^\gamma)  < 0.
\end{equation}
This corresponds to the third intuition: a user typically gives correction only when the robot is in user intended safety region. 

% This property is also used to search for $\boldsymbol \theta_H$.

In sum, the two constraints(\ref{equ.core-neq}) and (\ref{equ.implication_constraint})  formalize the intuitions we have observed at the beginning of this subsection. Those constraints  will be used in the robot learning next.

\subsection{Main Algorithm: Learning by Hypothesis Space Cutting}
\label{sec.alg_pipeline}
To find $\boldsymbol \theta_H\in\bar\Theta_H\subseteq\Theta$ for the Safe MPC, our idea is that for each human  correction $i=0,1,2,....$, the robot performs one learning iteration, by updating a hypothesis space $\Theta_i\in\mathbb{R}^r$, a set containing all possible $\bar\Theta_H$.  Initially, the robot may choose a  large box hypothesis space $\Theta_0$:
\begin{equation}\label{equ.hypothesis}
    \Theta_0=\{\boldsymbol \theta \ |\ \boldsymbol{c}_l \leq \boldsymbol \theta \leq  \boldsymbol{c}_u\} \subset \mathbb{R}^r,
\end{equation}
in which $\boldsymbol{c}_l, \boldsymbol{c}_u$ are the vectors of lower and upper bounds of the box space, respectively, and the inequalities are applied element-wise. By saying a $\Theta_0$ large enough, we assume
\begin{equation}\label{equ.initial_hypothesis}
\Bar \Theta_H \subset \Theta_0.
\end{equation}
Later, we will show our method can self-access the misspecification, where (\ref{equ.initial_hypothesis}) is not met.

% We now return to the notation $\boldsymbol{a}_{k_i}$ to explicitly denote the time index of human correction. 
Recall the procedure of human  correction in Section \ref{sec.prob_procedure} and Fig. \ref{fig-update}. 
At $i$-th learning update, $i=1,2,3...$, the robot executes $\boldsymbol{\pi}_{\boldsymbol{\theta}_i}^\gamma$ with $\boldsymbol{\theta}_i\in \Theta_{i-1}$. The user gives a  correction $\boldsymbol{a}_{k_i}$ at  rollout step ${k_i}$. By letting $\mathbf{\bar{a}}_i =(\boldsymbol{a}_{k_i}^T, 0, 0,...)$ as in (\ref{equ.a_extension}), we can  establish the following two constraints  from the previous hypotheses (\ref{equ.core-neq}) and (\ref{equ.implication_constraint}):
\begin{equation}
    % \bar\Theta_H\subset 
    \mathcal{C}_i=\left\{\boldsymbol \theta | \left< \mathbf{\bar{a}}_i,\nabla B(\boldsymbol\xi_{\boldsymbol \theta_i}^\gamma ,\boldsymbol \theta) \right> \leq 0,\,  \,\,    g_{\boldsymbol \theta}(\boldsymbol\xi_{\boldsymbol \theta_i}^\gamma ) <0\right\},
    \label{equ.C_steps}
\end{equation}
which we refer to as the \textit{cutting set}. %Note that the first inequality $\left< \mathbf{\bar{a}}_i,\nabla B(\boldsymbol\xi_{\boldsymbol \theta_i}^\gamma ,\boldsymbol \theta) \right> \leq 0$ is from (\ref{equ.core-neq}). 
Using $\mathcal{C}_i$, one can update the hypothesis space 
\begin{equation}\label{equ.intersection}
    \Theta_{i}=\Theta_{i-1} \cap \mathcal{C}_i.
\end{equation}
The above update can be viewed as using  (\ref{equ.C_steps}) to cut the  hypothesis space $\Theta_{i-1}$ to remove  portions where $\bar\Theta_H$ is impossible. The new hypothesis space $\Theta_{i}$ will be  used for the next iteration.  With the above update, we  have the following result.
\begin{lemma}\label{lemma.bar_theta}
Given $\Bar \Theta_H \subset\Theta_0$, if one follows the update (\ref{equ.intersection}),  then $\Bar \Theta_H \subset\Theta_i$ for all $i=1,2,3,...$.
\end{lemma}
\begin{proof}
    See the proof in Appendix \ref{proof.lemma bar_theta}.
\end{proof}

As  in Fig. \ref{fig.pipeline_cutting},
the procedure of robot learning is to cut and move the hypothesis space based on human directional correction. Therefore, we name our learning method  \textit{Safe MPC Alignment by Hypothesis Space Cutting}. As a summary, we present the three main steps of our algorithm below.
\begin{figure}[h]
\centering
\includegraphics[width=0.33\textwidth]{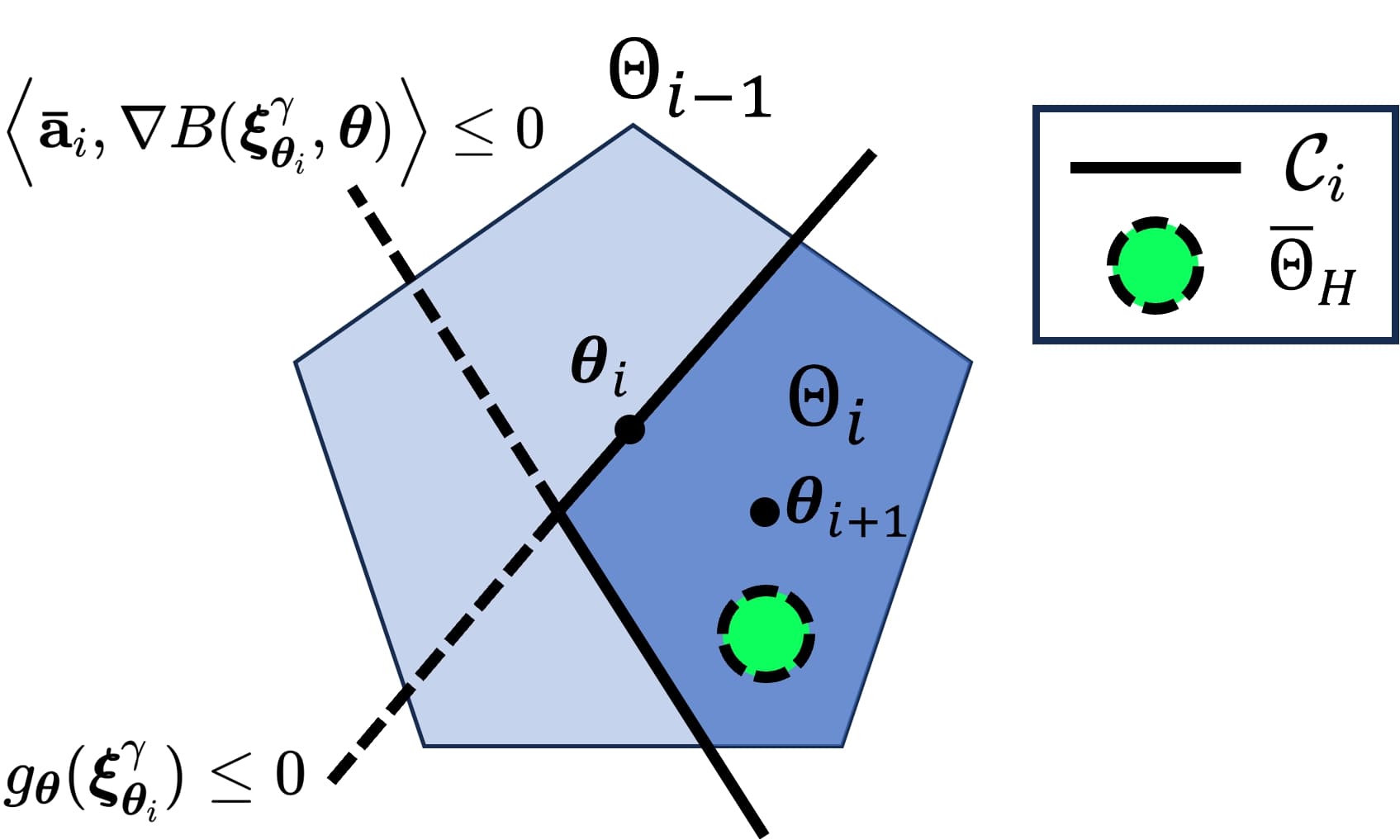}
\caption{The update of hypothesis space (2D example). Two cutting lines are used to cut $\Theta_{i-1}$ to $\Theta_i$. The user intent set $ \bar\Theta_H$ is contained in $\Theta_i$. New $\boldsymbol \theta_{i+1}$ is chosen from $\Theta_i$ for the updated Safe MPC policy.}

\label{fig.pipeline_cutting}
\end{figure}

\begin{tcolorbox}[title=\textbf{\small Safe MPC Alignment by Hypothesis Space Cutting},left=0.5mm, right=0.7mm, bottom=-1.5mm]
\begin{itemize}[leftmargin=4.2em]
    \item[\textbf{Initialize}] $\Theta_0$ that contains  $\bar\Theta_H$.
    Repeat the following three steps for $i=1,2,...,K$ until $\boldsymbol{\theta}_i\in\bar\Theta_H$.
    \item[\textbf{Step 1:}] Robot chooses parameter $\boldsymbol{\theta}_{i}\in\Theta_{i-1}$ for its current Safe MPC  $\boldsymbol{\pi}_{\boldsymbol{\theta}_{i}}^\gamma$.
    \item[\textbf{Step 2:}] Human user applies a directional correction $\boldsymbol{a}_{k_i}$ while robot executes $\boldsymbol{\pi}_{\boldsymbol{\theta}_{i}}^\gamma$ 
    \item [\textbf{Step 3:}] Perform hypothesis space cut using (\ref{equ.intersection}) to get new hypothesis space $\Theta_{i}$.\\
\end{itemize}
\end{tcolorbox}

At $i$th learning update, the robot chooses 
 $\boldsymbol{\theta}_{i}\in\Theta_{i-1}$ for 
 its current Safe MPC $\boldsymbol{\pi}_{\boldsymbol{\theta}_{i}}^\gamma$. While executing $\boldsymbol{\pi}_{\boldsymbol{\theta}_{i}}^\gamma$, the robot  receives human  correction $\boldsymbol{a}_{k_i}$ at the rollout time step $k_i$. This correction leads to a cutting set $\mathcal{C}_i$ (\ref{equ.C_steps}), which will update $\Theta_{i-1}$ to $\Theta_{i}$ in (\ref{equ.intersection}). The above iteration repeats $i=1,2,3,...$.
From (\ref{equ.intersection}), one can easily see that $\Theta_0\supseteq\Theta_1\supseteq...\supseteq\bar\Theta_H$. Given a volume measure, it follows $\vol\Theta_0\geq \vol\Theta_1\geq..., \geq \vol{\bar\Theta_H}$. Thus, the volume of the hypothesis space is at least not increasing after each  correction. It is hoped that the volume of the hypothesis space decreases quickly, such that new $\boldsymbol{\theta}_{i+1}$ increases its chance of choosing from $\bar \Theta_H$ --- meaning a successful safety alignment between the robot and human.

For the  \textit{Safe MPC Alignment}, we will later (Section \ref{sec.results}) show  within finite number $K$ of correction, the method will either
\begin{itemize}
    \item[(I)] surely have   $\boldsymbol{\theta}_i\in \bar\Theta_H \subset \Theta_H$, $i < K$, or 
    \item[(II)] declare the failure due to misspecification ($\bar\Theta_H \cap \Theta_0=\emptyset$)
\end{itemize}
Also in Section \ref{sec.results}, we will provide implementation details for the above hypothesis space cutting procedure.

\subsection{Properties and Geometric Interpretations}

We next show properties and geometric interpretations for the proposed safe MPC alignment process. Those properties will be used to establish the convergence of the algorithm.

\subsubsection{Piecewise Linear  Cut}
By recalling the weight-feature parameterization of the safety constraint  (\ref{equ.safey_param}), one can examine that the inequities in (\ref{equ.C_steps}) are both linear in $\theta$. This means that the cutting set $\mathcal{C}_i$  is a piecewise linear cut, as stated below.

% we can examine  $g_{\boldsymbol \theta}(\boldsymbol\xi_{\boldsymbol \theta_i}^\gamma ) < 0$ is a linear inequality in $\boldsymbol \theta$. 

% Based on the definition of the penalty objective function (\ref{equ.barrier}), $\left< \mathbf{\bar{a}}_i,\nabla B(\boldsymbol\xi_{\boldsymbol \theta_i}^\gamma ,\boldsymbol \theta) \right> \leq 0$ is a linear fractional inequality
% \begin{equation}
%      \left< \mathbf{\bar{a}}_i,\nabla J(\boldsymbol\xi_{\boldsymbol \theta_i}^\gamma ) \right> + \gamma \frac{1}{- g_{\boldsymbol{\theta}}(\boldsymbol\xi_{\boldsymbol \theta_i}^\gamma)} \left< \mathbf{\bar{a}}_i,\nabla g_{\boldsymbol \theta}(\boldsymbol\xi_{\boldsymbol \theta_i}^\gamma ) \right> \leq 0,
%      \label{equ.lin_ineq}
% \end{equation}
% with a positive denominator $- g_{\boldsymbol{\theta}}(\boldsymbol\xi_{\boldsymbol \theta_i}^\gamma)>0$ due to (\ref{equ.C_steps}).

% Multiplying the denominator on both side yields another linear inequality. Formally, we give a lemma here about the piecewise linearity of the cut $\mathcal{C}_i$ in (\ref{equ.C_steps}).
\begin{lemma}\label{lemma.linear_cut}
 $\forall i$, $\mathcal{C}_i$ in (\ref{equ.C_steps}) can be expressed as an intersection of two half-spaces in $\mathbb{R}^r$:
\begin{equation}\label{equ.pwl_cut}
    \mathcal{C}_i=\{\boldsymbol \theta\ |\ \boldsymbol{\theta}\tran \boldsymbol h_i \leq b_i\} \cap \{\boldsymbol \theta\ |\ {\boldsymbol{\theta}}\tran\boldsymbol{\phi}(\boldsymbol\xi_{\boldsymbol \theta_i}^\gamma ) < -\phi_0(\boldsymbol\xi_{\boldsymbol \theta_i}^\gamma )\},
\end{equation}
with
\begin{gather}
\label{equ.h_and_b}
\small
     \boldsymbol{h}_i=-\left< \mathbf{\bar{a}}_i,\nabla J(\boldsymbol\xi_{\boldsymbol \theta_i}^\gamma ) \right> \boldsymbol{\phi}(\boldsymbol\xi_{\boldsymbol \theta_i}^\gamma )+ \gamma  \frac{\partial \boldsymbol \phi}{\partial \boldsymbol u}(\boldsymbol\xi_{\boldsymbol \theta_i}^\gamma ) \mathbf{\bar{a}}_i,\\
     b_i=\left< \mathbf{\bar{a}}_i,\nabla J(\boldsymbol\xi_{\boldsymbol \theta_i}^\gamma ) \right> \phi_0(\boldsymbol\xi_{\boldsymbol \theta_i}^\gamma)-\gamma \left< \mathbf{\bar{a}}_i,\nabla \phi_0(\boldsymbol\xi_{\boldsymbol \theta_i}^\gamma) \right>. 
\end{gather}
\end{lemma}
\begin{proof}
See Appendix \ref{proof.lemma_linear_cut}.
\end{proof}

According to the above lemma,  $\Theta_{i-1} \cap \mathcal{C}_i$ can be viewed as cutting $\Theta_{i-1}$ with   two  hyperplanes (Fig. \ref{fig.pipeline_cutting}):
\begin{equation}
\begin{split}
    \mathbf{bd}\ \mathcal{C}_i=\{\boldsymbol \theta\ |\ \boldsymbol{\theta}\tran \boldsymbol h_i = b_i, {\boldsymbol{\theta}}\tran\boldsymbol{\phi}(\boldsymbol\xi_{\boldsymbol \theta_i}^\gamma) \leq -\phi_0(\boldsymbol\xi_{\boldsymbol \theta_i}^\gamma)\}\ \cup \\
    \{\boldsymbol \theta\ |\ \boldsymbol{\theta}\tran \boldsymbol h_i \leq b_i, {\boldsymbol{\theta}}\tran\boldsymbol{\phi}(\boldsymbol\xi_{\boldsymbol \theta_i}^\gamma) = -\phi_0(\boldsymbol\xi_{\boldsymbol \theta_i}^\gamma)\},
\end{split}
\label{equ.C_boundary}
\end{equation}

As  the initial hypothesis space $\Theta_0\subset\mathbb{R}^r$ in (\ref{equ.hypothesis}) is a convex polytope, iteratively cutting with  hyperplanes yields a series of convex polytopes  $\Theta_i$s, $i=0,1,2,...$.

\subsubsection{Relationship between  \texorpdfstring{$\boldsymbol  \theta_{i}$}{TEXT} and \texorpdfstring{$\mathbf{bd}\ \mathcal{C}_i$}{TEXT}}

% According to (\ref{equ.pwl_cut}), the cutting set $\mathcal{C}_i$ depends  on  plan $\boldsymbol\xi_{\boldsymbol \theta_i}^\gamma $ which is  a solution to penalty-based MPC (\ref{equ.robot_mpc_approx}) with   $\boldsymbol  \theta_{i}\in\Theta_{i-1}$. 
The following assertion states a  relationship between $\boldsymbol  \theta_{i}$ and $\mathbf{bd}\ \mathcal{C}_i$.

\begin{lemma}\label{lemma gh}
    $\forall i,\ \boldsymbol{\theta}_i$ is on $\mathbf{bd}\ \mathcal{C}_i$ in (\ref{equ.C_boundary}). Specifically, $\boldsymbol{\theta}_i$ satisfies
\begin{equation}\label{equ.gh}
    \boldsymbol{\theta}_i\tran \boldsymbol h_i =b_i \quad
    \text{and}
    \quad
    \boldsymbol{\theta}_i\tran\boldsymbol{\phi}(\boldsymbol\xi_{\boldsymbol \theta_i}^\gamma) < -\phi_0(\boldsymbol\xi_{\boldsymbol \theta_i}^\gamma)
\end{equation}
for all $i=1,2,...$.
\end{lemma}

\begin{proof}
    See Appendix \ref{proof.lemma gh}.
\end{proof}

Lemma \ref{lemma gh}  states $\boldsymbol  \theta_{i} \in \mathbf{bd}\ \mathcal{C}_i$. Its geometric interpretation  is  in Fig. \ref{fig.theta_choice.subfig_centerpass}: the cut $\mathbf{bd}\ \mathcal{C}_{i}$ always passes $\boldsymbol \theta_{i}$. This suggests that the choice of $\boldsymbol \theta_{i}\in\Theta_{i-1}$ is critical in the learning process: it directly affects the position of the cut  $\mathbf{bd}\ \mathcal{C}_{i}$ and consequently the removal volume. It is important to choose a good $\boldsymbol \theta_{i}\in\Theta_{i-1}$ such that a large hypothesis space is removed after each cut.

We briefly discuss the implication of different choices of $\boldsymbol  \theta_{i} \in \Theta_{i-1}$. When choosing $\boldsymbol \theta_{i}$ near the boundary of $\Theta_{i-1}$, as  in Fig. \ref{fig.theta_choice.subfig_edgepass} and Fig. \ref{fig.theta_choice.subfig_bigcut}, 
the removal of the hypothesis space depends on the specific direction of cut $\mathbf{bd} \,\mathcal{C}_i$. 
Since the direction of $\mathcal{C}_i$ will be eventually determined by the direction of user correction according to (\ref{equ.C_boundary}), the removal of the hypothesis space can be relatively small, as in Fig. \ref{fig.theta_choice.subfig_edgepass}, or large, as in Fig. \ref{fig.theta_choice.subfig_bigcut}, depending on the specific direction of the correction. 

As the specific direction of user correction is hard to assume, a good choice of $\boldsymbol \theta_{i}\in\Theta_{i-1}$ will be near the center of $\Theta_{i-1}$. As such,  it can lead to an averaged large cut of the hypothesis space regardless of the direction of $\mathcal{C}_i$, as in Fig. \ref{fig.theta_choice.subfig_centerpass}. We will rigorously discuss the centering choice of $\boldsymbol{\theta}_i$ and volume reduction rate of the hypothesis space in the next section.

\begin{figure}
    \centering
    \subfigure[]
    {
        \includegraphics[width=0.25\linewidth]{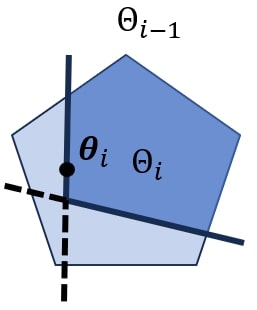}
        \label{fig.theta_choice.subfig_bigcut}
    }
    \subfigure[]
    {
        \includegraphics[width=0.25\linewidth]{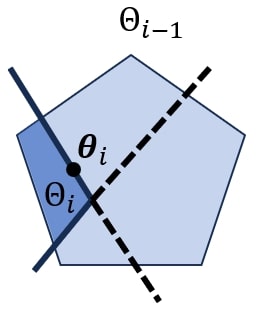}
        \label{fig.theta_choice.subfig_edgepass}
    }
        \subfigure[]
    {
        \includegraphics[width=0.25\linewidth]{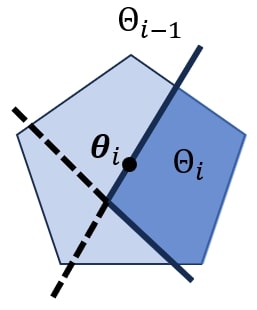}
        \label{fig.theta_choice.subfig_centerpass}
    }
    \caption{Illustration of the different choices  of  $\boldsymbol \theta_{i}$. (a) and (b): When  $\boldsymbol \theta_{i}$ is chosen near the set boundary,  the removed volume  depends on the specific direction of $\mathbf{bd}\ \mathcal{C}_{i}$. (c): When  $\boldsymbol \theta_{i}$ is near the center, the cut volume is large on average regardless of the specific user directional correction. }
    \label{fig.theta_choice}
\end{figure}

\section{Implementation Details and Main Results}\label{sec.results}

\subsection{The Choice of \texorpdfstring{$\boldsymbol \theta_{i}$}{TEXT} }
At each learning iteration $i$, based on the previous geometric intuition, our choice of $\boldsymbol{\theta}_{i}$ is to achieve a large  reduction of the hypothesis space for all possible human directional corrections.  Thus, $\boldsymbol{\theta}_{i}$ needs to be picked from the "center" of $\Theta_{i-1}$. One option is to choose the mass center of $\Theta_{i-1}$. 
 However, computing the centroid of a polytope can be \#P-Hard \cite{rademacher2007approximating}. 
 %While methods like the Hit-And-Run algorithm \cite{laddha2023convergence} estimate the mass center by sampling, they can suffer from high sample complexity. Another possible choice is to use the center of the maximum inscribed ball of a polytope \cite{tong2001support}, but it lacks of the guarantee in volume reduction. 
For computational efficiency, we choose $\boldsymbol{\theta}_{i}$ as the center of the Maximum Volume Ellipsoid (MVE) inscribed in $\Theta_{i-1}$, and will later compare its performance with other choices of $\theta_i$ in our experiments.

% shown in Fig. \ref{fig:MVE}.
% \begin{figure}[h]
% \centering
% \includegraphics[width=0.18\textwidth]{figs/MVE.jpg}
% \caption{ $\boldsymbol{\theta}_{i}$ is chosen  as the center of MVE of $\Theta_{i-1}$ (2D example).}

% \label{fig:MVE}
% \end{figure}
\begin{definition}[Maximum Volume inscribed Ellipsoid (MVE)\cite{boyd2004convex}]
    The MVE of a given convex set $\Theta \subset \mathbb{R}^r$ is defined as the image of the unit ball under an affine transformation
    \begin{equation}
        \mathcal{E}=\{\bar{H}\boldsymbol \theta + \bar{\boldsymbol d}\ |\ \Vert \boldsymbol \theta \Vert_2 \leq 1 \},
    \end{equation}
    where the affine transformation parameter $\bar{H} \in \mathbb{S}^r_{++}$ (a $r\times r$ symmetric positive definite matrix), and $\bar{\boldsymbol d} \in \mathbb{R}^r$ is called the center of $\mathcal{E}$. $\bar{H}, \bar{\boldsymbol d}$ are determined via the following optimization problem to maximize the volume of the ellipsoid $\det H$:
    \begin{equation}
    \begin{split}
    \bar{H}, \bar{\boldsymbol d}=\underset{H \in \mathbb{S}^r_{++},\boldsymbol d}{\arg\max}\quad & \log \det H\\
    \rm{s.t.}\quad 
    & \sup \nolimits_{\Vert \boldsymbol \theta \Vert_2\leq 1} \mathbbm{1}_{\Theta}(H\boldsymbol \theta + \boldsymbol d) \leq 0,
    \end{split}
    \label{equ.MVE_opt}
    \end{equation}
    where the indicator  $\mathbbm{1}_{\Theta}(\boldsymbol \theta)=0$ for $\boldsymbol \theta \in \Theta$ and otherwise $\mathbbm{1}_{\Theta}(\boldsymbol \theta)=+\infty$. Specifically, when $\Theta$ is a polytope $\{ \boldsymbol{\theta}\tran \boldsymbol n_i \leq c_i,\ i=1,\dots,m \}$, the problem can be converted to a convex optimization with $m$ constraints, see chapter 8.4.2 in \cite{boyd2004convex}:
    \begin{equation}
    \begin{split}
    \bar{H}, \bar{\boldsymbol d}=\underset{H \in \mathbb{S}^r_{++},\boldsymbol d}{\arg\min}\quad & \log \det H^{-1}\\
    \rm{s.t. }\quad 
    & \norm{H \boldsymbol n_i}_2 + \boldsymbol d\tran \boldsymbol n_i \leq c_i,\ i=1, ...,m.
    \end{split}
    \label{equ.MVE_opt_cvx}
    \end{equation}    
\end{definition}

Recall that in  Safe MPC Alignment the learning update  \eqref{equ.intersection} is to iteratively cut the hypothesis space by hyperplanes \eqref{equ.C_boundary}, and that the initial hypothesis space (\ref{equ.initial_hypothesis}) is a box. By induction, it leads to that any updated hypothesis space $\Theta_i$ is a convex polytope, $i = 0, 1, 2, ...$. Thus, for any $\Theta_{i-1}$, by the above definition, let $\mathcal{E}_{i-1}$ denote the MVE of $\Theta_{i-1}$ and $\Bar{\boldsymbol d}_{i-1}$ denote the center of $\mathcal{E}_{i-1}$. Then, the choice of $\boldsymbol \theta_{i}$ is
\begin{equation}
    \boldsymbol \theta_{i}=\bar{\boldsymbol d}_{i-1}.
    \label{equ.theta_choice}
\end{equation}
The convex optimization (\ref{equ.MVE_opt}) can be solved efficiently \cite{cvx}. 
\subsection{Result 1: Finite Iteration Convergence}
 When choosing $\boldsymbol \theta_{i}$ as the center of MVE of $\Theta_{i-1}$, the following lemma states that the volume ratio ${\mathbf{Vol} (\Theta_{i})}/{\mathbf{Vol} (\Theta_{i-1})}$ after each hypothesis space cut can be upper bounded.

\begin{lemma}
    In the \textit{Safe MPC Alignment}, when choosing $\boldsymbol \theta_{i} \in \mathbb{R}^r$  as the center  of MVE of $\Theta_{i-1}$, one has
    \begin{equation}
       \frac{\mathbf{Vol} (\Theta_{i})}{\mathbf{Vol} (\Theta_{i-1})} \leq 1-\frac{1}{r},
    \end{equation}
    for all $i=1,2,3,...$ and any enclosure volume measure $\bf{Vol}$.
    \label{lemma v}
\end{lemma}
\begin{proof}
    See Appendix \ref{proof.lemma v}.
\end{proof}
The lemma above asserts an upper bound of the volume of the hypothesis space $\mathbf{Vol} (\Theta_{i})$ relative to the previous volume. 
With this, we can now establish the finite iteration convergence of the \textit{Safe MPC Alignment}, stated below.
\begin{theorem}[Successful learning in finite iterations]
    In the \textit{Safe MPC Alignment} algorithm,  $\boldsymbol \theta_{i}$ is selected as the center  of MVE of $\Theta_{i-1}$, $i=1,2,3,...$. If $\Bar\Theta_H \subset \Theta_0$,  the algorithm will terminate successfully within 
    \begin{equation} \label{equ.K}
        K=\left\lceil \frac{\ln \frac{\tau_r \rho_H
^r}{\mathbf{Vol} (\Theta_0)}}{\ln (1-\frac{1}{r})}\right\rceil.
    \end{equation}
    iterations, where $\rho_H$ is the radius of a ball contained within $\Bar\Theta_H$ and $\tau_r$ is the volume of a unit ball in $\mathbb{R}^r$ space. That is, there exists $i < K$, such that $\boldsymbol{\theta}_i\in\Bar\Theta_H \subset \Theta_H$.

    \label{theorem_convergence}
\end{theorem}
\begin{proof}
    See Appendix \ref{proof.theorem_convergence}.
\end{proof}
The above result asserts the maximum learning iteration to find  $\boldsymbol{\theta}_i\in\Theta_H$. In other words, within $K$ rounds of human correction, the robot can successfully learn a safety constraint. The above result requires a pre-set of $\rho_H
^r$, which is the radius of a ball contained in $\Bar\Theta_H$. Since the size of $\Bar\Theta_H$ is unknown in reality, $\rho_H>0$  can be interpreted as an estimated size of the unknown user intent set $\bar\Theta_H$. It also serves as a termination threshold for learning convergence: a smaller $\rho_H$ can lead to a more concentrated $\boldsymbol{\theta}_H$ at convergence.

The finite iteration convergence in  Theorem \ref{theorem_convergence} comes with a prerequisite: $\bar\Theta_H\in\Theta_0$ has to be satisfied in the first place for the {Safe MPC Alignment}. Although such a condition can be achieved by empirically choosing $\Theta_0$ large enough to contain all possible $\bar\Theta_H$ (\ref{equ.hypothesis}), it is also desired to have a verification for the hypothesis misspecification,  where such condition is not met. We will provide such result in the next subsection.

\subsection{Result 2: Hypothesis Misspecification Certification}
Consider hypothesis space misspecification, where $\Bar \Theta_H\cap\Theta_0=\emptyset$. In this case, no $\boldsymbol{\theta} \in \Theta_0$ can account for the user intended safety   (\ref{equ.user_safety}). The hypothesis misspecification can lead to the undesired behavior of the robot \cite{bobu2020quantifying,ghosal2023effect}. The following result states the  {Safe MPC Alignment}   can automatically certify the misspecification of $\Theta_0$.
\begin{theorem}[Certifying misspecification]
\label{theorem_misspec}
In the \textit{Safe MPC Alignment}, $\boldsymbol \theta_{i}$ is selected as the center of MVE of $\Theta_{i-1}$, $i=1,2,3,...$. 
     If $\bar\Theta_H \cap \Theta_0 = \emptyset$, the algorithm  converges to a set $\Theta^*$ as $i\rightarrow \infty$, such that $\mathbf{Vol} (\Theta^*)=0$ and $\Theta^* \cap \mathbf{int}\ \Theta_0= \emptyset$.
\end{theorem}
\begin{proof}
    See Appendix \ref{proof.theorem_misspec}.
\end{proof}
The above result states  when the initial hypothesis space is misspecified, $\bar\Theta_H \cap \Theta_0 = \emptyset$, the {Safe MPC Alignment} algorithm will finally converge to a set of zero volume on the boundary of the initial hypothesis space $\Theta_0$.
 The above theorem inspires us to empirically identify the misspecification by monitoring the position of the MVE center of $\Theta_i$ during the learning update. As shown in Fig. \ref{fig:misspec},
if after the maximum iteration $K-1$ in (\ref{equ.K}), the MVE center $\boldsymbol{\theta}_{K-1}$ is very close to the boundary of the initial hypothesis set $\Theta_0$, there is likely hypothesis space misspecification, and the algorithm terminates with misspecification failure.

% The convexity of the cutting set $\mathcal{C}_{K-1}$,  including $\boldsymbol{\theta}_{K-1}$ and  $\Bar \Theta_H$, suggests that, by expanding the boundary of $\Theta_0$ near $\boldsymbol{\theta}_{K-1}$, we can possibly include $\bar\Theta_H$ in $\Theta_0$. 

\begin{figure}[h]
\centering
\vspace{-10pt}
\includegraphics[width=0.18\textwidth]{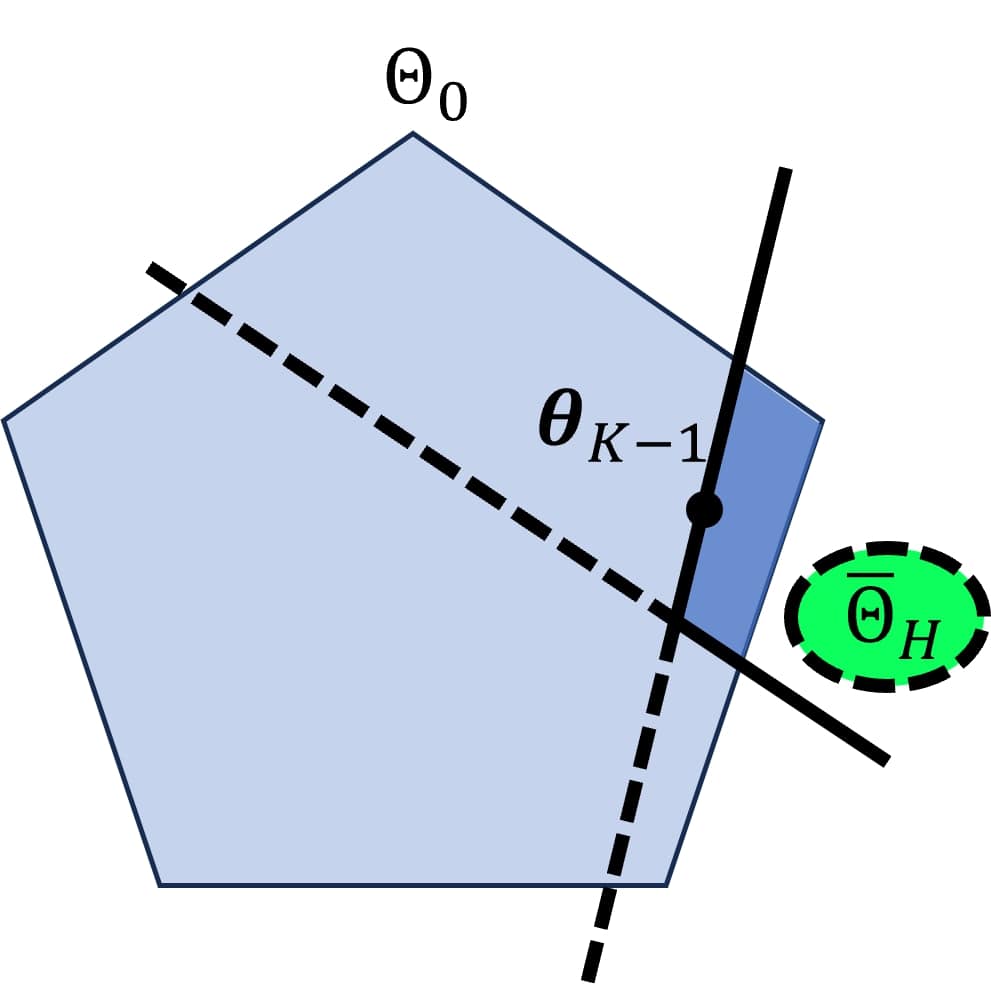}
\caption{ 
When there is misspecification, the {Safe MPC Alignment}  will converge to some point on the boundary of $\Theta_0$. So we can use the distance of $\boldsymbol{\theta}_{K-1}$ to the boundary as an indicator for misspecification.
}
\label{fig:misspec}
\end{figure}

Specifically, if $\Theta_0$ is a box as defined in (\ref{equ.hypothesis}) with bounds $\boldsymbol{c}_l, \boldsymbol{c}_h$, we can pre-define a small threshold $\epsilon$ to examine if  $\boldsymbol{\theta}_{K-1}$ is on the boundary of $\Theta_0$. If $\boldsymbol{\theta}_{K-1}$ satisfies
\begin{equation}
        \label{equ.expand_cond}
\min_{j{=}1,...,r}\ \min\big(\lvert  \boldsymbol{c}_l[j]- \boldsymbol{\theta}_{K-1} [j] \rvert, \lvert \boldsymbol{c}_u[j]- \boldsymbol{\theta}_{K-1} [j] \rvert \big) \leq \epsilon
\end{equation}
it means $\boldsymbol{\theta}_{K-1}$ is very close to the boundary of $\Theta_0$.  We consider there is likely misspecification in this case. 
%We can then enlarge the initial hypothesis space $\Theta_0$ in the identified dimension of $j$ and rerun the algorithm.

\subsection{Detailed Algorithm of Safe MPC Alignment}
With the main results on convergence and misspecification, we present the detailed  algorithm in Algorithm \ref{algo.cutting}.  With  $\Theta_0$ and  a termination threshold $\rho_H$ pre-set, three steps in the hypothesis space cutting are executed until the user satisfies robot motion or the iteration count reaches the maximum count $K$ in (\ref{equ.K}). Note that during the execution of the MPC policy $\boldsymbol{\pi}_{\boldsymbol{\theta}_{i}}^\gamma$, the human can perform emergency stopping and reset the robot's state, which corresponds to Section \ref{sec.human_supervise}. 
After $K-1$ iterations, the algorithm checks the misspecification by (\ref{equ.expand_cond}).   

\begin{algorithm2e}[htbp]
	\small 
	\SetKwInput{Parameter}{Hyperparameter}	
	\SetKwInput{Initialization}{Initialization}
	\Initialization{Initial hypothesis space $\Theta_0$,  termination threshold $\rho_H$, $i=1$} 
 \smallskip
 Compute  maximum iteration count $K$ from  $\rho_H$ in \eqref{equ.K}\;
	\smallskip
\While{$i<K$}{
        Choose $\boldsymbol \theta_{i} \in \Theta_{i-1}$ using (\ref{equ.theta_choice})\;

Robot executes $\boldsymbol{\pi}_{\boldsymbol{\theta}_{i}}^\gamma$\;

\uIf{Human provides directional correction $\boldsymbol{a}_{k_i}$}{
 Generate $\bar{\boldsymbol{\mathbf{a}}}_i$ using (\ref{equ.a_extension})\;
Calculate the cutting set $\mathcal{C}_i$ using (\ref{equ.pwl_cut})\;
Update hypothesis space $\Theta_{i}=\Theta_{i-1} \cap \mathcal{C}_i$\;
 $i \gets i+1$\;
    }

        \smallskip
        
    \uElseIf{
     Human triggers emergency stop
    }
    {
    Human restarts robot from safe region\;
    }
    
    \smallskip
    
    \uElseIf{
     Human user is satisfied
    }
    {
    \Return $\boldsymbol{\theta}_i$\;
    }
  }
 \eIf{$\boldsymbol{\theta}_{K-1}$ satisfies \eqref{equ.expand_cond}}{
      \Return misspecification failure\;
    }{
      \Return $\boldsymbol{\theta}_{K-1}$
    }
\caption{Safe MPC Alignment Algorithm} \label{algo.cutting}
\end{algorithm2e}

\section{Numerical Examples}

In this section, we will evaluate the Safe MPC Alignment using two numerical examples: A pendulum task, where we will illustrate all theoretical aspects of the proposed  method, and  a thrust-controlled quadrotor navigation task, where we will demonstrate the capability of the method to learn safety constraints in complex 3D environments. Simulated human correction knowing ground-truth safety constraints will be used. We will compare the proposed method with baselines, and ablate the hyperparameters of the algorithm. All numerical experiments are implemented in Python, running on a desktop 
with Intel i9-13900K CPU. All our experiment codes are open-sourced at \url{https://github.com/asu-iris/Safe-MPC-Alignment}.

% The organization in this section is as follows: Subsection \ref{sec.pendulum} presents the numerical evaluation on inverted pendulum and its analysis. Subsection \ref{sec.reacher} presents another numerical experiment on 2D planer robot arm. Subsection \ref{sec.comp_ablat} shows the procedure and results for comparisons with baselines and ablations.

\begin{figure*}[!htpb]
    \centering
    \subfigure[Initial trajectory (with $\boldsymbol{\theta}_1$)]
    {
        \includegraphics[width=1.7in]{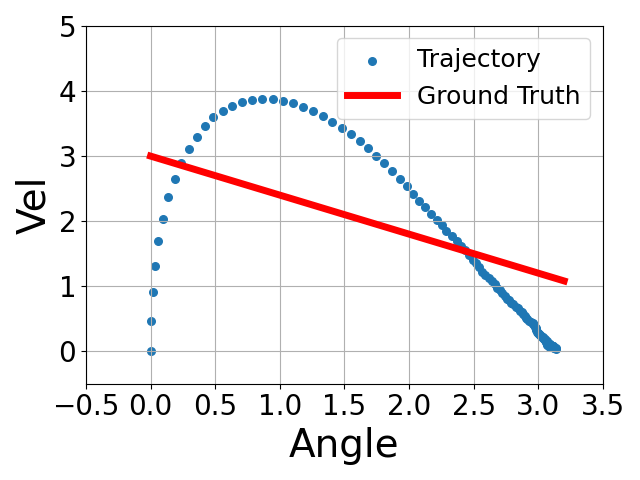}
    }
    \subfigure[After 3 corrections (with $\boldsymbol{\theta}_4$)]
    {
        \includegraphics[width=1.7in]{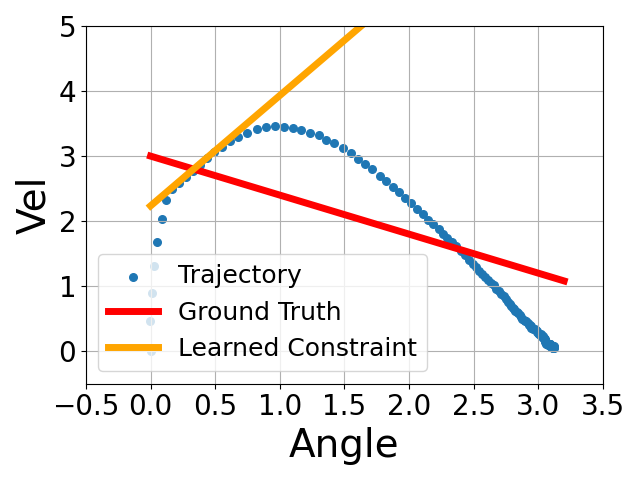}
    }
    \subfigure[After 6 corrections (with $\boldsymbol{\theta}_7$)]
    {
        \includegraphics[width=1.6in]{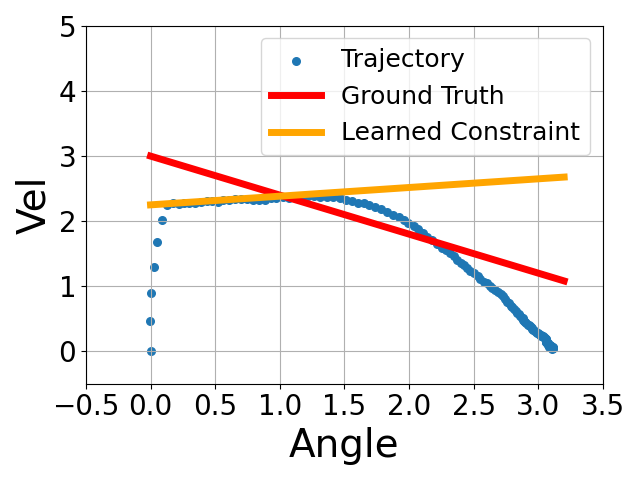}
    }
    \subfigure[After 14 corrections (with $\boldsymbol{\theta}_{15}$)]
    {
        \includegraphics[width=1.6in]{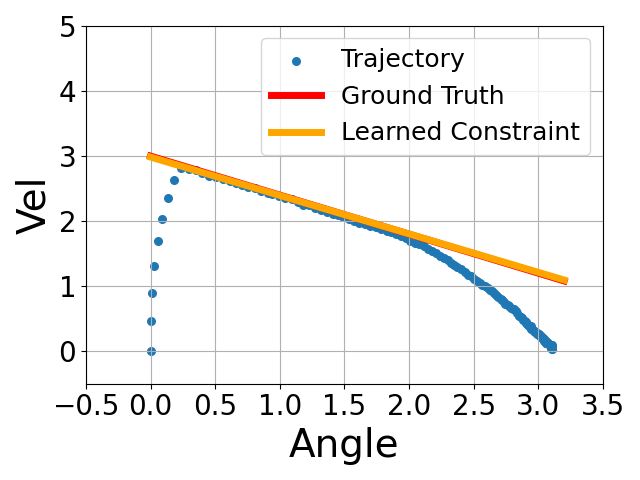}
    }
    \caption{The Safe MPC rollout trajectory after different learning interaction. The learned constraint with $\boldsymbol{\theta}_1$ is omitted since it covers the whole first quadrant. As learning proceeds, the constraint boundary (orange line) converges to the ground-truth  (red line). When the learning process terminates, the learned and ground-truth constraint boundaries almost overlap, suggesting the success of constraint learning from correction.}
    \label{fig.pen_trajs}
\end{figure*}

\subsection{Inverted Pendulum}
\label{sec.pendulum}
\subsubsection{Setup}
The dynamics of an inverted pendulum is
\begin{equation}\label{equ.pendulum_dyn}
    \ddot \alpha = \frac{3}{ml^2}(-\frac{1}{2}m\mathrm{g}l\sin{\alpha} + u - d \dot \alpha),
\end{equation}
with $\alpha$ being the angle between the pendulum and gravity,
 $\mathrm{g}$ the gravity constant, $u$  input torque, $m$, $l$ and $d$ the length,   inertia, and damping ratio. 
The system state is $\boldsymbol{x}=[\alpha,\dot \alpha]\tran$,  (\ref{equ.pendulum_dyn}) discretized by the Euler method with $\triangle t = 0.02 \mathrm{s}$. A  zero-mean Gaussian noise with covariance $\Sigma=\mathbf{diag}(1,4)\times 10^{-5}$ is added to $\boldsymbol{x}$ at each MPC rollout to simulate the state noise. The task is  to control the pendulum to a target $\boldsymbol{x}^*=[\pi,0]\tran$. 
 We  pose an unknown global  
state constraint on pendulum motion:
\begin{equation}
    0.6 \alpha+ \dot{\alpha}=\boldsymbol{\theta}^H\boldsymbol{x} \leq 3,
    \label{equ.pen_gt_constraint_rollout}
\end{equation}
where $\boldsymbol{\theta}^H=[0.6,1]\tran$ is  to be learned. All other parameter values are given in Appendix.

In Safe MPC (\ref{equ.robot_mpc}), we set $T=40$ and the cost functions as
\begin{equation}
\begin{gathered} 
c(\boldsymbol{x}_t,u_t) =  R u_t^2\\ 
h(\boldsymbol{x}_T) =  (\boldsymbol{x}_T - \boldsymbol{x}^*)^T Q (\boldsymbol{x}_T - \boldsymbol{x}^*),
\end{gathered}
\end{equation}
respectively, with $R=0.1$ and $Q = \mathrm{diag}(25,10)$.
To ensure the safe MPC to satisfy a global constraint, it is sufficient to learn a safety constraint defined on the closest controllable state. Thus,  we define
\begin{equation}
    {g}_{\boldsymbol{\theta}}(\boldsymbol{\xi}) = -3 + \boldsymbol{\theta}\tran\boldsymbol{x}_1.
    \label{equ.pen_param_constraint_rollout}
\end{equation}
In the penalty-based  MPC (\ref{equ.robot_mpc_approx}), we set $\gamma=0.1$.

Suppose  the user intent set $\Bar\Theta_H=\{\boldsymbol{\theta}|\norm{\boldsymbol{\theta}-\boldsymbol{\theta}_H}\leq 0.02\}$. The simulated (synthetic)  correction $\mathbf{a}$ are generated as follows: when the system state is near the safety boundary, i.e.,  ${g}_{\boldsymbol{\theta}_H}(\boldsymbol{\xi}_{\boldsymbol{\theta}}^\gamma)>-\epsilon_g$ ($\epsilon_g=0.25$ here), a directional correction $\mathbf{a}$ is be generated in probability $p=0.3$ using $\mathrm{sign}(-\nabla B(\boldsymbol\xi_{\boldsymbol \theta}^\gamma, \boldsymbol \theta_H))$ (from which only the first step is taken as in  {\eqref{equ.a_extension}}). The correction is probabilistic to simulate the randomness of human providing corrections.  It is easy to verify that the simulated $\mathbf{a}$ is valid in term of satisfying the hypotheses in section \ref{sec.idea_hypothesis}. 

% When the system state violates ${g}_{\boldsymbol{\theta}_H}(\boldsymbol{\xi}_{\boldsymbol{\theta}}^\gamma) \leq 0$ or it is near the target, the pendulum state is uniformly randomly reset between $[0,0]\tran$ and $[\frac{2\pi}{3},3]\tran$. Simulated human correction will continue to be applied until the learned parameters $\boldsymbol{\theta}_i$ fall in $\Bar\Theta_H$, which terminates one trial. 

\subsubsection{Convergence Analysis}

\begin{figure}[h]
\centering

\subfigure[Logarithm MVE Volume]
    {
        \includegraphics[width=0.22\textwidth]{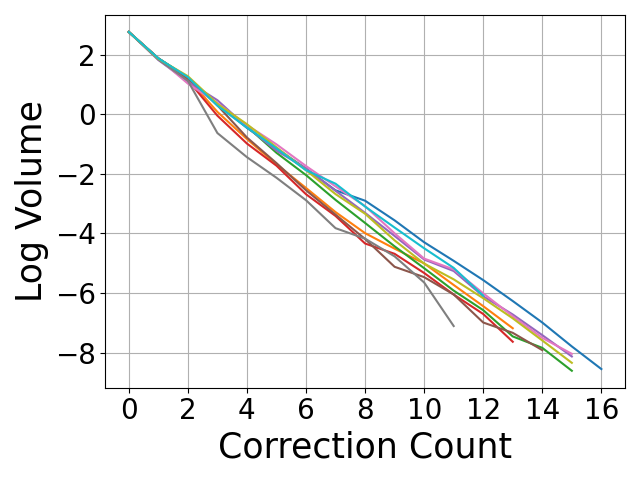}
    } 
\subfigure[Parameter Estimation Error]
    {
        \includegraphics[width=0.22\textwidth]{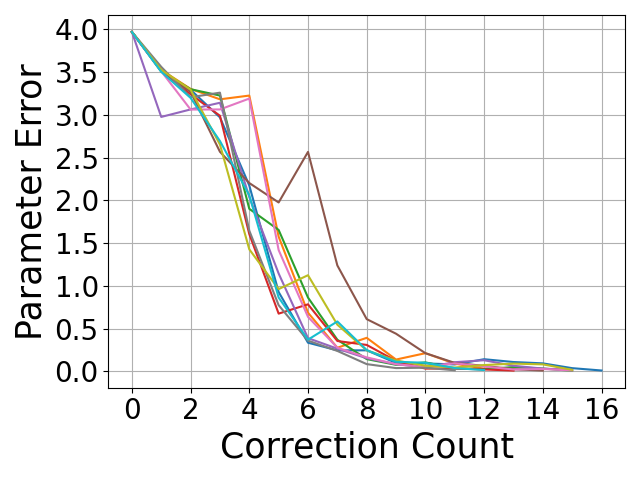}
    }

\caption{(a) The plot of $\log \det \bar H$ with respect to correction numbers. In all algorithm runs (different colors), the plot shows the near-linear decrease of the logarithm volume, suggesting the exponential decrease of the volume of MVE of the hypothesis spaces, as asserted in Theorem \ref{theorem_convergence}. (b) The plot of the parameter estimation error $\norm{\boldsymbol{\theta}_i-\boldsymbol{\theta}_H}_2$ versus correction numbers. In all runs, the parameter estimation error can converge to a value lower than 0.02 within 17 corrections, consistent with the theoretical bound in Theorem \ref{theorem_convergence}.}

\label{fig.pen_aspects}
\end{figure}

We first evaluate the convergence  of our algorithm without hypothesis misspecification. In Algorithm \ref{algo.cutting}, we set the initial hypothesis $\Theta_0$ with box boundaries $\boldsymbol{c}_l = [-6,-6]\tran$ and $\boldsymbol{c}_u = [2,2]\tran$; thus $\Theta_0$  includes  $\Bar\Theta_H$. The initial  $\boldsymbol{\theta}_1=[-2,-2]\tran$ (i.e., the center of  $\Theta_0$). 

With each simulated  correction, we track  $\boldsymbol{\theta}_i$ in each learning update. Each update of ${\Theta}_i$ takes less than 30ms. $\boldsymbol{\theta}_i$ falls in the ground-truth set $\Bar\Theta_H$ after 14 corrections. The MPC rollout trajectory (phase plot) with each $\boldsymbol{\theta}_i$ is shown in Fig. \ref{fig.pen_trajs}. The ground-truth constraint (\ref{equ.pen_gt_constraint_rollout}) and the learned constraints $\boldsymbol{g}_{\boldsymbol{\theta}_i}$ are also plotted.  Fig. \ref{fig.pen_trajs} shows that the true safety constraint can be successfully learned from simulated correction.

To  analyze the convergence, we run the algorithm for 10 times. In each run, we keep track of two metrics over corrections (iterations) and show them in  Fig. \ref{fig.pen_aspects}. The first metric (left panel) is the logarithm volume measure of the solved MVE $\log \det \bar H$ in (\ref{equ.MVE_opt}); this provides a lower bound of the volume of  $\Theta_i$ to show the exponential convergence. The second (right panel) is the parameter error $\norm{\boldsymbol{\theta}_i-\boldsymbol{\theta}_H}_2$.  From Fig. \ref{fig.pen_aspects}, we can conclude the proposed method is

% M1 (left panel of Fig. \ref{fig.pen_aspects}): The logarithm volume measure of the solved MVE $\log \det \bar H$ in (\ref{equ.MVE_opt}). This provides a lower bound of $\Theta_i$ to show the exponential convergence property.

% M2  (right panel of Fig. \ref{fig.pen_aspects}): The euclidean distance between learned parameter $\boldsymbol{\theta}_i$ and $\boldsymbol{\theta}_H$, i.e $\norm{\boldsymbol{\theta}_i-\boldsymbol{\theta}_H}_2$. This suggests the similarity between the learned constraint and the ground-truth constraint, see Fig. \ref{fig.pen_aspects}.

\begin{itemize}
    \item[(I)] Fast-converging. The logarithm volume $\log \det \bar H$ has a near-linear decreasing speed. This is consistent with the theory in Lemma \ref{lemma v}.
    \item[(II)] Data-efficient.  $\norm{\boldsymbol{\theta}_i-\boldsymbol{\theta}_H}_2$ converges to the value lower than 0.02 within 17 corrections. This is consistent with the theoretical upper bound of  correction count in Theorem~\ref{theorem_convergence}:
    % for the number of corrections by replacing $r=2$, $\tau_r=\pi$, $\rho_H=0.02$, and $\mathbf{Vol} (\Theta_0)=8^2=64$:
\begin{equation}
    {\ln \frac{\tau_r \rho_H^r}{\mathbf{Vol} (\Theta_0)}}/{\ln (1-\frac{1}{r})}=17.28.
    \label{equ.pendulum_bound}
\end{equation}
\end{itemize}

\subsubsection{Misspecification Analysis}

\begin{figure}[h]
\centering
\vspace{-10pt}
\includegraphics[width=0.30\textwidth]{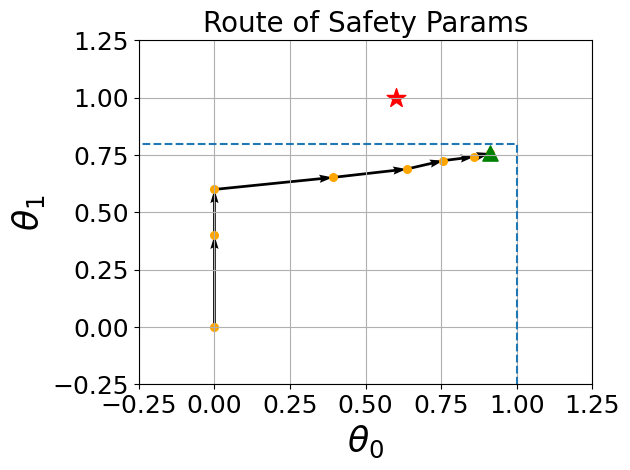}
\vspace{-10pt}
\caption{Experiment of misspecification case. The blue dashed lines are the boundary of $\Theta_0$.  Path of $\theta$ convergence, $\boldsymbol{\theta}_i,\ i=1,\dots,7$,  is shown in  orange dots with arrows showing the update direction. The ground-truth intent set is marked as the red star. Since misspecification happens for this case,  $\boldsymbol{\theta}$ will converge toward the boundary of $\Theta_0$ (the green triangle). This is consistent with the statement in Theorem \ref{theorem_misspec}. }
\label{fig.pen_misspec}
\end{figure}
To analyze the performance of our algorithm to certify misspecification cases, we set the initial hypothesis space $\Theta_0$  as ${\boldsymbol{c}_l} = [-1,-1]\tran$ and ${\boldsymbol{c}_u} = [0.8,0.8]\tran$ which does not contain $\Bar \Theta_H$. In this setting, the algorithm automatically stops after 7th correction. The path of $\boldsymbol{\theta}_i,\ i=1,\dots,8$ is shown in Fig. \ref{fig.pen_misspec}. $\boldsymbol{\theta}_i$ is continuously moving towards the boundary of $\Theta_0$. This fact verifies Theorem \ref{theorem_misspec} and justifies the proposed method is capable of misspecification detection.

\subsection{Quadrotor Navigation in Complex Environment}
\begin{figure}[h]
\centering
\includegraphics[width=0.22\textwidth]{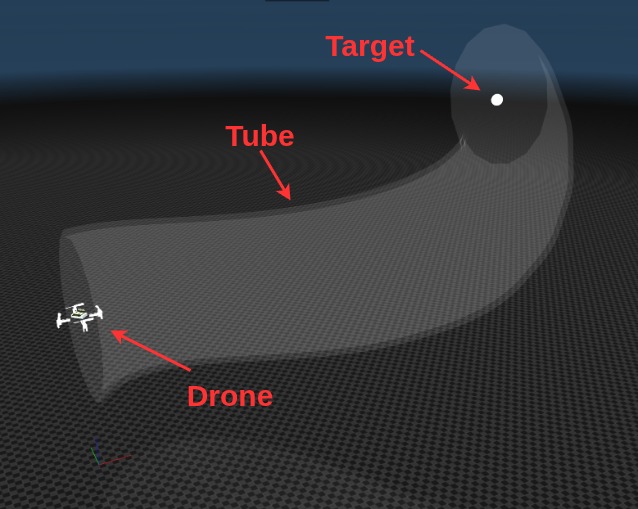}
\caption{Quadrotor navigation in an 3D  zigzag tube.}
% \vspace{pt}
\label{fig.drone_sim}
\end{figure}

\subsubsection{Setup}
We consider a thrust-controlled quadrotor navigating in an 3D  zigzag tube,  shown in Fig. \ref{fig.drone_sim}. The simulation environment is based on MuJoCo \cite{todorov2012mujoco}, and the quadrotor  dynamics follows \cite{jin2022learning}. Initially, the quadrotor has no knowledge of the safety constraint. When it nears the tube wall, a simulated (synthetic) correction is generated to push it away from the tube wall. Details of the setting are provided in Appendix~\ref{appendix.uav_sim}. 

To learn  positional safety constraints, we experiment with three  feature vectors $\boldsymbol{\phi}(\boldsymbol{\xi})$ in (\ref{equ.safey_param}):  
% To demonstrate that our method works with generic and neural features for complex safety constraints,  we use two types of feature vectors  $\boldsymbol{\phi}(\boldsymbol{\xi})$ in (\ref{equ.safey_param}):
\paragraph{Polynomial feature vector} A  vector of polynomial basis functions up to 3rd order to properly capture the zig-zag pattern of the tube)  is defined in the quadrotors' 3D position $[x,y,z]$ at the 5th timestep in  motion plan $\boldsymbol{\xi}$:
\begin{multline}
        \boldsymbol{\phi}(\boldsymbol{\xi}) {=} [x_5^3, y_5^3, z_5^3, x_5^2, y_5^2, \\z_5^2, x_5, y_5, z_5, x_5y_5, y_5z_5, x_5z_5]^\top.
\end{multline}
The features are defined using the quadrotor’s 3D position at the 5th prediction step.
%The polynomial features are generic and require no task-specific design.  
The hypothesis space is initialized as a box space with $\boldsymbol{c}_l =[-80,\dots,-80]\tran$ and $\boldsymbol{c}_u = [200,\dots,200]\tran$. We select $\phi_0 = 1$ and $\gamma=60$.

\paragraph{Pre-trained neural feature vector} We obtain the neural spacial features from a pre-trained Signed Distance Function (SDF) \cite{park2019deepsdf}. Specifically, we pre-train an SDF neural network using the point cloud data samples near the tube surface. This simulates the application where some geometric information of the environment is available. The neural SDF is pre-trained in a regression-like
manner, with distance, normal, and Eikonal losses (see details in  Appendix~\ref{appendix.uav_sim}). With the pretrained SDF function, we take the 16-dimensional output of the second-to-last layer of the SDF network as our 
 feature vector $\boldsymbol \phi(\boldsymbol{\xi})$. The input to $\boldsymbol \phi(\boldsymbol{\xi})$ is the quadrotor's 3D position $[x,y,z]$ at the 5th perdition step. Such a choice of feature vector $\boldsymbol \phi(\boldsymbol{\xi})$ means that we will use the proposed Safe MPC Alignment to fine tune the last layer of the learned SDF, $g(\boldsymbol{\xi})=\boldsymbol{\theta}\tran\boldsymbol{\phi}(\boldsymbol{\xi})$, for its safe navigation control.

To facilitate the real-time Safe MPC with the neural safety constraint $g(\boldsymbol{\xi})=\boldsymbol{\theta}\tran\boldsymbol{\phi}(\boldsymbol{\xi})$, the optimization (\ref{equ.robot_mpc}) uses the linearized $g(\boldsymbol{\xi})$ for calculation, which is linearized at the real robot state $\boldsymbol{x}^{\text{real}}_k$. The initial bounds of the hypothesis space are set to $\boldsymbol{c}_l =[-0.2,\dots,-0.2]\tran$, $\boldsymbol{c}_u = [0.1,\dots,0.1]\tran$. We select $\phi_0 = 0.1$ and $\gamma = 50$.

To ensure the feasibility of the learned constraint, in the above feature vectors, two additional cuts are added in each MVE update to enforce that the start and goal remain within the intended safe region. The task is successful if the drone reaches the goal without triggering this stop.

\paragraph{Polynomial features with velocity features} We consider the safety constraints are additionally defined in the velocity space. After learning the positional constraint with the polynomial feature vector, a new set of velocity-based features over all prediction steps, $\boldsymbol{\phi}(\boldsymbol{\xi}) = [\|\boldsymbol{v}_1\|^2, \dots, \|\boldsymbol{v}_{10}\|^2]^\top$, is added. In this case, we will freeze the weights for polynomial features, and only learn the weights for velocity features. Our objective is to impose an additional constraint that caps the quadrotor’s velocity magnitude below 0.45 [m/s], while ensuring that the previously learned positional constraint remains satisfied for collision avoidance. When the quadrotor's velocity exceeds a correction triggering threshold (0.4), a simulated correction is applied to push the quadrotor in the direction of negative velocity. The hypothesis space for the velocity features is initialized with bounds $\boldsymbol{c}_l = [0, \dots, 0]^\top$ and $\boldsymbol{c}_u = [0.16, \dots, 0.16]^\top$. We set the bias feature $\phi_0 = 1$ and select $\gamma = 100$.

\begin{figure}[!htpb]
\centering

\subfigure
    {
        \includegraphics[width=0.13\textwidth]{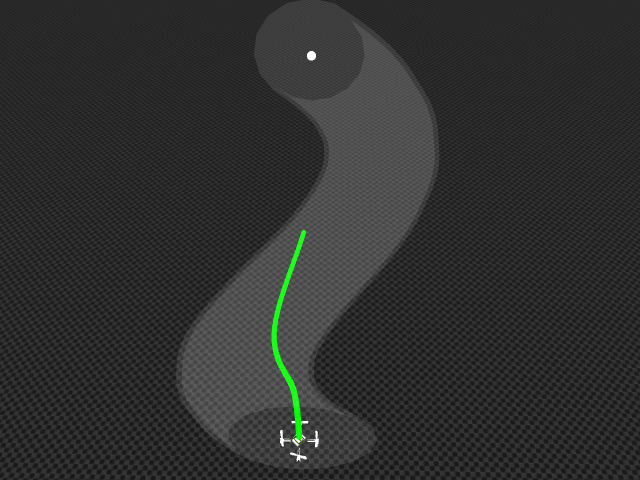}
    } 
\subfigure
    {
        \includegraphics[width=0.13\textwidth]{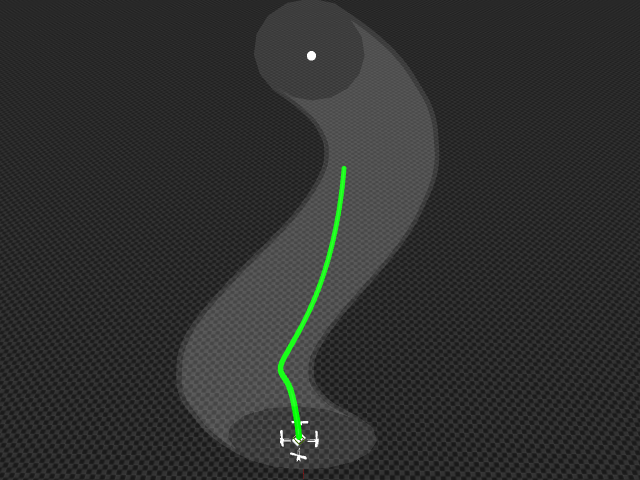}
    }
\subfigure
    {
        \includegraphics[width=0.13\textwidth]{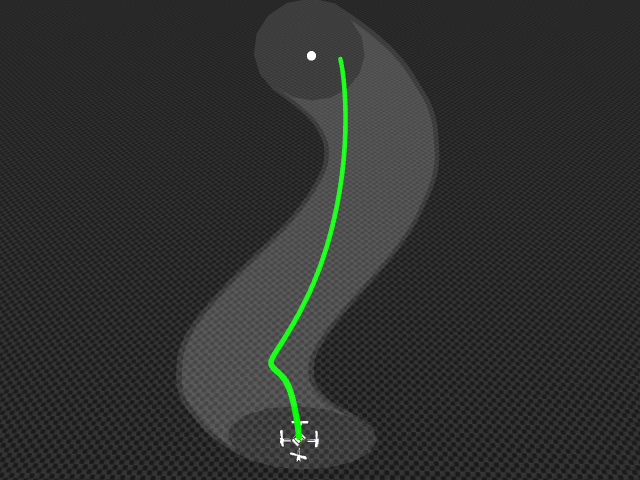}
    }

\subfigure
    {
        \includegraphics[width=0.13\textwidth]{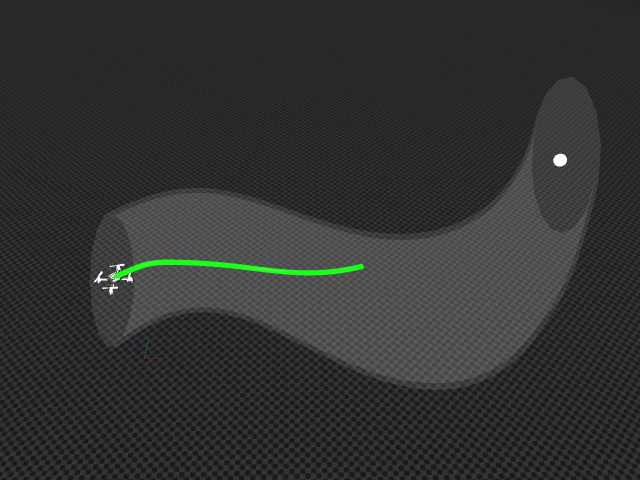}
    } 
\subfigure
    {
        \includegraphics[width=0.13\textwidth]{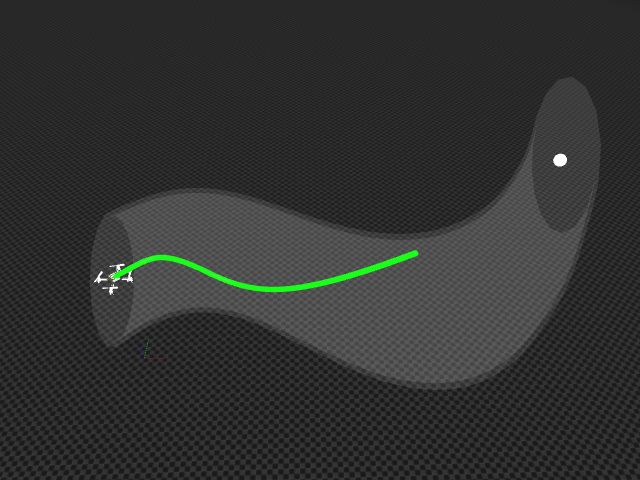}
    }
\subfigure
    {
        \includegraphics[width=0.13\textwidth]{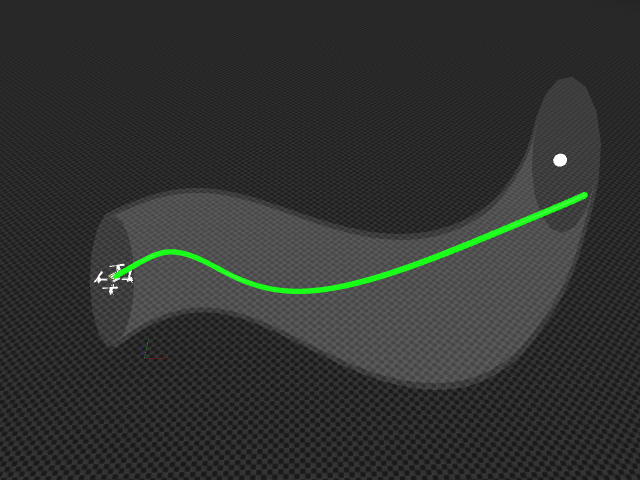}
    }
\caption{The Safe MPC  rollout trajectories (in green) for learning safety constraints  with polynomial features after 26, 38, 45 corrections (different columns). The first row and second row show the top and  the side views of the navigation environment.}
\vspace{0pt}
\label{fig.res_poly}
\end{figure}

\begin{figure}[!htpb]
\centering
\subfigure
    {
        \includegraphics[width=0.13\textwidth]{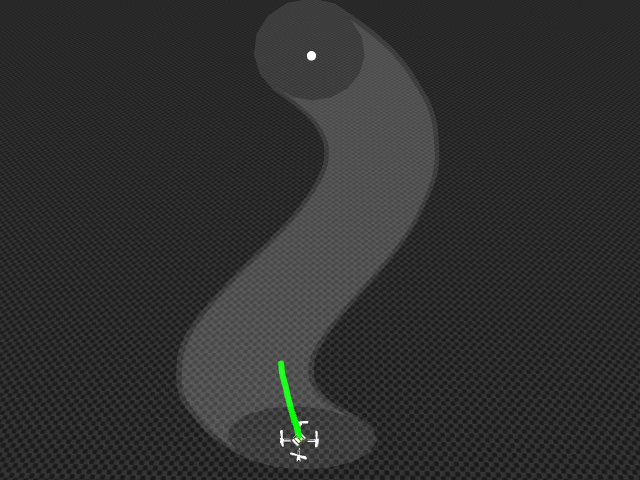}
    } 
\subfigure
    {
        \includegraphics[width=0.13\textwidth]{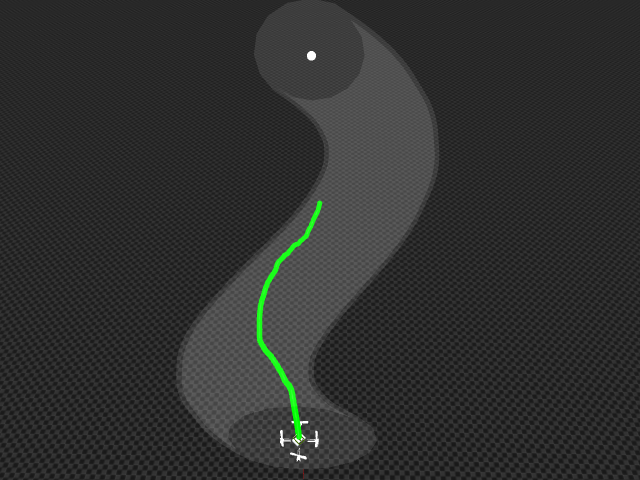}
    }
\subfigure
    {
        \includegraphics[width=0.13\textwidth]{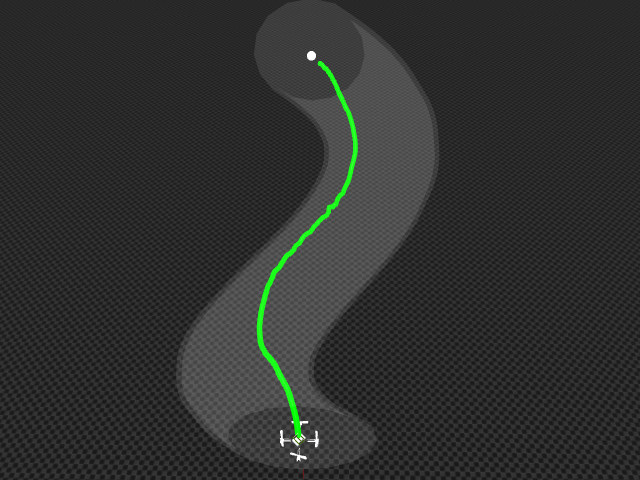}
    }

\subfigure
    {
        \includegraphics[width=0.13\textwidth]{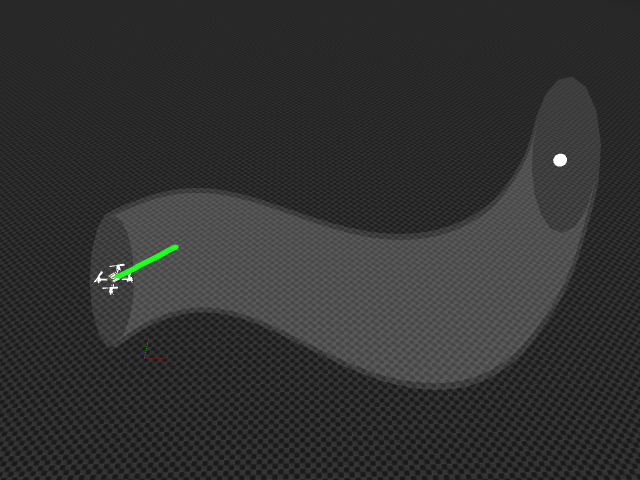}
    } 
\subfigure
    {
        \includegraphics[width=0.13\textwidth]{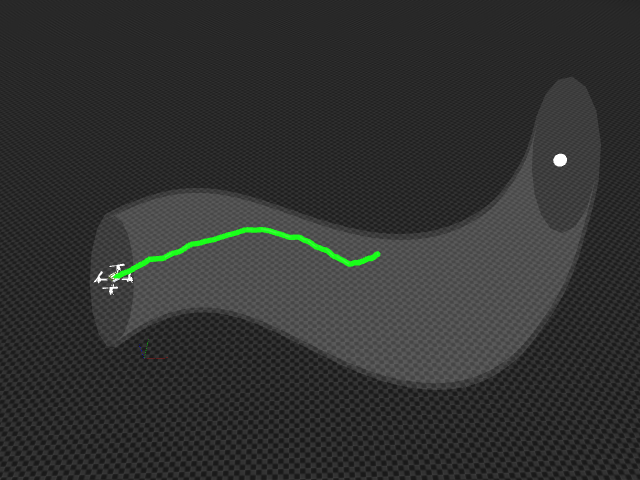}
    }
\subfigure
    {
        \includegraphics[width=0.13\textwidth]{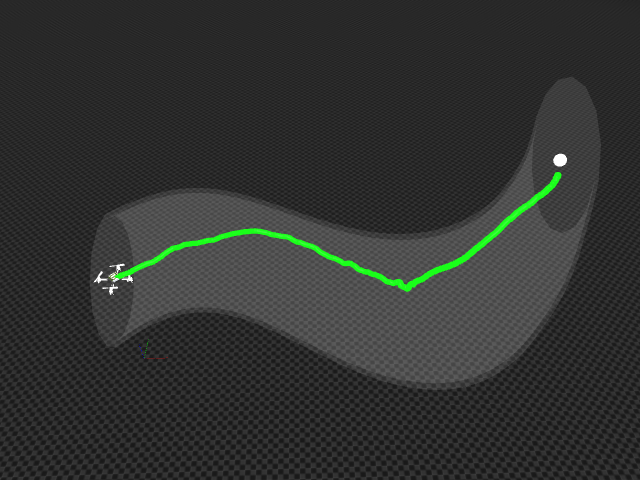}
    }
\caption{The  Safe MPC  rollout trajectories (in green) for learning safety constraints  with  neural features after 6, 9, 11 corrections (different columns). The first row show the top view,  and second row shows the side view of the navigation environment.}
\vspace{-10pt}
\label{fig.res_neural}
\end{figure}

\subsubsection{Result}
Fig. \ref{fig.res_poly} and Fig. \ref{fig.res_neural} show the rollout trajectories of the learned Safe MPC policy using polynomial and neural features, respectively. With polynomial features, the task is completed in 45 corrections, while the neural features require only 11. This is because the neural features, constructed from the second-to-last layer of pre-trained neural SDF, have already encoded the partial knowledge of spatial constraints. %The results confirm the effectiveness of the proposed method to handle both generic and learned feature spaces. 

\begin{figure}[h]
    \centering
    \includegraphics[width=0.75\linewidth]{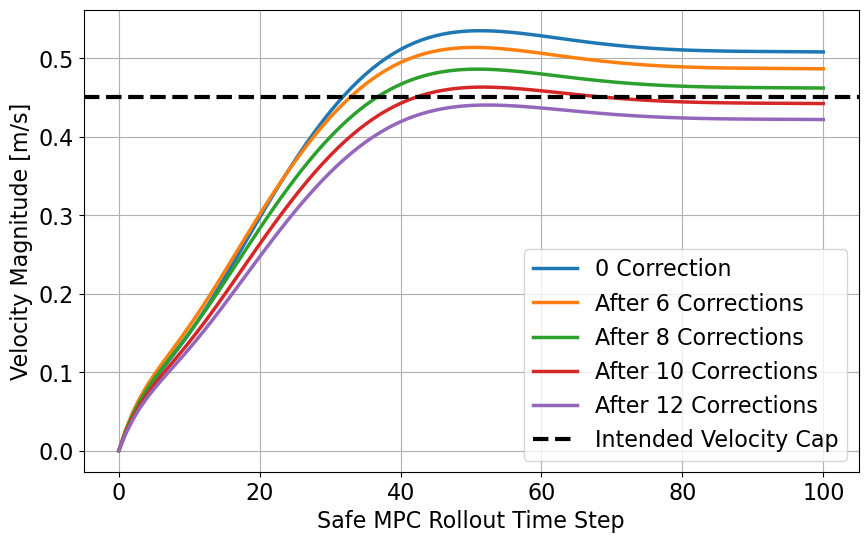}
    \caption{Velocity profile during the first 100 timesteps (acceleration phase) of the drone trajectory. After 12 corrections,
the quadrotor successfully learns to stay below the intended
velocity cap of 0.45 (dashed line). }
    \label{fig.vel_res}
\end{figure}

For the learning of velocity constraint, Fig. \ref{fig.vel_res} shows the velocity profile of the quadrotor during the Safe MPC rollout in the 100 time steps (acceleration phase) after different number of simulated corrections  (different colors). After 12 corrections, the quadrotor successfully learns to stay below the intended velocity cap of 0.45 (dashed line). Throughout this process, the positional safety constraint remains satisfied, with the quadrotor continuing to navigate within the tube to the target. This result demonstrates the flexibility of the proposed method in scenarios involving both positional and velocity constraints.

\subsection{Baseline Comparison}
\label{sec.comp_ablat}
In this session, we compare our method with two baseline methods using magnitude corrections in the pendulum environment. Since the literature on directly learning safety constraints from human feedback is sparse, the chosen baselines are an adaptation from the most relevant methods.  
% works is conducted. We also conduct two ablation studies focusing on justifying our choice of the MVE center in weight update and showing the fact that our method can succeed under the different choices of parameter $\gamma$. 
% In this section, we use the same pendulum environment in Section \ref{sec.pendulum}, with $\boldsymbol{c}_l = [-2,-2]\tran$ and $\boldsymbol{c}_h = [2,2]\tran$.

% There is no prior work to learn safety constraints from human feedback in MPC setting, so we compare the learning performance of our approach with two methods inspired by previous works on learning cost function from human corrections and learning safety constraints from human stopping signals in reinforcement learning:

     \subsubsection{Baseline 1: Gradient Matching} This baseline is inspired by \cite{losey2019learning}, where raw human correction (with magnitude) is directly used as the gradient direction. In a similar spirit, when a human correction $\mathbf{a}$ is made on robot motion $\xi_{\boldsymbol \theta}^\gamma$, we directly use the human correction as the true gradient label to learn the constraint $\boldsymbol{\theta}$. This leads to a learning loss
    \begin{equation}
        \mathcal{L}(\boldsymbol{\theta}) = \norm{\mathbf{a} + \nabla B(\boldsymbol\xi_{\boldsymbol \theta}^\gamma, \boldsymbol{\theta})}^2.
    \end{equation}
    
    \subsubsection{Baseline 2: Maximum Likelihood}. This baseline follows \cite{poletti2023learning}, where it views $e^{{g}_{\boldsymbol{\theta}}(\boldsymbol{\xi})}$  as a probability of human applying the "stop-feedback" to robot motion $\boldsymbol{\xi}$, when the robot reaches in a $\epsilon_g$-proximity of the safety boundary, i.e., when ${g}_{\boldsymbol{\theta}_H}(\boldsymbol{\xi}_{\boldsymbol{\theta}}^\gamma)>-\epsilon_g$ . Thus, a dataset of trajectory-correction paris, i.e.,  $\mathcal{D} = \{\mathbf{a}_i, \boldsymbol{\xi}_i \}_{i=0}^{N}$, can be used as the training data to learn $\boldsymbol{\theta
    }$ using maximum likelihood, while also penalizing positivity of  ${g}_{\boldsymbol{\theta}} (\boldsymbol{\xi}_i)$ as a regulizer term. This leads to the following loss
    \begin{equation}
        \mathcal{L}(\boldsymbol{\theta}) = \sum\nolimits_{i=1}^N \big( -{g}_{\boldsymbol{\theta}}(\boldsymbol{\xi}_i) + \alpha\cdot\mathrm{softplus}( g_{\boldsymbol{\theta}}(\boldsymbol{\xi}_i)) \big).
    \end{equation}
    We choose $N=15$ (similar to the correction number needed for our method to converge) and  $\alpha=5$ for its best performance.

The comparison between our method and the gradient matching baseline focuses on data efficiency. The convergence criteria of $\norm{\boldsymbol{\theta}-\boldsymbol{\theta}_H}\leq 0.02$ is set for both methods. 50 independent runs are performed and the total number of correction at convergence are counted. We use Adam optimizer \cite{diederik2014adam} with a learning rate of 0.02.  We summarize the results in Table \ref{tbl.baseline_comparison}. 
% The Gradient Matching method needs $240.38 \pm 120.70$ corrections to converge, while ours only needs $11.76 \pm 2.14$ corrections
The parameter error $\norm{\boldsymbol{\theta}-\boldsymbol{\theta}_H}$ versus learning iteration for a best run (for the baseline) is plotted in Fig. \ref{fig.baseline_I}. The comparison results demonstrate data efficiency of the proposed method compared to the gradient matching baseline, even though the latter uses both magnitude and direction information of corrections to learn.
% our proposed safe MPC alignment shows great superiority in data efficiency.  . Plots of one learning curve showing this advantage in convergence speed is shown in Fig. \ref{fig.baseline_1_converge}.

\begin{table}[h]
\begin{center}
\caption{Comparison of Safe MPC Alignment and Baseline I}
\begin{tabular}{cc}
     \hline
     Method & Correction count for convergence\\
     \hline
     Safe MPC Alignment & $11.76 \pm 2.14$ \\
     \hline
     Gradient Matching & $240.38 \pm 120.70$\\
     \hline
\end{tabular}
\label{tbl.baseline_comparison}
\end{center}
\end{table}

 \begin{figure}[h]
\centering
\vspace{-20pt}
\subfigure[Comparison: gradient matching]
    {
        \includegraphics[width=0.22\textwidth]{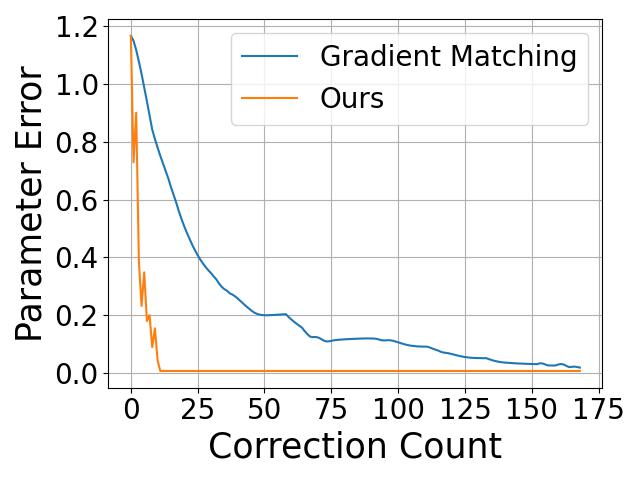}
        \label{fig.baseline_I}
    } 
\subfigure[Comparison: maximum likelihood]
    {
        \includegraphics[width=0.22\textwidth]{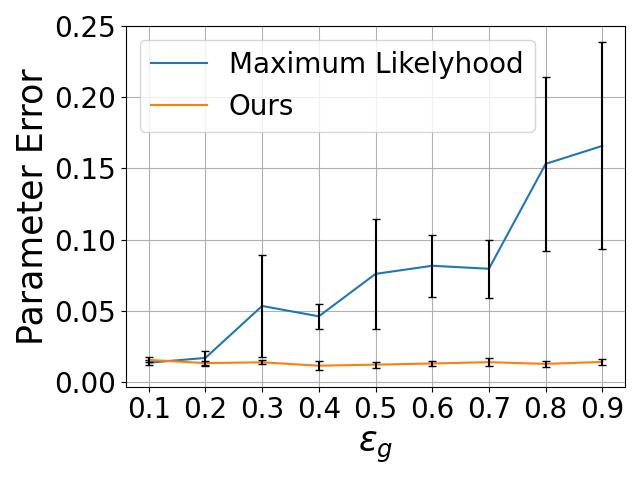}
        \label{fig.baseline_II}
    }
    \vspace{-5pt}
\caption{(a) Comparison with gradient matching baseline: the parameter error $\norm{\boldsymbol{\theta}-\boldsymbol{\theta}_H}$ versus learning iteration. This is best run with the fewest corrections for the baseline. Our method converges after 12 corrections, while the baseline needs more than 160 corrections. (b) Comparison with the maximum likelihood baseline under the different  $\epsilon_g$: the y-axis is $\norm{\boldsymbol{\theta}-\boldsymbol{\theta}_H}$ at convergence with the error bar denoting the standard deviation.  The maximum likelihood baseline fails to learn the correct weights when $\epsilon_g$ is large.}
\label{fig.baseline_converge}
\end{figure}

The comparison with maximum likelihood baseline uses the varying positions of simulated correction, i.e., varying $\epsilon_g$. Recall that the simulated corrections are generated  only when ${g}_{\boldsymbol{\theta}_H}(\boldsymbol{\xi}_{\boldsymbol{\theta}}^\gamma)>-\epsilon_g$. The value of $\epsilon_g$ determines the position of correction. In our comparison, we vary $\epsilon_g$ from 0.1 to 0.9 and track the converged parameter error for both methods. As shown in Fig. \ref{fig.baseline_II}, Safe MPC alignment consistently converges to the true $\boldsymbol{\theta}_H$ regardless of $\epsilon_g$, while the maximum likelihood baseline fails when $|\epsilon_g|$ is large. The results suggest the robustness of the proposed method against different correcting conditions compared to the maximum likelihood baseline.

% \begin{figure}[!htpb]
% \centering
% \includegraphics[height=4cm]{figs/comparison/baseline_2_epsilon.jpg}
% \caption{\textcolor{myblue}{Behavior of Maximum Likelihood Method and Safe MPC alignment under the different value of $\epsilon_g$. In the plot, the y-axis is $\norm{\boldsymbol{\theta}-\boldsymbol{\theta}_H}$ and the error bar denotes half of the standard deviation.  The Maximum Likelihood method fails to learn the correct weights when $|\epsilon_g|$ is big, while our method still has good precision of constraint learning.}}
% \label{fig.baseline_2}
% \end{figure}

\subsection{Ablation Study}
We conduct two ablation studies to show the performance of our algorithm given different design aspects. The task is still based on the pendulum task in Section \ref{sec.pendulum}.

\subsubsection{Choosing $\boldsymbol{\theta}_i$: MVE centering versus random sampling}
We compare two strategies in selecting  $\boldsymbol{\theta}_i$ from the hypothesis space $\Theta_{i-1}$: MVE centering (used in our method) and random sampling.  For the latter,  $\boldsymbol{\theta}_i$ is sampled from the uniform distribution supported by $\Theta_{i-1}$.
We compare the total number of corrections at learning convergence under the two strategies.

With 50 independent runs, we observe the MVE centering requires $11.76 \pm 2.14$ corrections, while random sampling  $11.88 \pm 3.50$  corrections. While the mean value of the correction count looks similar, the random sampling shows a slightly larger variance. This is understandable given sampling randomness. We also observed that the sampling process introduced a large computational overhead, which is around 25\% time slower than MVE centering. In contrast, the MVE centering used in our algorithm shows low variance and computational efficiency. Because selecting $\boldsymbol{\theta}_i$ is cast as a convex programming for MVE centering, such an advantage would be more significant for a high-dimensional parameter space. 

\subsubsection{Choosing $\gamma$}
We study how the setting of $\gamma$ affects the learning performance. We vary $\gamma$ from 0.01 to 10. Fig. \ref{fig.ablation_gamma} shows the total number of corrections needed at convergence versus different $\gamma$s.  The results show that the learning performance remains the same given different settings of $\gamma$. Since $\gamma$ is a parameter to weigh the safety conservativeness in the penalty-based MPC (\ref{equ.robot_mpc_approx}), it means 
our method can successfully learn a safety constraint when the robot executes with a Safe MPC policy of different conservativeness levels. 
\begin{figure}[h]
\centering
\includegraphics[height=3cm]{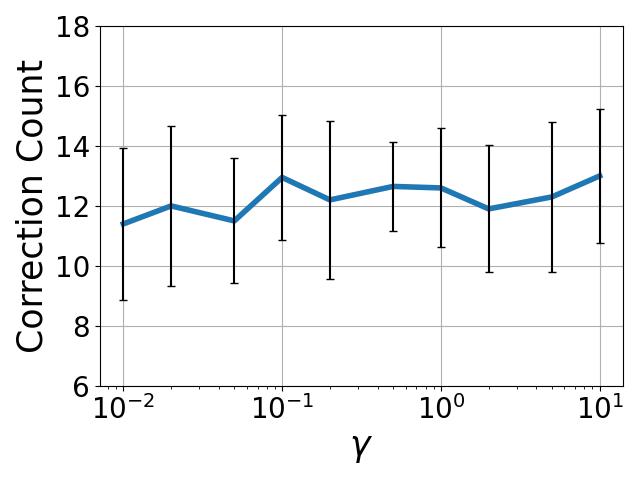}
\vspace{-10pt}
\caption{The total number of corrections needed for convergence under different $\gamma$, the x-axis is of log scale.}
\label{fig.ablation_gamma}
\end{figure}

\section{User Study in Simulation Games}

We conduct two human user studies to evaluate  the Safe MPC Alignment. One is the drone navigation game and the other the  Franka robot arm reaching game. In both games, the robot is controlled by an MPC policy and human participants provide real-time corrections through a keyboard. The objective of each game is to teach a robot to successfully learn Safe MPC to move through environments with obstacles.  Section \ref{sec.simulation_setting} describes the settings. Section \ref{sec.sim_participants} details the user study procedure. The results  are given in Section \ref{sec.sim_result}.

\subsection{Two Simulation Games}
\label{sec.simulation_setting}
In both user studies,  human participants are asked to give correction to the following two physics-based simulation games developed based on  MuJoCo \cite{todorov2012mujoco}. 

(1) Drone Navigation Game.
As shown in Fig. \ref{subfig.drone}, the goal of this game is to let a drone to successfully navigate through a narrow gate (yellow rectangle), reaching different targets (shown in yellow points behind the gate). The safety constraint $\boldsymbol{\theta}\tran\boldsymbol{\phi} $  is parameterized using a linear combination of spacial radial basis functions (RBFs), defined 
on y-z plane at $x=10$. The drone's control inputs are the thrusts of four rotors and human correction is applied via keyboard, with specific keys mapping to the drone's control space. See the details in Appendix \ref{sec.uav_game}.

(2) Franka arm reaching game.
As shown in Fig. \ref{subfig.arm}, the goal of the Franka arm reaching game is to teach a Franka arm such that its gripper hand successfully reaches different target positions (cyan ball) through an narrow slot (two yellow bars). Note that the safety constraint is only applied to the end effector and the collision between the link and the slot is ignored. The safety constraint function $\boldsymbol{\theta}\tran\boldsymbol{\phi} $ is parameterized as a weighted combination of RBFs defined on the $z$ axis and rotations around the $x$ axis. Human corrections are applied via keyboard, with specific keys mapping to the robot control space $\mathbb{R}^6$ (we use operational space control). See the details in Appendix \ref{sec.arm_game}.

In both games, the robots are controlled with  Safe MPC policy \eqref{equ.robot_mpc_approx} but fail to maintain safety with the initial safety constraints. To simulate an emergency stop, the “Enter” key is used as a reset command, allowing users to reset the environment to its initial state whenever the participant deems the robot violating the safety constraint.

%  \begin{figure}[!htpb]
% \centering

% \subfigure[Tracking View]
%     {
%         \includegraphics[width=0.18\textwidth]{figs/UAV/mj_100.jpg}
%     } 
% \subfigure[Fixed View]
%     {
%         \includegraphics[width=0.18\textwidth]{figs/UAV/aux_100.jpg}
%     }

% \caption{Drone navigation game.  (a)  Tracking view. (b)  Fixed view. Users can see the motion of the drone with the help of these two views. The $xyz$-axis (red, green, and blue axes), the gate (yellow bars), the drone, the local plan solved from the current MPC (red curve) and the motion target (yellow balls) are visible to users. In each game reset, the user will see one target randomly.}

% \label{fig.uav_views}
% \end{figure}

% \begin{figure}[!htpb]
% \centering

% \subfigure[Starting Configuration]
%     {
%         \includegraphics[width=0.18\textwidth]{figs/Arm/mj_89_axis.jpg}
%         \label{subfig.arm_start}
%     } 
% \subfigure[Motion Illustration]
%     {
%         \includegraphics[width=0.18\textwidth]{figs/Arm/arm_comp_axis.jpg}
%         \label{subfig.arm_motion}
%     }

% \caption{Franka arm reaching game. (a) The starting pose of the arm. RGB arrows denote the world $xyz$-frame. (b)  One example of successful reaching, and long green arrow denotes the motion direction. Users can see the motion of the robot arm by adjusting the view. The slot, the robot arm, and current motion target (cyan ball) are all visible to users.}

% \label{fig.arm_views}
% \end{figure}

\begin{figure}[h]
\centering

\subfigure[Drone Navigation game]
    {
        \includegraphics[width=0.18\textwidth]{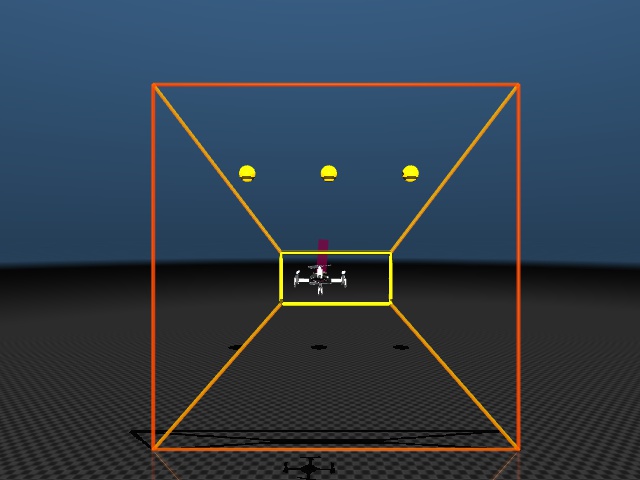}
        \label{subfig.drone}
    } 
\subfigure[Franka arm reaching game]
    {
        \includegraphics[width=0.18\textwidth]{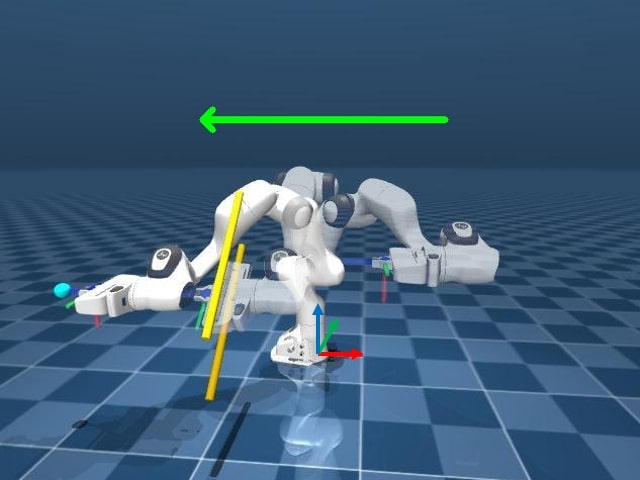}
        \label{subfig.arm}
    }

\caption{Two simulated games in user study. (a) Drone Navigation game. (b) Franka arm reaching game.}

\label{fig.sim_games}
\end{figure}

\begin{figure*}[h]
    \centering
    \subfigure[After 2 corrections]
    {
        \includegraphics[width=2.0in]{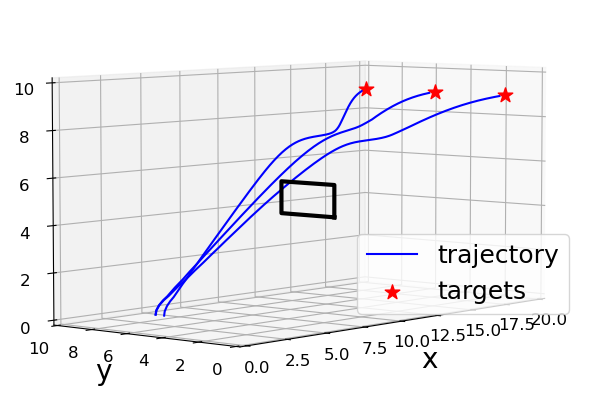}
    }
    \subfigure[After 6 corrections]
    {
        \includegraphics[width=2.0in]{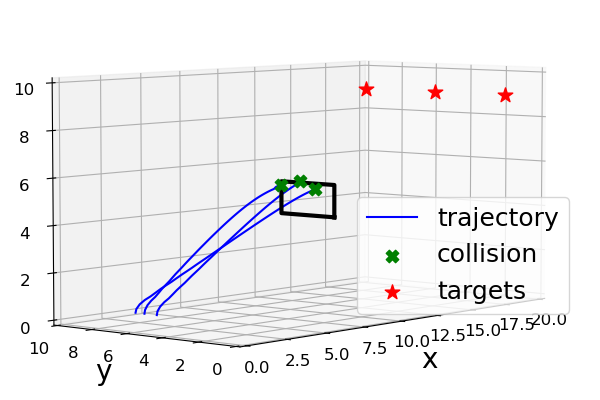}
    }
    \subfigure[After 12 corrections]
    {
        \includegraphics[width=2.0in]{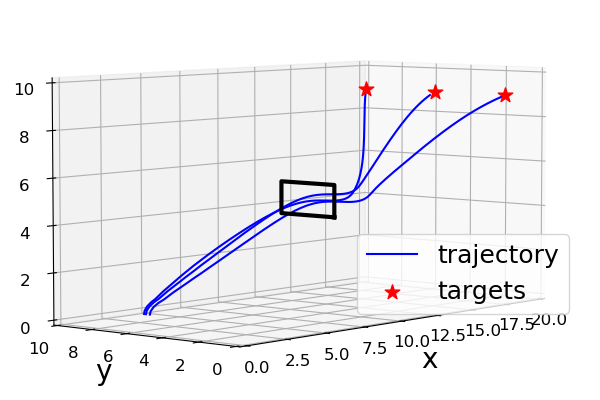}
    }
    \caption{Illustrations of the drone MPC rollout trajectory after $2,6,12$ corrections. The red stars denote the targets and blue curves denote the trajectories. The gate is drawn as black frame and the collision is highlighted in green crosses. User corrections typically lower the drone to make it pass through the gate.}
    \label{fig.uav_trajs}
\end{figure*}

\subsection{Participants, Procedures, and Metrics}
\label{sec.sim_participants}
\label{sec.sim_metrics}
\subsubsection{Participants}
We recruited 15 participants from Arizona State University, including undergraduates, masters, and PhD students with different backgrounds. Among all participants, 13 are male and 2  female. All participants are new to the user study. This user study had been reviewed and approved by the Institutional Review Board (IRB), and all participants had signed the consent forms. 

\subsubsection{Procedure}
Each participant was asked to play the two simulation games. In each game, each participant plays 10 trials.
In the drone navigation game, a trial is successful when the drone navigates through the gate without collision and reaches different targets  6 consecutive times with the learned safe MPC policy. A trial is deemed failed when there are too many corrections (exceeding  25), or the user believes the robot cannot be corrected anymore and gives up.
In the Franka reaching game, a trial is deemed successful when the robot arm reaches different targets with the end-effector having no collision with the slot. The failure of one trial is determined when there are too many corrections (exceeding 25) or the user believes the robot cannot be corrected anymore and gives up.
The result of every trial is recorded. Before formally starting the study, every user is instructed about the settings of the game and is given 2 chances to interact with the environment hands-on to get familiar with the keys.

\subsubsection{Metrics for Evaluation}
Two metrics are used to evaluate the outcomes:
(I) \textit{Success rate}: The overall  success rate suggest the  \textbf{effectiveness} of the proposed method.
(II)  \textit{The number of corrections needed for a successful trial}: The averaged number of user corrections for successful trials is a measure for the \textbf{efficiency} of the proposed method. Fewer corrections in successful trials suggest better human data efficiency.

% The result of the metrics above is recorded and evaluated in both games to justify the effectiveness and efficiency of the proposed method.

\subsection{Result and Analysis}
\label{sec.sim_result}

% {\color{myblue}
% \subsubsection{Computational Speed}
% In both simulations, every update of weights $\boldsymbol{\theta}$ can be completed around 0.15s with identical hardware setting in \ref{sec.pendulum}. Although the simulation is paused during the period of computing updated $\boldsymbol{\theta}_i$, the whole process of the simulation is relatively smooth because human corrections are sparse in time. 
% }

\subsubsection{Results}

\begin{table}[!htpb]
\begin{center}
\caption{Result of user study for simulation games}
\begin{tabular}{ccc}
     \hline
     Environment & Success Rate & Correction count\\
     \hline
     Drone navigation & $89.3\% \pm 7.7\%$ & $15.7 \pm 10.6$\\
     Franka arm reaching &  $89.3\% \pm 11.2\%$ & $16.5 \pm 6.1$\\
     \hline
\end{tabular}
\label{tbl.sim_result}
\end{center}
\end{table}
We report the success rate and correction count in a successful game trial in Table \ref{tbl.sim_result}. Both metrics are averaged across users. The success rate is $89.3\% \pm 7.7\%$ for the drone navigation game and $89.3\% \pm 11.2\%$ for the Franka arm reaching game. The number of corrections for a successful trial is $15.7 \pm 10.6$ for the drone navigation game and   $16.5 \pm 6.1$ for Franka arm reaching game.  
In both games, each learning update of $\boldsymbol{\theta}$ takes approximately 0.15 seconds, during which the simulation briefly pauses. This brief pause is imperceptible to the user, especially since human corrections occur infrequently.

 The above result shows that the proposed method is effective (i.e., with a high success rate) in learning a safety constraint for varying tasks (recall targets are changing across trials).  It also shows the efficiency of the proposed method: successful learning only requires tens of human corrections. Since this method is online and the user can perform multiple corrections in one robot rollout trajectory, the constraint can be learned in a few runs. In Table \ref{tbl.sim_result}, one can note that the drone navigation task has a higher variance in the correction count than the Franka arm reaching task. This could be because the drone game has a larger motion range and longer rollout horizon. As a result, human correction strategies, such as correction position and timing,  vary more.

\subsubsection{Analysis}
To illustrate learning progress, in Fig. \ref{fig.uav_trajs} and Fig. \ref{fig.arm_behavior},  we show the intermediate robot motion after different numbers of human corrections.
Fig. \ref{fig.uav_trajs} shows the safe MPC rollout trajectory of the drone after different numbers of human corrections. Fig. \ref{fig.arm_behavior} shows the keyframes of the robot gripper when passing the slot after a different human corrections.

\begin{figure}[!htpb]
\centering

\subfigure[After 4 Corrections]
    {
        \includegraphics[width=0.45\textwidth]{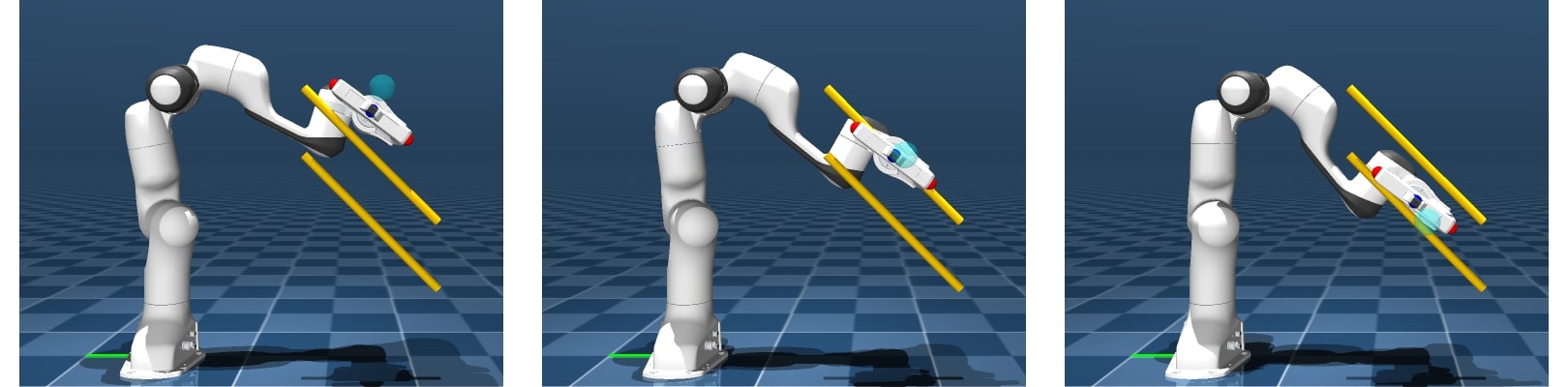}
    } 
% \subfigure[]
%     {
%         \includegraphics[width=0.45\textwidth]{figs/Arm/8_cat.jpg}
%     }
\subfigure[After 17 Corrections]
    {
        \includegraphics[width=0.45\textwidth]{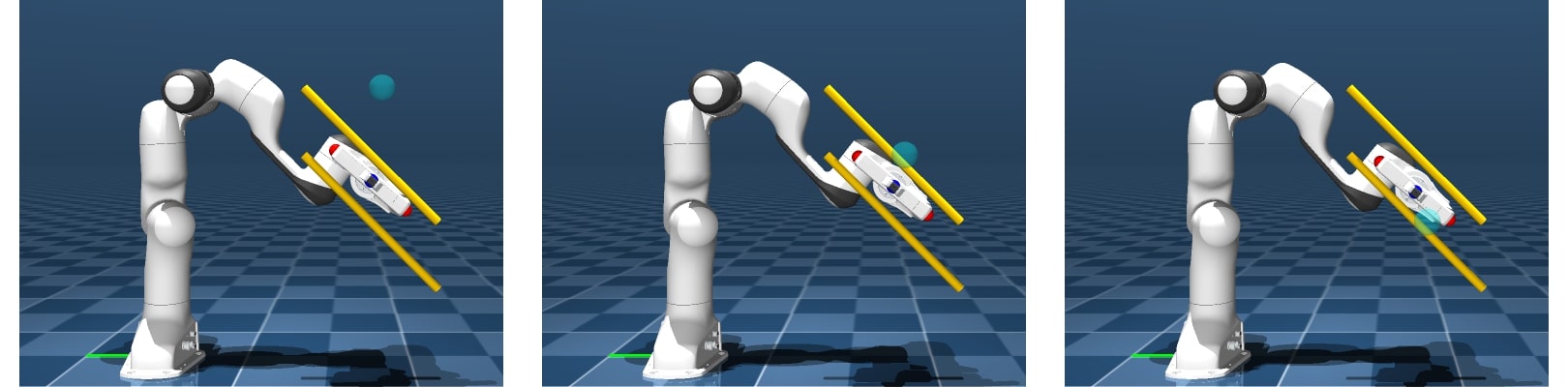}
    }

\caption{Screenshots of the robot gripper when passing through the slot after $k=4,17$ corrections,  given three different targets (cyan ball in different columns). (a): After 4 corrections, the gripper collides with the slot.  (b): After 17 corrections, the gripper succeeds in passing the slot for all targets.}
\label{fig.arm_behavior}
\end{figure}

To  understand the learned safety constraint ${g}_{\boldsymbol{\theta}}(\boldsymbol{\xi})$, in Fig. \ref{fig.uav_heatmap}, we draw the heat map of the learned ${g}_{\boldsymbol{\theta}^*}(\boldsymbol{\xi})$  for the drone navigation game, with $\boldsymbol{\theta}^*$ the averaged of weights learned. Recall in drone navigation game,  ${g}_{\boldsymbol{\theta
}}$ is defined 
on y-z plane at $x=10$. In the heat map, we draw  $0$-level contour of ${g}_{\boldsymbol{\theta
}^*}$ in dashed lines, which is the learned safety boundary. Also, we plot  y-z plane projection of the gate, which is the black box.

Fig. \ref{fig.uav_heatmap} shows that the learned safety boundary (corresponding to ${g}_{\boldsymbol{\theta}^*}(\boldsymbol{\xi})\leq 0$) is only similar to the top edge of the gate. This is explained below. As shown in Fig. \ref{subfig.drone}, the targets are set higher than the gate, and thus during the learning process, the motion of the drone is always going up, as shown in Fig. \ref{fig.uav_trajs}. As a result, the majority of human corrections happen when the drone approaches the top edge of the gate, leaving the learned safety boundary accurate near the top edge of the gate and bottom edge untouched. 
It is expected that to learn a safety constraint that fully represents the gate,  lower target positions should be added so that the drone will receive human corrections near the bottom edge of the gate.

\begin{figure}[!htpb]
\centering

\includegraphics[width=0.3\textwidth]{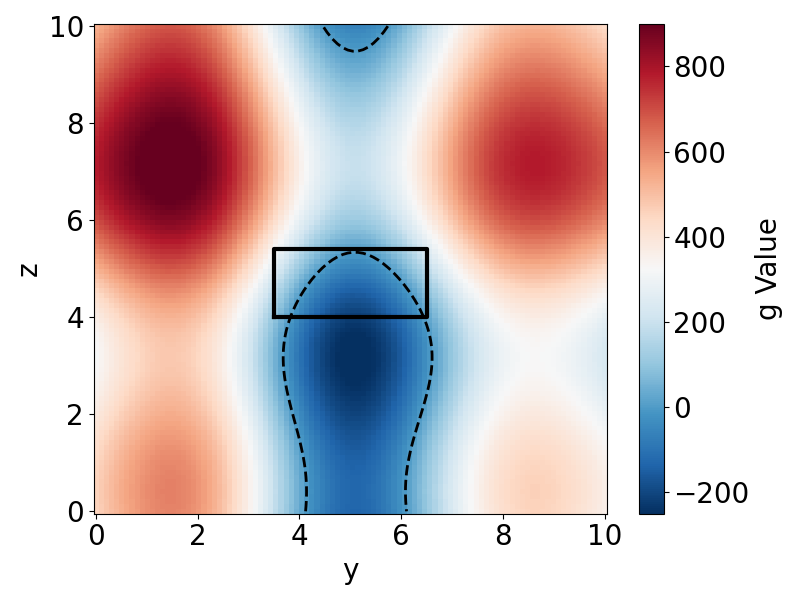}
\caption{Heatmap for the learned constraint in the drone game. The black solid line denotes the actual position of the gate in y-z plane. The black dashed curve denotes the boundary of the learned safety zone. Since the majority of human directional corrections happen when the drone approaches the top edge of the gate, the learned safety boundary is accurate near the top edge of the gate and the bottom boundary is untouched.}
\label{fig.uav_heatmap}
\end{figure}

\section{Real-world Experiment}
We conduct a user study to evaluate the performance of the {Safe MPC Alignment} in a real-world experiment. We focus on the task of \emph{mobile liquid pouring}: the Franka robot arm performs liquid pouring while moving from one location to targets. 
The task is challenging because it is hard to manually define a safety constraint to accomplish liquid pouring without spilling. The robot arm will learn a safe MPC controller from human physical corrections.

\subsection{Hardware Setup}

% \begin{figure}[!htpb]
% \centering
% \includegraphics[width=0.35\textwidth]{figs/realworld/photo_sys.jpg}
% \caption{Mobile water pouring using Franka robot arm. \circled{1}: The robot arm. \circled{2}: The watering pot. \circled{3}: The bar sink. \circled{4}: The scale to measure the amount of water poured into the sink. The target is marked as the red star. The long arrow indicates the motion direction, and the RGB  frame is the world frame.}
% \label{fig.real_sys_photo}
% \end{figure}

\begin{figure}[h]
\centering

\subfigure[]
    {
        \includegraphics[height=3.0cm]{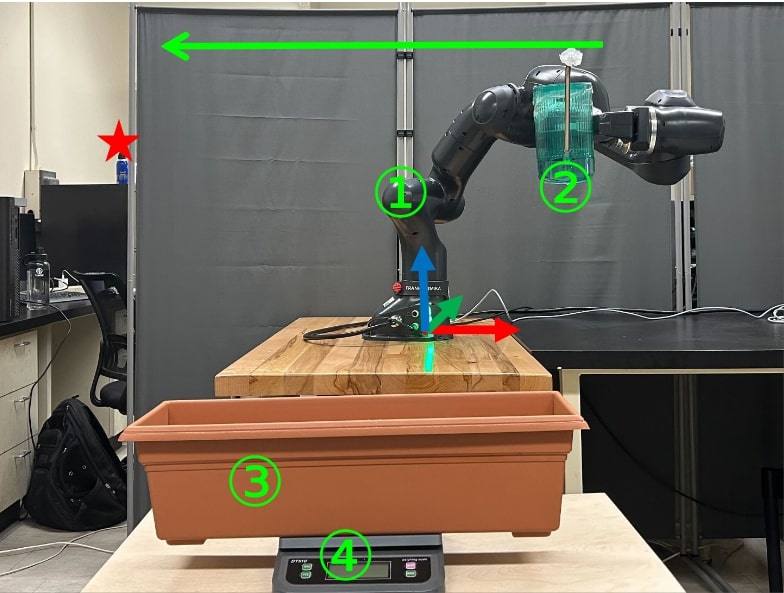}
        \label{fig.real_sys_photo}
    } 
\subfigure[]
    {
        \includegraphics[height=3.0cm]{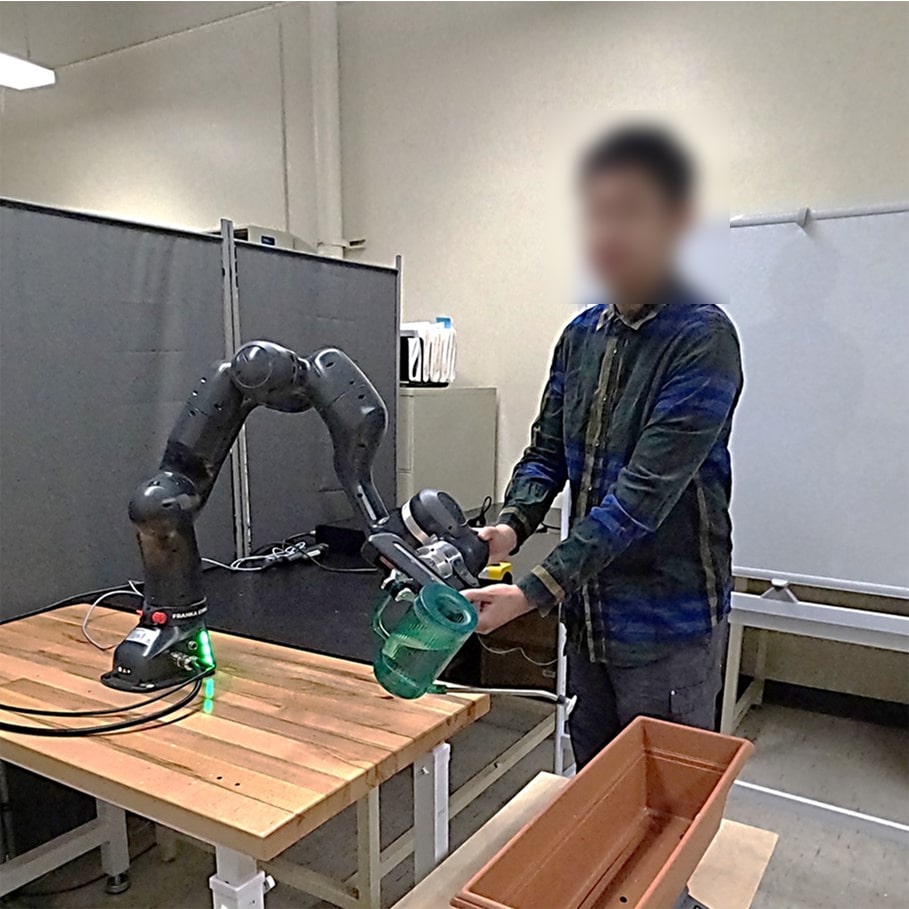}
        \label{fig.real_sys_corr}
    }
\caption{Illustration images of the real-world experiment. (a) Mobile water pouring task using Franka robot arm. \circled{1}: The robot arm. \circled{2}: The watering pot. \circled{3}: The bar sink. \circled{4}: The scale to measure the amount of water poured into the sink. The target is marked as the red star. The long arrow indicates the motion direction, and the RGB  frame is the world frame. (b) Illustration of human correction to the robot in the real-world experiment, by applying a contact force to the robot arm. The delta position of the end effector resulting from human contact interaction is used as human directional correction.}

\label{fig.real_sys_correction}
\end{figure}

The setup of the mobile water pouring task is shown in Fig. \ref{fig.real_sys_photo}.  
 A water pot is gripped by the robot's end-effector. The robot  moves from the starting point $[0.22,-0.5,0.48]\tran$ (with  quaternion $[0,-0.707,0,0.707]\tran$) to target point $[-0.65,-0.5,0.5]\tran$ (with quaternion $[0,-0.707,0,0.707]\tran$). During the motion, the robot needs to pour more than 50 grams of water into a bar sink. The bar sink is placed in the middle of the robot's path, but the precise size and location of the bar sink are unknown.

\begin{figure}[h]
\centering
\includegraphics[width=0.30\textwidth]{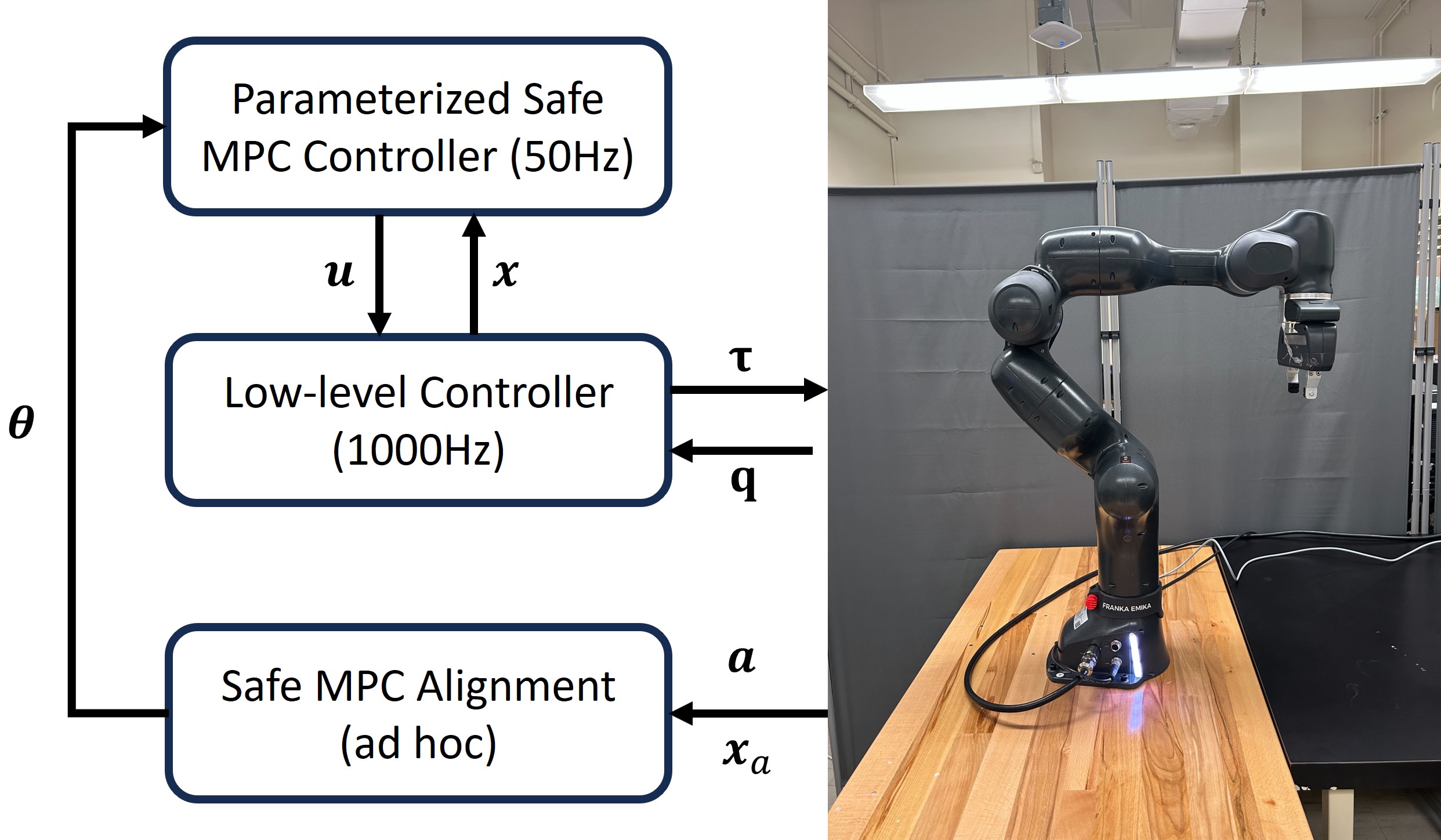}
\caption{System diagram for the hardware implementation. Different blocks indicate separate processes  communicated via ROS.}
\label{fig.real_sys}
\end{figure}

Fig. \ref{fig.real_sys} shows  three main components of the system: the parameterized Safe MPC controller (high-level controller), the low-level controller, and the Safe MPC alignment module. The Safe MPC controller follows the safe MPC formulation of Franka arm reaching game in Appendix \ref{sec.arm_game} (recall that the learnable safety constraint is parameterized by RBFs). The parameters of this Safe MPC controller are listed in Table \ref{tbl.real_params}. The Safe MPC controller receives the end-effector's current pose $\boldsymbol{x}$ and generates the action $\boldsymbol{u}$ (see (\ref{equ.franka_xu}) in Appendix), which is the desired velocity of the end-effector. The low-level controller implements the operational space control to track the $\boldsymbol{u}$ given by the safe MPC controller. 
  
The Safe MPC Alignment module updates the parameter $\boldsymbol{\theta}$ of the Safe MPC controller. The update only happens when the alignment module detects a human correction $\boldsymbol{a}$. The module implements our {Safe MPC Alignment} algorithm. Note that the alignment module and controller modules runs in different processes, thus the learning update does not affect the real-time MPC control of the robot arm. ROS is used for communication between different processes.
\begin{table}[h]
\begin{center}
\caption{parameter list in real-world user study}
\begin{tabular}{cc}
     \hline
     Parameter & Value \\
     \hline
     $ Q_r$ & $\mathrm{diag}(0.8,0.8,0.8)$ \\
     $k_q$ & 0 \\
     $R_v$ & $\mathrm{diag}(30,30,1)$\\
     $R_w$ & $\mathrm{diag}(0.85,0.85,0.85)$ \\
     $S_r$ & $\mathrm{diag}(80,80,80)$  \\
     $\mu_q$ & 80\\
     \hline
\end{tabular}
\label{tbl.real_params}
\end{center}
\end{table}

% In the Safe MPC, the cost function is set to penalize the current and target, as in Appendix \ref{sec.arm_game}. Since it is difficult to explicitly define a parameterized safety constraint using the state of pouring water, we instead define the constraints only using the robot state. Specifically, we use the weighted RBFs to parameterize the safety constraint, similar to the setting of Franka arm reaching game

As shown in Fig. \ref{fig.real_sys_corr}, to give a physical correction to the Franka robot arm, a human user need physically touch/hold any link of the robot arm, and the applied force will immediately activate the zero-torque control mode. Suppose the joint positions of the robot arm before and after human physical correction are $\mathbf{q}_{\text{before}}$ and ${\mathbf{q}_{\text{after}}}$, respectively, we calculate the corresponding raw human correction $\boldsymbol{a}_r$ using
\begin{equation}\label{equ.realworld_a}
    \boldsymbol{a}_r=\mathrm{J}(\mathbf{q}_{\text{before}}) ({\mathbf{q}}_{\text{after}}-\mathbf{q}_{\text{before}}),
\end{equation}
with $\mathrm{J}$ the Jacobian of the end-effector. The calculated $\boldsymbol{a}_r\in\mathbb{R}^6$ is considered as the displacement of the end-effector pose, intended by the human user.
% In the computed $\boldsymbol{a}_r\in\mathbb{R}^6$, we mask the entries in the other dimensions with $0$s and only keep the entries corresponding to the tilt angle along $x$-axis and the height in $z$-axis, since they are the only dimensions relevant to the given task. The masked vector are is as the actual human correction $\boldsymbol{a}$  fed to the learning module.

\begin{figure*}[!htpb]
\centering

\subfigure[After 2 Corrections]
    {
        \includegraphics[width=0.25\textwidth]{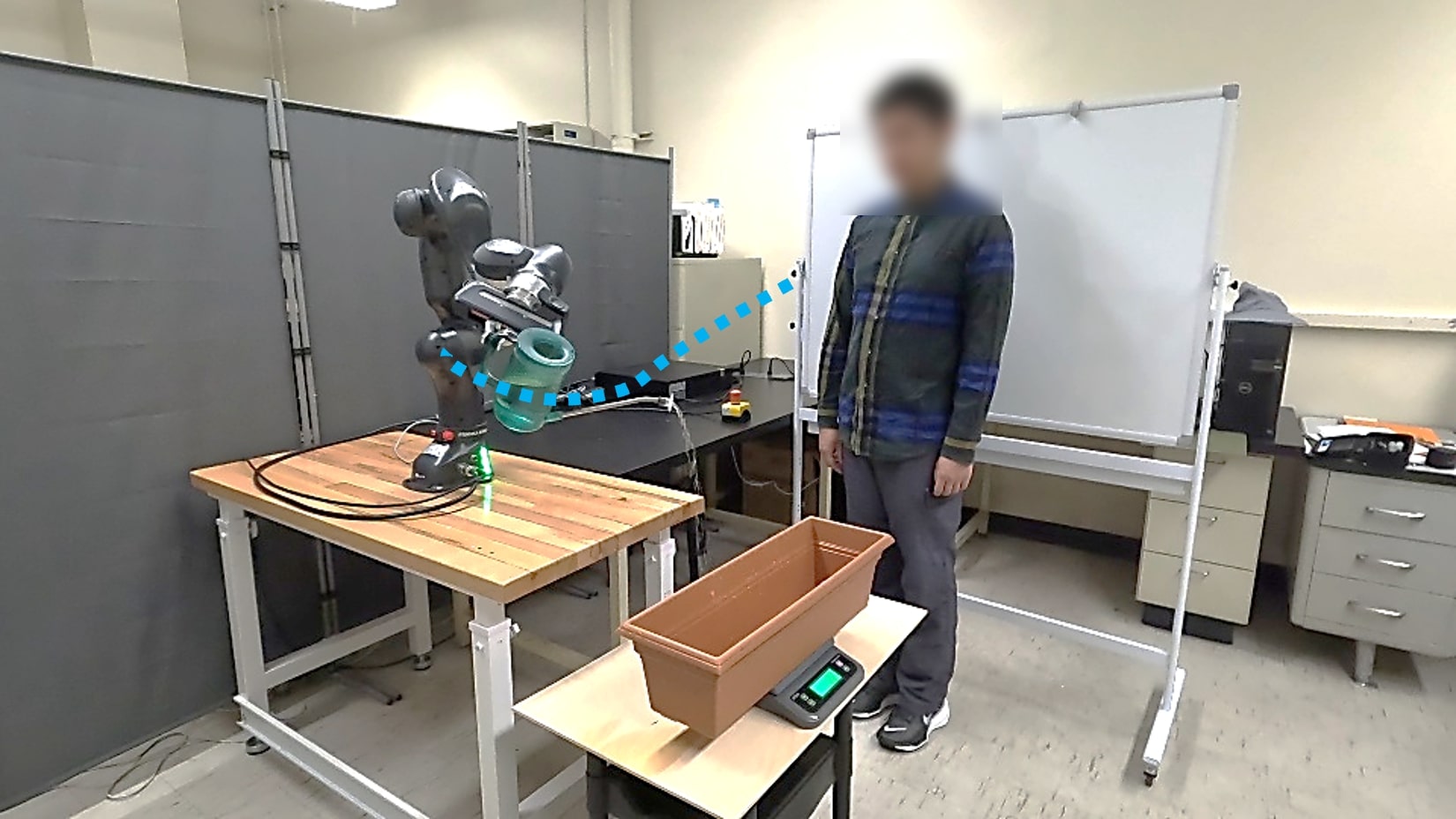}
    } 
\subfigure[After 3 Corrections]
    {
        \includegraphics[width=0.25\textwidth]{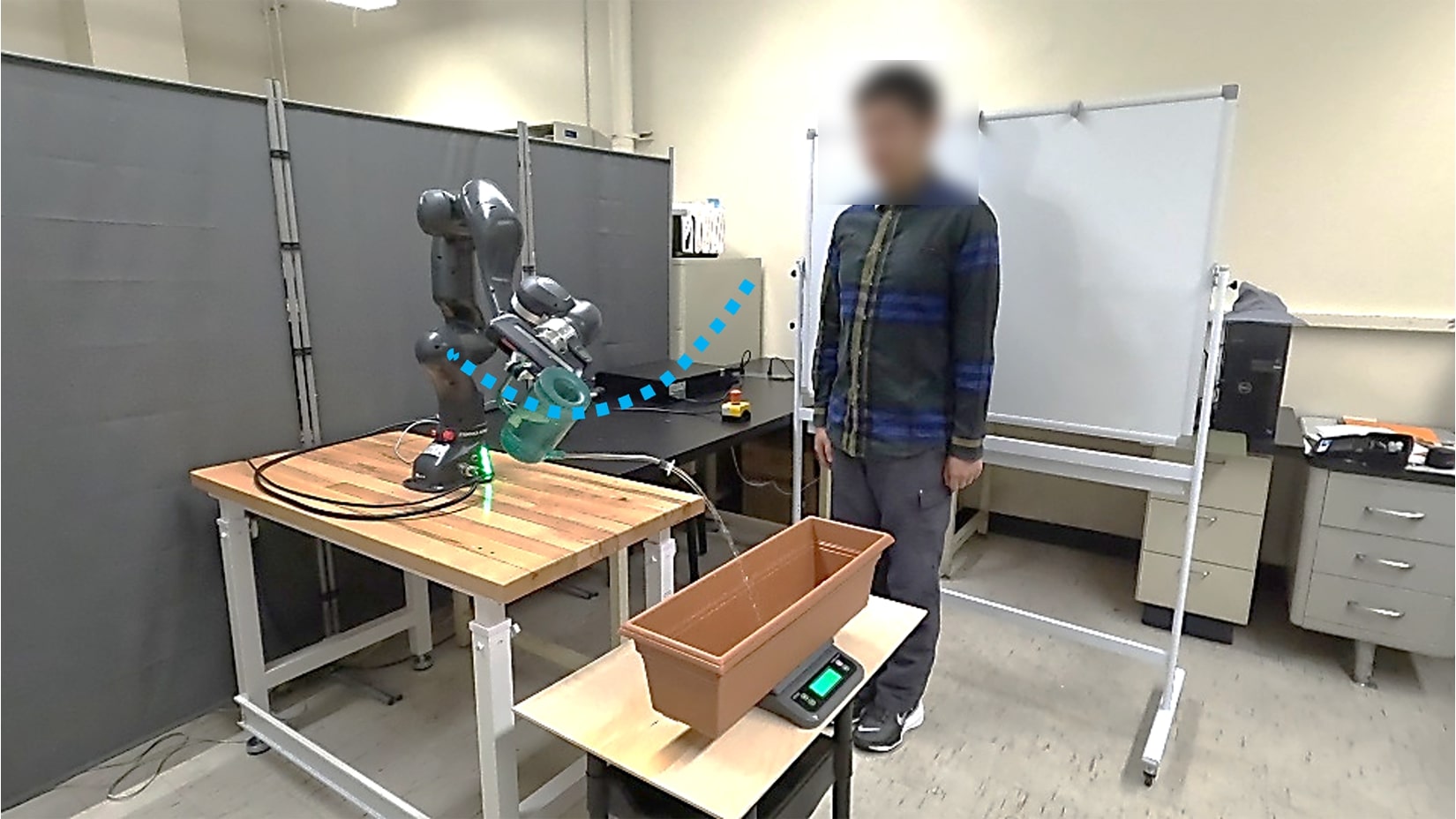}
    }
\subfigure[After 5 Corrections]
    {
        \includegraphics[width=0.25\textwidth]{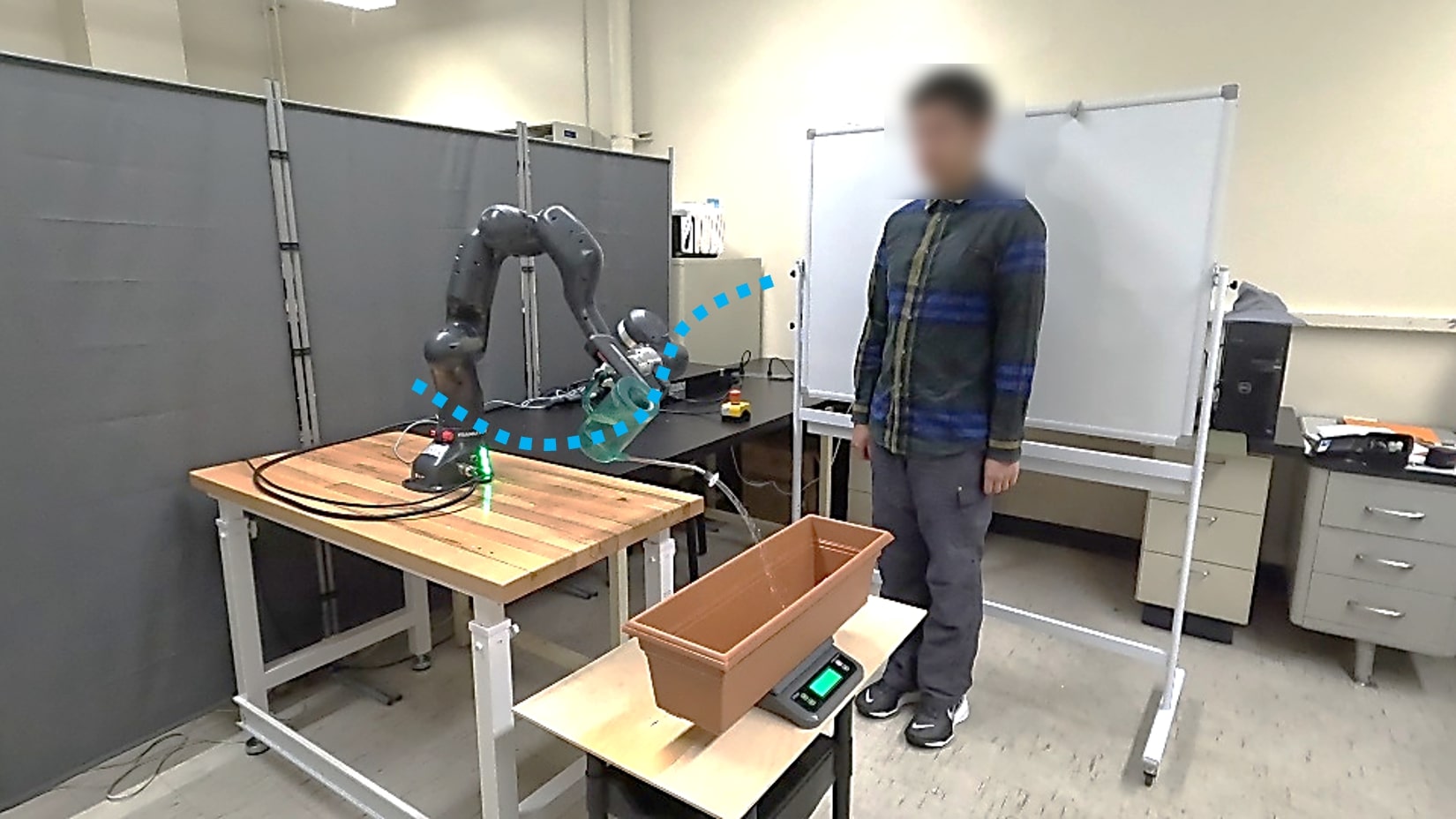}
    }

\caption{Illustrations of one trial of the real-world experiment. The trajectories of the watering pot is denoted by blue dashed lines. (a) After 2 corrections, the arm begins tilting the pot but spills out the water. (b) After 3 corrections, the arm can pour water into the sink but the human is not satisfied with its height and the altitude of pouring. (c) After 5 corrections, the arm learns to tilt only along the direction of the $x-$axis and successfully aligns with human intent.}

\label{fig.real_result}
\end{figure*}

\subsection{Participants, Procedures, and Metrics}
All participants recruited in Section \ref{sec.sim_participants} also participate in the hardware experiment. They are told of the objective before starting trials. One episode is defined as one robot motion from the starting position to the targets, during which the human can apply corrections at any time instances. One learning trial may include multiple episodes  (with a maximum of 15 corrections). During learning, the nozzle of the water pot is sealed, and we switch to the unsealed nozzle and measure the amount of water poured into the flower pot after the user is satisfied with the motion of the robot. The successful trial is defined when the amount of poured water is over 50g. Each user performs 5 trials, and the same metrics in Section \ref{sec.sim_metrics} are recorded.

\subsection{Result and Analysis}
In the real-world experiment, the success rate is $84.0\% \pm 16.7\%$ and the average number of corrections needed for successful trials is $6.88 \pm 3.26$. The results show the efficacy and efficiency of our method in real-world applications. To illustrate the learning progress,  Fig. \ref{fig.real_result} shows the robot behavior after different numbers of corrections in one trial.

Since this real-world Franka arm experiment is similar to the simulated Franka arm reaching game in Appendix \ref{sec.arm_game}, we compare the user study results for both and have the following conclusion and analysis.

\textit{(I) Real-world experiment has slightly lower success rate}: While the difference between simulation and real-world studies is small in success rate, we have identified some factors for such a difference. On human aspects, most of the participants have no experience or knowledge in operating robot hardware, and thus when directly interacting with hardware, their corrections can be 
conservative or ineffective due to safety considerations. On hardware side,  noisy correction signals and the possible timing delay between the recorded correction and the robot state can contribute to failure.

% The real-world experiment may be caused by the  state observation and communication lag between different modules, and command errors from multiple filters in the control loops. In addition, the total number of trials performed by each user in real-world tests is less than in simulation tests. The users can be less familiar with the system and may not be able to exploit their learned knowledge to boost the success rate by more trials. Another possible reason is the coupled correction of the tilt angle and height. When the users just want to correct one dimension, they may accidentally correct another. This can cause some unexpected behavior of the robot.

\textit{(II) Real-world experiment needs less human corrections}: This could be attributed to two reasons. First, the constraint needed for accomplishing the water pouring task could be more flexible than that for reaching task in the simulation game. This makes the safety constraints easy to learn for the real-world task. Second, different from simulation games, where human users can only perceive the robot arm from the screen, a human user has a better perception of the robot's motion in the real world. A good perception of the robot's state could help humans give more effective feedback each time, reducing the total number of human interactions for successful robot learning.

% The feasible region for the weights $\boldsymbol{\theta}$ is larger than when passing through two bars with fixed heights. The coupled correction of two dimensions also helps, since the user can correct the height and the tilt angle at one time, rather than pressing two different keys.

\section{ Discussion and Conclusion}
\subsection{Discussion}
Based on the above hardware and simulation user studies, we discuss some design choices and potential limitations of the proposed method.

\subsubsection{Learning Task-Independent Safety Constraints}
In the real-world and simulation experiments,  the learned safety constraint could be  task-specific, instead of task-independent. For instance, in the drone navigation game, since the three task targets are placed above the gate, the learned constraint shown in Fig. \ref{fig.uav_heatmap} can only reflect the upper edge of the unknown gate. One limitation is that the learned constraints have poor generalization to new tasks. For example, in drone navigation, if a new target is placed below the gate, the drone with the learned safety constraint Fig. \ref{fig.uav_heatmap} will fail.

To improve the generalization of the learned safety constraints or learning task-independent safety constraints, one solution is to increase the diversity of task samples. 
By doing so, it is expected that the robot can receive human correction on the entire boundary of the safety region.

%  The choice of task-specific safety constraints can lead to poor generalization to new tasks.As we have discussed 
% Apart from sampling the tasks exhaustively, there is one solution to re-design the features to be context-dependent, i.e. we add a context-input $\boldsymbol{c}$ to $\boldsymbol{\phi}$, using $\boldsymbol{\phi}(\boldsymbol{c},\boldsymbol\xi_{\boldsymbol{\theta}}^\gamma)$. $\boldsymbol{c}$ can be the explicit description of the task or the robot perception which implicitly defines the task. Using context-dependent features in constraint function can help the robot to learn generalized, task-adaptive constraints to achieve multi-task safety performance.

\subsubsection{Robustness Against Human Mistakes}
The proposed {Safe MPC Alignment}  is based on the hypothesis cutting. In our formulation,  a hypothesis cut (\ref{equ.core-neq}) is based on an average of multiple human corrections,  such an averaging treatment can mitigate human mistakes for one cut. But still, in practice, human irrationality can lead to the wrong cut of the hypothesis space, which is the main reason for failure in user studies. 
For example, in the drone navigation game, we found that failure game trials usually happen when a human gives too many wrong directional corrections. A  failure case is in Fig. \ref{fig.uav_failure}. %This can be explained as follows: the first few cuts can remove greater amount of the hypothesis space, and false cuts in the beginning stage are more likely to remove $\bar \Theta_H$ from the hypothesis space by mistake and cause failure.

\begin{figure}[!htpb]
\centering
\includegraphics[width=2.0in]{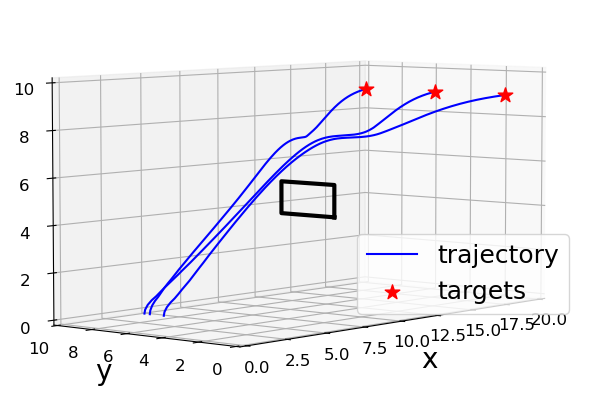}
\caption{A failure case in the drone game user study. A  user makes a wrong "up" correction at the beginning and fails to correct the motion of the drone.}
\label{fig.uav_failure}
\end{figure}

To further improve the robustness of the method against human mistakes, two extensions of the current method can be explored. The first is to add a relaxation factor $\epsilon\geq 0$ in our cutting hypothesis. For example, we can re-write (\ref{equ.core-neq}) as 
\begin{equation}
        \forall \boldsymbol{\theta}_H \in \bar\Theta_H,\ 
    \E_{\mathbf{a}}\left\langle \mathbf{a},-\nabla B(\boldsymbol\xi_{\boldsymbol \theta}^\gamma, \boldsymbol \theta_H) \right\rangle +\epsilon \geq 0.
\end{equation}
Another is to add a selection mechanism for human correction. For example, one can use a voting system to preserve more fraction of the hypothesis space and avoid cutting out the human safety set, as shown in \cite{xie2025robust}.

\subsubsection{Limitation of Weight-Feature Constraints}\label{sec.feature_selection_limitation}
In this work, we assume  the set of task-relevant features in $\boldsymbol{\phi}(\boldsymbol{\xi})$ is selected as a priori. These features define the space in which the safety constraints are represented. However, in practice, identifying appropriate features can be challenging. Poorly chosen features may fail to capture the intended safety constraints, leading to undesired learning outcomes. Furthermore, feature representations are often task-dependent—features that are effective in one task may not generalize to another. This task-specific dependence inherently limits the development of a unified framework for constraint learning across diverse tasks, making the lack of task-generalizable feature selection a key limitation of the proposed method.
At the same time, learning unstructured constraints entirely from scratch would demand extensive human interaction and  involvement, rendering such approaches impractical for most real-world applications. This challenge likely explains why recent works on learning from human feedback \cite{bajcsy2017learning,jin2022learning,losey2022physical} commonly assume task-specific feature structures or prior knowledge of the environment/tasks. 

We  recognize that identifying or learning a universal feature representation that can flexibly describe a wide variety of safety constraints remains an important and open research challenge. We consider this an exciting direction for future research, as it could significantly increase the applicability and effectiveness of safety constraint learning in real world.

\subsection{Conclusion}
This paper proposes a Safe MPC Alignment method to online learn a safety constraint in Safe MPC controllers from human directional feedback. The method, named hypothesis space cutting, has a geometric interpretation: it exponentially cuts the hypothesis space based on each received human directional feedback. To our knowledge, this is the first certifiable method to enable online learning of safety constraints with human feedback.  The method has a provable guarantee for the total number of human feedback needed for learning,  and also can declare hypothesis misspecification.  Numeric examples,  user studies on simulation games, and a user study on the real-world robot arm, covering varying robot tasks, were used to validate the effectiveness of the proposed approach. The results show that the proposed method can learn safety constraints effectively and efficiently.
\bibliographystyle{unsrt}
\bibliography{mybib} 

\appendices

\section{Proof of lemmas and theorems}
\subsection{Proof of Lemma \ref{lemma.bar_theta}}
\label{proof.lemma bar_theta}
\begin{proof}
    We prove this lemma by induction. First, we have $\bar \Theta_H \subset \Theta_0$.
    Suppose $\bar \Theta_H \subset \Theta_{i-1}$, $i=1,2,3,...$.
    From the previous assumptions (\ref{equ.core-neq}) and (\ref{equ.implication_constraint}), $\bar \Theta_H \subset\mathcal{C}_i$. Following the update (\ref{equ.intersection}), we have $\bar \Theta_H \subset \Theta_{i}=\Theta_{i-1} \cap \mathcal{C}_i$, which completes the proof.
\end{proof}

\subsection{Proof of Lemma \ref{lemma.linear_cut}}
\label{proof.lemma_linear_cut}
\begin{proof}
(\ref{equ.pwl_cut}) can be directly obtained from   (\ref{equ.C_steps}), (\ref{equ.barrier}) and (\ref{equ.safey_param}). Based on the definition of the penalty objective function (\ref{equ.barrier}), $\left< \mathbf{\bar{a}}_i,\nabla B(\boldsymbol\xi_{\boldsymbol \theta_i}^\gamma ,\boldsymbol \theta) \right> \leq 0$ is a linear fractional inequality
\begin{equation}
     \left< \mathbf{\bar{a}}_i,\nabla J(\boldsymbol\xi_{\boldsymbol \theta_i}^\gamma ) \right> + \gamma \frac{1}{- g_{\boldsymbol{\theta}}(\boldsymbol\xi_{\boldsymbol \theta_i}^\gamma)} \left< \mathbf{\bar{a}}_i,\nabla g_{\boldsymbol \theta}(\boldsymbol\xi_{\boldsymbol \theta_i}^\gamma ) \right> \leq 0,
     \label{equ.lin_ineq}
\end{equation}
with a positive denominator $- g_{\boldsymbol{\theta}}(\boldsymbol\xi_{\boldsymbol \theta_i}^\gamma)>0$ due to (\ref{equ.C_steps}).

Multiplying the denominator on both side and substituting the weight-feature expression of $g_{\boldsymbol{\theta}}$ in \eqref{equ.safey_param} yields $\boldsymbol{\theta}\tran \boldsymbol h_i \leq b_i$ in Lemma \ref{lemma.linear_cut}. The condition $- g_{\boldsymbol{\theta}}(\boldsymbol\xi_{\boldsymbol \theta_i}^\gamma)>0$ itself yields $\ {\boldsymbol{\theta}}\tran\boldsymbol{\phi}(\boldsymbol\xi_{\boldsymbol \theta_i}^\gamma ) < -\phi_0(\boldsymbol\xi_{\boldsymbol \theta_i}^\gamma )$ in Lemma \ref{lemma.linear_cut}.

This completes the proof.
\end{proof}

\subsection{Proof of Lemma \ref{lemma gh}}
\label{proof.lemma gh}
\begin{proof}
    For any $i=1,2,3,...$, the  plan $\boldsymbol{\xi}_{\boldsymbol \theta_i}^\gamma$ is a solution to the optimization (\ref{equ.robot_mpc_approx}), with $\boldsymbol{\theta}=\boldsymbol{\theta}_i$. Recall Lemma \ref{lemma sol},  $g_{\boldsymbol{\theta}_i}(\boldsymbol \xi_{\boldsymbol{\theta}_i}^\gamma)<0$. Extending the feature-weight expression of $g_{\boldsymbol{\theta}_i}(\boldsymbol \xi_{\boldsymbol{\theta}_i}^\gamma)$, one can get $\boldsymbol{\theta}_i\tran\boldsymbol{\phi}(\boldsymbol\xi_{\boldsymbol \theta_i}^\gamma) < -\phi_0(\boldsymbol\xi_{\boldsymbol \theta_i}^\gamma)$. 
    Since the optimal solution $\boldsymbol{u}^{\boldsymbol{\theta}_i,\gamma}_{0:T-1}$ is a minimizer to (\ref{equ.robot_mpc_approx}), it means
    \begin{equation}
    \label{equ.proof4_eq1}
        \nabla  B(\boldsymbol\xi_{\boldsymbol \theta_i}^\gamma , \boldsymbol \theta_i)=0.
    \end{equation}
    Note that with  $g_{\boldsymbol{\theta}_i}(\boldsymbol \xi_{\boldsymbol{\theta}_i}^\gamma)< 0$, the above $\nabla  B(\boldsymbol\xi_{\boldsymbol \theta_i}^\gamma , \boldsymbol \theta_i)$ is well defined.    
    With (\ref{equ.proof4_eq1}), 
    \begin{equation}\label{equ.proof4_eq2}
        \left< \mathbf a,\nabla B(\boldsymbol\xi_{\boldsymbol \theta_i}^\gamma , \boldsymbol \theta_i) \right> = 0,
    \end{equation}
    always holds. Extending  the above (\ref{equ.proof4_eq2}) leads to
    \begin{equation}
        \left< \mathbf{a},\nabla J(\boldsymbol\xi_{\boldsymbol \theta_i}^\gamma ) \right> - \gamma \frac{1}{g_{\boldsymbol \theta_i}(\boldsymbol\xi_{\boldsymbol \theta_i}^\gamma)}\left< \mathbf{a},\nabla g_{\boldsymbol \theta_i}(\boldsymbol\xi_{\boldsymbol \theta_i}^\gamma) \right> = 0.
    \end{equation}
    Due to  $-g_{\boldsymbol \theta_i}(\boldsymbol\xi_{\boldsymbol \theta_i}^\gamma)>0$, one can multiply  $-g_{\boldsymbol \theta_i}(\boldsymbol\xi_{\boldsymbol \theta_i}^\gamma)$ on both sides of the above equality, which leads to $\boldsymbol{\theta}_i\tran \boldsymbol h_i = b_i$ based on the definition (\ref{equ.h_and_b}). The second term in (\ref{equ.gh}) directly follows $g_{\boldsymbol{\theta}_i}(\boldsymbol \xi_{\boldsymbol{\theta}_i}^\gamma)< 0$. This concludes the proof.
\end{proof}
\subsection{Proof of Lemma \ref{lemma v}}
\label{proof.lemma v}
\begin{proof}
    By Lemma \ref{lemma gh}, the MVE center $\boldsymbol \theta_i=\bar{\boldsymbol d}_{i-1}$ is on the cutting hyperplane $\boldsymbol{\theta}\tran \boldsymbol h_i = b_i$. Directly from \cite{tarasov1988method}, one has
    \begin{equation}
        \frac{\mathbf{Vol} (\Theta_{i-1} \cap \{\boldsymbol \theta|\boldsymbol{\theta}\tran \boldsymbol h_i \leq b_i\})}{\mathbf{Vol} (\Theta_{i-1})} \leq 1-\frac{1}{r}.
    \end{equation}
    Since
    \begin{equation}
        \mathbf{Vol} (\Theta_{i-1} \cap \{\boldsymbol \theta|\boldsymbol{\theta}\tran \boldsymbol h_i \leq  b_i\}) \geq \mathbf{Vol} (\Theta_{i-1} \cap \mathcal{C}_i)=\mathbf{Vol} (\Theta_{i}),
    \end{equation}
    based on (\ref{equ.C_steps}), (\ref{equ.intersection}),
    one can get
    \begin{equation}
       \frac{\mathbf{Vol} (\Theta_{i})}{\mathbf{Vol} (\Theta_{i-1})} \leq 1-\frac{1}{r},
    \end{equation}
   which concludes the proof.
\end{proof}
\subsection{Proof of Theorem \ref{theorem_convergence}}
\label{proof.theorem_convergence}
\begin{proof}
We complete the proof in two parts. First, we prove the number of learning iterations is less than $K$, then we use this to prove there will be a $\boldsymbol{\theta}_k \in \bar\Theta_H$, $k=0,1,2...,K-1.$

We first prove the total number of iterations must be less than $K$ by contradiction. Suppose the learning iteration number $i \geq K$. 
% since
% \begin{equation}
%     \mathcal{C}_i=\{\boldsymbol \theta|\ \boldsymbol{\theta}\tran \boldsymbol h_i \leq b_i\} \cap \{{\boldsymbol{\theta}}\tran\boldsymbol{\phi}(\boldsymbol\xi_{\boldsymbol \theta_i}^\gamma) \leq -\phi_0(\boldsymbol\xi_{\boldsymbol \theta_i}^\gamma)\}.
% \end{equation}
With recursively applying lemma \ref{lemma v}, one can obtain
% \begin{equation}
%     \frac{\mathbf{Vol} (\Theta_{i}) }{\mathbf{Vol} (\Theta_{i-1})} \leq 1-\frac{1}{r}.
%     \label{algo_diff}
% \end{equation}
    % By the relation (\ref{algo_diff}) there is:
    \begin{equation}
        \mathbf{Vol} (\Theta_i) \leq (1-\frac{1}{r})^i \mathbf{Vol} (\Theta_0).
    \end{equation}
    When $i \geq K$, 
    \begin{equation}\label{equ.proofthm1_equ1}
        \mathbf{Vol} (\Theta_i) < \frac{\tau_r \rho_H^r}{\mathbf{Vol} (\Theta_0)} \mathbf{Vol} (\Theta_0)=\tau_r \rho_H^r\leq \mathbf{Vol} (\bar\Theta_H)
    \end{equation}
    However,  from Lemma \ref{lemma.bar_theta}, one has $\bar\Theta_H \subset \Theta_i$, meaning $\mathbf{Vol} (\bar\Theta_H) \leq \mathbf{Vol} (\Theta_i)$. %Recall the assumption that $\mathbf{Vol} (\bar\Theta_H)$ contains an $r$-dimension ball with radius $\rho_H$, $\mathbf{Vol} (\bar\Theta_H) \geq \tau_r \rho_H^r$. 
    This leads to a contradiction with (\ref{equ.proofthm1_equ1}). Therefore, the total number of iterations must be less than $K$. 
    
    Next, we prove $\exists k < K$ such that $\boldsymbol{\theta}_k \in \bar\Theta_H$. In contradiction, suppose $\forall k < K$, $\boldsymbol{\theta}_k \notin \bar\Theta_H$. By Lemma \ref{lemma.bar_theta}, this means that for all $i<K$, $\bar\Theta_H \subset \mathcal{C}_i$. Then, the algorithm will continue for $i=K$. Consequently, the total number of iterations can reach $K$, which contradicts that the algorithm will terminate within $K$ iterations. This concludes the proof.
\end{proof}
\subsection{Proof of Theorem \ref{theorem_misspec}}
\label{proof.theorem_misspec}
\begin{proof}
    We complete the proof in two parts. First we prove the algorithm converges to $\Theta^*$ which has $\mathbf{Vol} (\Theta^*)=0$. Based on this, we next  prove $\Theta^* \cap \mathbf{int}\ \Theta_0= \emptyset$.

    We first prove the algorithm converges to $\Theta^*$ which has $\mathbf{Vol} (\Theta^*)=0$. Given the condition $\Theta_0 \cap \bar\Theta_H=\emptyset$, the algorithm will never terminate because $\forall i=1,2,3,...\infty$, $\mathcal{C}_i \supset \bar\Theta_H$ due to Lemma \ref{lemma.bar_theta}. Since the volume of the initial hypothesis space is finite, by applying Lemma \ref{lemma v} we have 
    \begin{equation}
    \label{equ.zero_vol}
        \mathbf{Vol} (\Theta^*)=\lim_{i\rightarrow \infty}(1-\frac{1}{r})^i \mathbf{Vol} (\Theta_0)=0.
    \end{equation}
    
    Next,  we prove $\Theta^* \cap \mathbf{int}\ \Theta_0= \emptyset$ by contradiction.  Suppose there exists  $\widehat{\boldsymbol{\theta}} \in \Theta^* \cap \mathbf{int}\ \Theta_0$. Define the ball contained in $\bar\Theta_H$ to be  $\mathcal{B}({\boldsymbol{\theta}^c_H}, {\rho_H})=\{\boldsymbol{\theta}\ |\ \norm{\boldsymbol{\theta}-\boldsymbol{\theta}_H^c} \leq \rho_H\} \subseteq \bar\Theta_H$ with $\boldsymbol{\theta}_H^c$ as the center. By the equation (\ref{equ.intersection}), we can note
    \begin{equation}
        \Theta^*=\Theta_0 \cap {\mathcal{C}},
    \end{equation}
    with ${\mathcal{C}}=\bigcap_{i=1}^{\infty} \mathcal{C}_i$ a convex set. So we have 
    \begin{equation}
        \Theta^* \subseteq \mathcal{C}.
    \end{equation}
    By equation (\ref{equ.C_steps}),
    \begin{equation}
        \mathcal{B}_H \subseteq \bar\Theta_H \subseteq \mathcal{C}.
    \end{equation}
    From the relationship between sets above, $\boldsymbol{\theta}_H^c\in \mathcal{C},\widehat{\boldsymbol{\theta}} \in \mathcal{C}$.

    Since $\boldsymbol{\theta}_H^c \notin \Theta_0$, there exists $\mu>0 $ such that
    \begin{equation}
        \widetilde{\boldsymbol{\theta}}=\widehat{\boldsymbol{\theta}} + \mu (\boldsymbol{\theta}_H^c-\widehat{\boldsymbol{\theta}}) \in \mathbf{int}\ \Theta_0.
    \end{equation}
Then, one can define a ball $\mathcal{B}(\widetilde{\boldsymbol{\theta}}, d)$ centered at $\widetilde{\boldsymbol{\theta}}$ with radius
    \begin{equation}
                d = \min_{\boldsymbol{\theta} \in \mathbf{bd}\ \Theta_0} \norm{\widetilde{\boldsymbol{\theta}}-\boldsymbol{\theta}}.
    \end{equation}
Thus, $\mathcal{B}(\widetilde{\boldsymbol{\theta}}, d)\subset \mathbf{int}\ \Theta_0$.

Next, we will show $\mathcal{B}(\widetilde{\boldsymbol{\theta}}, \mu\rho_H)\subset\mathcal{C}$. In fact, for any $\boldsymbol{\theta}$ satisfying $\norm{\boldsymbol{\theta}-\widetilde{\boldsymbol{\theta}}} \leq \mu\rho_H$, one has
    \begin{equation}
    \begin{split}
        \boldsymbol{\theta}&=\mu(\frac{1}{\mu}(\boldsymbol{\theta}-\widetilde{\boldsymbol{\theta}})+\boldsymbol{\theta}_H^c) + \widetilde{\boldsymbol{\theta}}-\mu \boldsymbol{\theta}_H^c\\
        &=\mu\underbrace{(\frac{1}{\mu}(\boldsymbol{\theta}-\widetilde{\boldsymbol{\theta}})+\boldsymbol{\theta}_H^c)}_{\in \mathcal{B}_H \subseteq \mathcal{C}} + (1-\mu) \underbrace{\widehat{\boldsymbol{\theta}}}_{\in \Theta^* \subseteq \mathcal{C}}
    \end{split} .
    \end{equation}
    The convexity of $\mathcal{C}$ leads to $\boldsymbol{\theta} \in \mathcal{C}$, which means the ball $\mathcal{B}_\mu \triangleq \{\boldsymbol{\theta}\ |\ \norm{\widetilde{\boldsymbol{\theta}}-\boldsymbol{\theta}} \leq \mu\rho_H\} \subset \mathcal{C}$.

Let a new ball be $\mathcal{B}(\widetilde{\boldsymbol{\theta}}, \min{(\mu\rho_H, d)})$. Based on the previous two results of  $\mathcal{B}(\widetilde{\boldsymbol{\theta}}, d)\subset \mathbf{int} \Theta_0$ and  $\mathcal{B}(\widetilde{\boldsymbol{\theta}}, \mu\rho_H)\subset\mathcal{C}$, one has $\mathcal{B}(\widetilde{\boldsymbol{\theta}}, \min{(\mu\rho_H, d)})\subset  (\Theta_0 \cap  \widetilde{\mathcal{C}}) = \Theta^*$. As a result,
\begin{equation}
    \mathbf{Vol}(\Theta^*)\geq \mathbf{Vol}\,\,\mathcal{B}(\widetilde{\boldsymbol{\theta}}, \min{(\mu\rho_H, d)})>0
\end{equation}
which contradicts $\mathbf{Vol}(\Theta^*)=0$ from (\ref{equ.zero_vol}). Thus, $\Theta^* \cap \mathbf{int}\ \Theta_0= \emptyset$ must hold. This completes the proof. 
\end{proof}

   %  another ball contained inside $\Theta_0$ can be constructed
   %  \begin{equation}
   %      \mathcal{B}_\Theta \triangleq \{\boldsymbol{\theta}\ |\ \norm{\widetilde{\boldsymbol{\theta}}-\boldsymbol{\theta}} \leq \min \{\mu\rho_H, d(\widetilde{\boldsymbol{\theta}}, \mathbf{bd}\ \Theta_0)\}\ \},
   %  \end{equation}
   %  $\mathcal{B}_\Theta$ has same center as $\mathcal{B}_\mu$ and smaller radius, leading to the relation $\mathcal{B}_\Theta \subseteq \mathcal{B}_\mu \subseteq \mathcal{C}$. The radius of $\mathcal{B}_\Theta$ is less than $d(\widetilde{\boldsymbol{\theta}}, \mathbf{bd}\ \Theta_0)$, which infers $\mathcal{B}_\Theta \subseteq \Theta_0$. Combining two inclusion relations,
   %  \begin{equation}
   % \label{equ.ball_in_space}
   %      \mathcal{B}_\Theta \subseteq (\Theta_0 \cap  \widetilde{\mathcal{C}}) = \Theta^*.
   %  \end{equation}
   %  The ball $\mathcal{B}_\Theta$ has positive volume so from (\ref{equ.ball_in_space}), 
   %  $\mathbf{Vol}(\Theta^*)>0$. It violates (\ref{equ.zero_vol}). Consequently,
   %  \begin{equation}
   %      \Theta^* \cap \mathbf{int}\ \Theta_0= \emptyset.
   %  \end{equation}
%    %  Which completes the proof. 
% \end{proof}

\section{Settings of Experiments and User Studies}
\subsection{Inverted Pendulum}
In the dynamics of an inverted pendulum (\ref{equ.pendulum_dyn}), the physical parameters is set in Table \ref{tbl.pendulum_params}.
\begin{table}[h]
\begin{center}
\caption{Physics Parameters in Pendulum Numerical Study}
\begin{tabular}{ccc}
     \hline
     Parameter & Value & Interpretation \\
     \hline
     $m$ & 1kg & Weight of the pendulum \\
     \hline
     $l$ & 1m & Length of the pendulum\\
     \hline
     $\mathrm{g}$ & $10\mathrm{m/s}^2$ & Gravitational Acceleration\\
     \hline
     $d$ & 0.1 & Dampening Coefficient\\
     \hline

\end{tabular}
\label{tbl.pendulum_params}
\end{center}
\end{table}
When the system state violates ${g}_{\boldsymbol{\theta}_H}(\boldsymbol{\xi}_{\boldsymbol{\theta}}^\gamma) \leq 0$ or it is near the target, the pendulum state is uniformly randomly reset between $[0,0]\tran$ and $[\frac{2\pi}{3},3]\tran$. Simulated human correction will continue to be applied until the learned parameters $\boldsymbol{\theta}_i$ fall in $\Bar\Theta_H$, which terminates one trial.

\subsection{Simulated Quadrotor Navigation}
\label{appendix.uav_sim}
\subsubsection{Environment}
The quadrotor starts from $[0.0, 0.0, 5.0]\tran$ [m]. The starting quaternion $\boldsymbol{q}_{\text{init}}=[1,0,0,0]\tran$. The target is to navigate to the target position $[15.0, 0.0, 10.0]\tran$[m] within a zigzag tube. The average radius of the tube section is set to be 2.5m.

\subsubsection{Learning Neural Features}
We construct the neural feature $\boldsymbol{\phi}$ using a neural Signed Distance Function (SDF) model trained on synthetic 3D point samples around the tube surface. Each input coordinate (scaled for stability) is first encoded using sinusoidal positional embeddings, then passed through a 5-layer MLP (width 128, ReLU activations). The positional embeddings are skip-connected to the input of the 5th hidden layer. The penultimate layer outputs a 16D feature vector used as $\boldsymbol{\phi}$; the final layer predicts a scalar signed distance supervised during training.

Training data is generated by randomly perturbing surface points on the tube. Signed distance labels are computed as the offset from the nearest point on the tube surface to the sample point, and normals are estimated from the same displacement vector pointing from the nearest surface point to the sample point. The model is trained with three loss terms: $L_2$ distance loss, cosine similarity loss on normals (via auto-diff), and Eikonal loss to regularize gradient norm, which is similar to the work \cite{ortiz2022isdf}. Optimization is done using Adam for 250 epochs with a batch size of 1024.

\subsubsection{Simulated Correction and Stop Behavior}
Correction signals $\boldsymbol{a}_x = [-1,0,1,0]\tran, \boldsymbol{a}_y = [0,1,0,-1]\tran, \boldsymbol{a}_z = [1,1,1,1]\tran$ in drone control space is used throughout the experiments, corresponding to different motion directions in cartesian space. During the process of learning positional constraints with  polynomial features or neural features, if the distance between the quadrotor and the nearest point on the surface of the tube is less than 0.6, a simulated correction $\boldsymbol{a}$ is applied in the action space of the drone, based on the vector $\boldsymbol{d} = (d_x,d_y,d_z)$ pointing from the closest tube surface point to the drone position:
\begin{equation}
    \boldsymbol{a} = d_y \boldsymbol{a}_y + d_z \boldsymbol{a}_z. 
\end{equation}
Intuitively, the simulated correction "pulls" the drone to the center of the tube (designated safe region). After each correction, the navigation restarts, emulating an emergency stop. During the process of learning velocity constraint, Simulated human corrections against the velocity direction of form:
\begin{equation}
    \boldsymbol{a} =-v_x \boldsymbol{a}_x -v_y \boldsymbol{a}_y - v_z \boldsymbol{a}_z
\end{equation}
are applied whenever the velocity $\boldsymbol{v} = (v_x,v_y,v_z)$ exceeds 0.4,  and an emergency stop is triggered if it surpasses 0.45.

\subsection{Quadrotor Navigation Game with Real Human Feedback}
\label{sec.uav_game}
\subsubsection{Environment}
The drone starts from $[0.1, y_{\text{init}}, 0.5]\tran$ with  $y_{\text{init}}$ uniformly randomly in an  interval $[4,6]$. The starting quaternion $\boldsymbol{q}_{\text{init}}=[1,0,0,0]\tran$. We set three different targets at positions $[19, 1, 9]\tran$, $[19, 5, 9]\tran$ and $[19, 9, 9]\tran$, with  quaternion $[1,0,0,0]\tran$. For the narrow gate (yellow rectangle) in Fig. \ref{subfig.drone}, both its size and position are unknown, and one only roughly knows its x position $x_{\text{gate}}\approx 9.8$. 

 Since the location of the gate is unknown, the drone needs to learn a safety constraint in its safe MPC policy with  humans guidance in order to successfully fly through the gate. In the drone MPC policy (\ref{equ.robot_mpc}), the dynamics and its parameters are set following \cite{jin2022learning}; and we generically set the constraint function ${g}_{\boldsymbol{\theta}}(\boldsymbol{\xi})$ using a set of $N$ Gaussian radial basis functions (RBFs)
\begin{equation}\label{equ.uav_safety}
    {g}_{\boldsymbol{\theta}}(\boldsymbol{\xi}) = \phi_0 + \sum\nolimits_{i=1}^{N} \theta_i \phi_i(\boldsymbol{\xi}).
\end{equation}
As discussed in Section \ref{sec.safety_param}, we fix $\phi_0=-1$ to avoid learning ambiguity, and use $ N=20$ temporally accumulated RBFs defined as 
\begin{equation}\label{equ.discounted_rbf}
    \phi_i(\boldsymbol{\xi}) = \sum\nolimits_{t=0}^T \beta_t \mathrm{RBF}_i(\boldsymbol{x}_t),
\end{equation}
where  $\beta_t$ is a weighting factor to balance the immediate and future constraint satisfaction. If one emphasizes instant safety, more weights will be given to the initial time step. In our implementation, we choose an exponentially decaying $\beta_t$ starting at $t=5$, i.e. $\beta_t = 0,\ t < 5$ and $\beta_t = 0.9^{t-5},\ t \geq 5$ to achieve better predictive safety performance. In (\ref{equ.discounted_rbf}), each  RBF is defined as 
\begin{equation}
        \mathrm{RBF}_i(\boldsymbol{x}_t)=
        \begin{cases}
            \exp(-\epsilon^2\norm{r_{xy,t}  - c_{xy,i}}^2),\,\, i =1,...,10 
         \\
            \exp(-\epsilon^2\norm{r_{xz,t}  - c_{xz,i}}^2),\,\, i =11,...,20 
        \end{cases}.      
\end{equation}
Here, the variables $r_{xy,t}=[r_{x,t}, r_{y,t}]\tran$ and $r_{xz,t}=[r_{x,t}, r_{z,t}]\tran$ with  $r_x, r_y, r_z$ from  $\boldsymbol{x}_t$. The $i$th center $c_{xy,i}=[c_{x,i}, c_{y,i}]\tran$ or $c_{xz,i}=[c_{x,i}, c_{z,i}]\tran$ is defined with $c_{x,i}={x_{\text{gate}}}=9.8$, $c_{y,i}$ and $c_{z,i}$ are evenly chosen on interval $[0,10]$. We set $\epsilon=0.45$. The initial hypothesis space $\Theta_0$ for possible $\boldsymbol{\theta}$ in (\ref{equ.uav_safety})  is set as $\boldsymbol{c}_l=[-80,\dots,-80]\tran$ and $\boldsymbol{c}_u=[100,\dots,100]\tran$. $\gamma$ is 50. 
This initialization blocks the drone from advancing to the gate at the beginning of the game, leading to conservative behavior in the learning process. The drone will proceed toward the gate after few corrections.

\smallskip
\subsubsection{Human Correction Interface}
The drone’s motion is visualized in MuJoCo using two complementary camera views: a first-person tracking view and a fixed third-person view. The gate, the current MPC plan, and the target position are always shown to the user. Human users provide corrections through keyboard inputs during flight. Key–command mappings are given in Table \ref{tbl.uav_cmds}. Whenever a correction is detected, the safety constraint (\ref{equ.uav_safety}) in the MPC policy is updated instantly. Pressing ENTER resets the drone to its initial state and reset the target.

\vspace{-10pt}
\begin{table}[h]
\begin{center}
\caption{keyboard commands in uav user study}
\begin{tabular}{m{0.6cm}  m{1.4cm} m{5.0cm}}
     \hline
     Key & Correction & Interpretation \\
     \hline
     "up" & $[-1,0,1,0]\tran$ & Go forward (adding positive y-axis torque) \\
     \hline
     "down" & $[1,0,-1,0]\tran$ & Go  backward  (adding negative y-axis torque) \\
     \hline
     "left" & $[0,1,0,-1]\tran$ & Go  left  (adding negative x-axis torque)   \\
     \hline
     "right" & $[0,-1,0,1]\tran$ & Go right  (adding positive x-axis torque)  \\
     \hline
     "shift" & $[1,1,1,1]\tran$ & Go up (increasing all thrusts)\\
     \hline
     "ctrl" & $-[1,1,1,1]\tran$ & Go down (decreasing all thrusts)\\
     \hline
     "enter" & N/A & Reset robot position and switch target.\\
     \hline
\end{tabular}
\label{tbl.uav_cmds}
\end{center}
\vspace{-10pt}
\end{table}

\subsection{Franka Constrained Reaching Game}
\label{sec.arm_game}
\subsubsection{Environment}
The arm starts at a configuration in which the end-effector position $\boldsymbol{r}_{\text{init}}=[0.22,-0.5,0.48]\tran$ and altitude $\boldsymbol{q}_{\text{init}}=[0,-0.707,0,0.707]\tran$. The reaching target position $\boldsymbol{r}^*$ will  be randomly chosen from three options $[-0.6,-0.5,0.4]\tran$, $[-0.6,-0.5,0.5]\tran$, $[-0.6,-0.5,0.6]\tran$. The reaching target attitude is $\boldsymbol{q}^*=[0,-0.707,0,0.707]\tran$. The slot size and position are unknown except its rough x-position which is $x_{\text{slot}}\approx-0.15$.

We adopt a two-level control scheme, where a safe MPC policy (high-level) is defined for the end effector, which generates the desired velocity for a given low-level operational space controller to track. In the safe MPC policy, we model the end-effector dynamics as
\begin{equation}
    \begin{gathered}
        \dot{\boldsymbol{r}}=\boldsymbol{v}^d, \\ 
        \dot{\boldsymbol{q}} = \frac{1}{2} \Omega(R\tran(\boldsymbol{q}) \boldsymbol{w}^d)\boldsymbol{q},
    \end{gathered}
\end{equation}
where $\boldsymbol{r}$ and $\boldsymbol{q}$ is the position and the quaternion of the arm end-effector in the world frame. $R\tran (\boldsymbol{q})$ is the rotation matrix from the end-effector frame to the body frame. $\Omega(\cdot)$ is the matrix operator of the angular velocity for quaternion multiplication. Define the robot state and control input
\begin{equation}\label{equ.franka_xu}
        \boldsymbol{x}_t=[\boldsymbol{r}_{t},\ \boldsymbol{q}_{t}] \in \mathbb{R}^{7},\quad
        \boldsymbol{u}_t = [\boldsymbol{v}^d_{t},\ \boldsymbol{w}^d_{t}] \in \mathbb{R}^6,
\end{equation}
respectively.

The above end-effector dynamics is discretized with $\triangle t = 0.1 \mathrm{s}$. In (\ref{equ.robot_mpc}), the MPC horizon is  $T=20$.  Let $\boldsymbol{x}^*=[\boldsymbol{r}^*\ \boldsymbol{q}^*]$ 
denote the target end-effector pose. The step cost $c(\boldsymbol{x}_t,\boldsymbol{u}_t)$ and terminal cost $h(\boldsymbol{x}_T)$ in (\ref{equ.robot_mpc}) is set to
\begin{equation}
\begin{aligned}
    c(\boldsymbol{x}_t,\boldsymbol{u}_t) =&(\boldsymbol{r}_{t}-\boldsymbol{r}^*)\tran Q_r (\boldsymbol{r}_{t}-\boldsymbol{r}^*) + k_q \mathbf{diff}(\boldsymbol{q}_t, \boldsymbol{q}^*) 
        \\ &\quad \quad + \boldsymbol{v}_{I,t}\tran R_v \boldsymbol{v}_{I,t} + \boldsymbol{w}_{I,t}\tran R_w \boldsymbol{w}_{I,t},\\
        h(\boldsymbol{x}_T)=&(\boldsymbol{r}_T-\boldsymbol{r}^*)\tran S_r (\boldsymbol{r}_T-\boldsymbol{r}^*) + \mu_q \mathbf{diff}(\boldsymbol{q}_T, \boldsymbol{q}^*).
\end{aligned}
\end{equation}
In this game, the  cost parameters are set as $Q=\mathrm{diag}(2,2,2)$, $k_q=1$,  $R_v=\mathrm{diag}(30,30,1)$, $R_w=\mathrm{diag}(0.85,0.85,0.85)$, $S_r=\mathrm{diag}(80,80,80)$, and $\mu_q=60$. The output of the MPC controller is sent to the low-level controller to track. 

\begin{table}[h]
\begin{center}
\caption{keyboard commands in robot arm user study}
\begin{tabular}{m{0.6cm}  m{2.0cm} m{4.6cm}}
     \hline
     Key & Correction & Interpretation \\
     \hline
     "left" & $[0,0,0,0,0,1]\tran$ & Clock-wise rotation of the gripper \\
     \hline
     "right" & $[0,0,0,0,0,-1]\tran$ & Anti-clock-wise rotation of the gripper\\
     \hline
     "shift" & $[0,0,0,-1,0,0]\tran$ & Go up in world frame\\
     \hline
     "ctrl" & $[0,0,0,1,0,0]\tran$ & Go down in world frame\\
     \hline
     "Enter" & N/A & Reset robot position and switch target.\\
     \hline
\end{tabular}
\label{tbl.arm_cmds}
\end{center}
\end{table}

Similar to Appendix \ref{sec.uav_game}, to constrain $r_z$ and $\boldsymbol{q}$ for proper height and tilt in order to pass through the slot. We use a set of $N=20$ Gaussian radial basis function (RBFs) to parameter the safety constraint: 
\begin{equation}\label{equ.safety_arm}
    {g}_{\boldsymbol{\theta}}(\boldsymbol{\xi}) = \phi_0 + \sum\nolimits_{i=1}^{N} \theta_i \phi_i(\boldsymbol{\xi}),
\end{equation}
with 
\begin{equation}
    \phi_i(\boldsymbol{\xi}) = \sum\nolimits_{t=0}^T \beta_t \mathrm{RBF}_i(\boldsymbol{x}_t), \quad \text{and}
\end{equation}
\begin{equation}
        \mathrm{RBF}_i(\boldsymbol{x}_t){=}
        \begin{cases}
            \sigma(r_{x,t})  k_z \exp{(-\epsilon_z^2(r_{z,t}-c_{z,i})^2)}, \,\,i {=}1,...,10 \\
            \sigma(r_{x,t})  \exp{(-\epsilon_q^2\mathbf{diff}(\boldsymbol{q}_{t}, \boldsymbol{q}_{c,i}))}, \,\,i {=}11,...,20 
        \end{cases}.
        \label{equ.arm_rbf}
\end{equation}
Here, $k_z$ is used to balance the scaling between position and orientation $\mathrm{RBF}$s,
$\sigma(r_{x,t})$ is defined as
\begin{equation}
    \sigma(r_{x,t}) = \mathrm{sigmoid}(50\big{(}(r_{x,t} - x_{\text{slot}})^2 - (\frac{w}{2})^2\big{)}),
\end{equation} which is used to constrain the $\mathrm{RBF}_i$ only taking effect around $x$ range $[x_{\text{slot}}-\frac{w}{2}, x_{\text{slot}}+\frac{w}{2}]$ (recall slot's $x$ location, $\text{slot}$, is a approximate  value). In the above definition, $w=0.3$, $k_z=0.05$, $\epsilon_z=12$ and $\epsilon_q=1.8$. The center of positional $\mathrm{RBF}$, ${c_{z,i}}$, is evenly chosen on interval $[0,1]$, and the center of orientation $\mathrm{RBF}$, ${\boldsymbol{q}_{c,i}}$, is   generated by $(\cos{\frac{\alpha}{2}},\sin{\frac{\alpha}{2}}, 0, 0)*\boldsymbol{q}_{\text{init}}$ with $\alpha$ evenly choosen from $[-\pi/4, 3\pi/4]$ to make sure ${\boldsymbol{q}_{c,i}}$ covers all possible orientations. $\beta_t=0.9^t$ if $t<5$ and $0$ otherwise.
The bounds of the initial hypothesis space  $\Theta_0$   is set to $\boldsymbol{c}_l=[-3,\dots,-3]\tran$ and $\boldsymbol{c}_u=[5,\dots,5]\tran$.  $\gamma$ is set to 1.5. This initialization leads to a conservative MPC policy, preventing the arm from approaching the target at the beginning.

\subsubsection{Human Correction Interface}
The robot arm’s motion is visualized in MuJoCo using a free-view mode (Fig. \ref{subfig.arm}), allowing users to adjust the camera to a preferred angle. The slot and the randomized reaching target remain visible throughout the task. Users provide corrections to the robot arm via keyboard input, following the specified key–command mappings. Whenever a correction is detected, the safety constraint (\ref{equ.safety_arm}) in the Franka arm’s safe MPC policy is updated immediately. Pressing ENTER resets the arm to its initial position and randomizes the reaching target.

\vspace*{-5\baselineskip}
\begin{IEEEbiography}[{\includegraphics[width=1in,height=1.25in,clip,keepaspectratio]{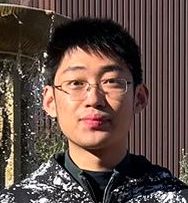}}]{Zhixian Xie}
Zhixian Xie is currently a Ph.D. student in Arizona State University, Tempe, AZ, USA. In 2022, he received B. Eng degree in automation at Shanghai Jiaotong University, Shanghai, China. His research focuses mainly on controls, robotic learning from human and robot dexterous manipulation.
\end{IEEEbiography}
\vspace*{-3\baselineskip}
\begin{IEEEbiography}[{\includegraphics[width=1in,height=1.25in,clip,keepaspectratio]{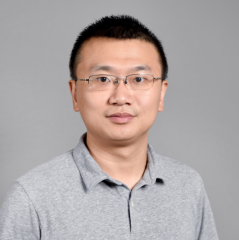}}]{Wenlong Zhang}
Wenlong Zhang is an associate professor in the School of Manufacturing Systems and Networks at Arizona State University, where he serves as the Research Director. He is a core faculty member in the Robotics and Autonomous Systems and Systems Engineering PhD programs. He joined ASU faculty in 2015 after earning his doctoral degree in mechanical engineering from the University of California, Berkeley. At ASU, he leads the Robotics and Intelligent Systems Laboratory (RISE Lab), and teaches undergraduate- and graduate-level classes on dynamics, control systems, optimization, and robotics.
\end{IEEEbiography}
\vspace*{-3\baselineskip}
\begin{IEEEbiography}[{\includegraphics[width=1in,height=1.25in,clip,keepaspectratio]{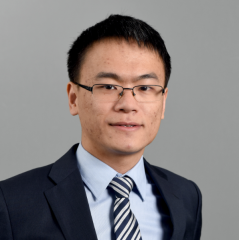}}]{Yi Ren}
Yi (Max) Ren is an associate professor of aerospace and mechanical engineering in the School for Engineering of Matter, Transport and Energy at Arizona State University. he directs the Design Informatics Laboratory at ASU. His group studies computational and data-driven methods that will augment or transform the activities of engineering and industrial design. Ren’s current research interests include optimization, product/configuration design, human-computer interaction and machine learning.
\end{IEEEbiography}
\vspace*{-3\baselineskip}
\begin{IEEEbiography}[{\includegraphics[width=1in,height=1.25in,clip,keepaspectratio]{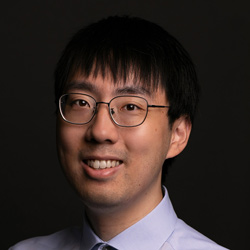}}]{Zhaoran Wang}
Zhaoran Wang is an associate professor with Northwestern University, working at the interface of machine learning, statistics, and optimization. He is the recipient of the AISTATS (Artificial Intelligence and Statistics Conference) notable paper award, ASA
(American Statistical Association) best student paper in statistical learning and data mining, INFORMS (Institute for Operations Research and the Management Sciences) best student paper finalist in data mining, Microsoft PhD Fellowship, SimonsBerkeley/J.P. Morgan AI Research Fellowship, Amazon Machine Learning Research Award, and NSF CAREER Award.
\end{IEEEbiography}
\vspace*{-1\baselineskip}
\begin{IEEEbiography}[{\includegraphics[width=1in,height=1.25in,clip,keepaspectratio]{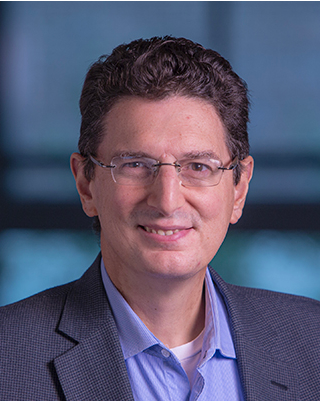}}]{George J. Pappas}
George J. Pappas is the UPS Foundation Professor at the Department of Electrical and Systems Engineering at the University of Pennsylvania. He also holds a secondary appointment in the Departments of Computer and Information Sciences, and Mechanical Engineering and Applied Mechanics. Previously, he served as the Deputy Dean for Research in the School of Engineering and Applied Science. Pappas’s research interest focus on control systems, robotics and autonomous systems, formal methods, machine learning for safe and secure cyber-physical systems. He has received numerous awards including the NSF PECASE, the Antonio Ruberti Young Researcher Prize, the George S. Axelby Award, the O. Hugo Schuck Best Paper Award, and the George H. Heilmeier Faculty Excellence Award. Pappas has mentored more than fifty students and postdocs, now faculty in leading universities worldwide. He is a Fellow of IEEE, IFAC, and was elected to the National Academy of Engineering in 2024.
\end{IEEEbiography}
\vspace*{-1\baselineskip}
\begin{IEEEbiography}[{\includegraphics[width=1in,height=1.25in,clip,keepaspectratio]{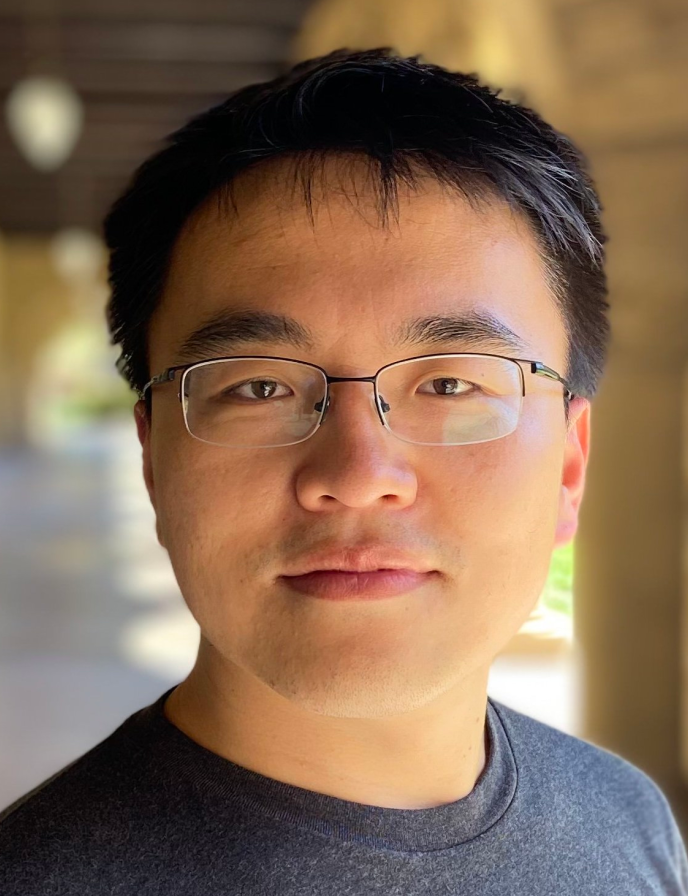}}]{Wanxin Jin}
Wanxin Jin is an assistant professor the School for Engineering of Matter, Transport and Energy at Arizona State University (ASU). Previously, he was a postdoctoral fellow at the GRASP Laboratory at the University of Pennsylvania. He obtained his doctorate from Purdue University in 2021. His research includes robotics, control and machine learning, with a focus on the autonomy of robots interacting with humans and objects.
\end{IEEEbiography}
\vfill

\end{document}